\documentclass{article}

    \PassOptionsToPackage{numbers, compress}{natbib}

 \usepackage[main,final]{neurips_2026}

\usepackage[utf8]{inputenc} 
\usepackage[T1]{fontenc}    
\usepackage{hyperref}       
\usepackage{url}            
\usepackage{booktabs}       
\usepackage{amsfonts}       
\usepackage{nicefrac}       
\usepackage{microtype}      
\usepackage{xcolor}  
\usepackage{amsmath}
\usepackage{float}
\usepackage{graphicx}
\newtheorem{prop}{Proposition}
\newtheorem{corollary}{Corollary}
\newtheorem{lemma}{Lemma}
\usepackage{thm-restate}
\usepackage{sidecap}
\sidecaptionvpos{figure}{c}

\newcommand{\sG}{s^{G}_{t}}
\newcommand{\sL}{s^{L}_{t}}
\newcommand{\piG}{\pi^{G}_{t}}
\newcommand{\piL}{\pi^{L}_{t}}

\definecolor{C0}{HTML}{3182ce} 
\definecolor{C1}{HTML}{dd6b20} 
\definecolor{newred}{HTML}{C32C39} 
\hypersetup{colorlinks,citecolor=C0,linkcolor=C0,urlcolor=C1}

\title{Theoretical guidelines for annealed Langevin dynamics in compositional simulation-based inference}

\author{%
 Camille Touron\\
  Univ. Grenoble Alpes, Inria
 \\
CNRS, Grenoble INP, LJK, France \\
  \texttt{camille.touron@inria.fr} \\
   \And
   Gabriel V. Cardoso \\
    Geostatistics team, Centre for geosciences and geoengineering\\
Mines Paris, PSL University, Fontaineableau, France\\
\texttt{gabriel.victorino\_cardoso@minesparis.psl.eu} \\
   \AND
   Julyan Arbel \\
 Univ. Grenoble Alpes, Inria
 \\
CNRS, Grenoble INP, LJK, France \\
   \texttt{julyan.arbel@inria.fr} \\
   \And
   Pedro L. C. Rodrigues \\
   Univ. Grenoble Alpes, Inria
 \\
CNRS, Grenoble INP, LJK, France \\
   \texttt{pedro.rodrigues@inria.fr} \\
}

\begin{document}

\maketitle

\begin{abstract}
Compositional score-based approaches to simulation-based inference (SBI) approximate the posterior over a shared parameter given $n$ independent observations by aggregating individually learned posterior scores: currently, there are two main propositions of such methods (\citet{geffner, linhart2024diffusion}). As the resulting composite score does not correspond to the score of any distribution along the forward diffusion path of the true multi-observation posterior, sampling from it via a reverse SDE leads to an irreducible bias. Annealed Langevin dynamics provides a principled alternative: it treats the composite score as the genuine score of a sequence of tractable bridging densities and samples from them in succession. When properly tuned, it could lead to a controllable bias. However, its hyperparameters, namely step sizes, the number of steps per level, and the number of annealing levels, have so far been chosen empirically. We derive Wasserstein bounds for annealed Langevin with approximate scores and translate them into explicit decision rules for these hyperparameters that guarantee a prescribed sampling accuracy, while highlighting different theoretical aspects of each composite score formulation. In the Gaussian setting, we obtain closed-form expressions for all relevant quantities and prove that the bridging densities of \citet{linhart2024diffusion} consistently admit larger step sizes and require fewer total Langevin steps than those of \citet{geffner}. Furthermore, we show empirically that the tuning obtained in the Gaussian setting generalizes to more complex problems, thus providing a well-understood and theoretically grounded starting point for practitioners using compositional score-based approaches.
\end{abstract}

\section{Introduction}


In simulation-based inference (SBI), the goal is to approximate the posterior distribution
$p(\theta \mid x)$ of a stochastic model whose likelihood is intractable but which can be simulated \citep{cranmer2020frontier, sbi_guide}.  A recurrent challenge arises when $n$ independent observations $x_1,\ldots,x_n$ share a common unknown parameter~$\theta^\star$: as $n$ grows, the multi-observation posterior $p(\theta \mid x_{1:n})$ concentrates, and one wants inference to sharpen accordingly. \emph{Pooling strategies} train a single network on all $n$ observations jointly
\citep{radev2020bayesflow, rodrigues2021hnpe}, but their simulation budget scales with $n$ and the network must be retrained for each new choice of~$n$.  \emph{Compositional strategies} \citep{geffner, linhart2024diffusion} instead combine individual posterior estimators trained once on single-observation data, amortizing over $n$ at no additional simulation cost: a key practical advantage that motivates our focus on this paradigm.

Score-based diffusion models \citep{song_annealed_langevin,song2021score,sharrock2024sequential}
provide a natural starting point: for a single observation, one learns the individual
posterior score $\nabla_\theta \log p_t(\theta \mid x_i)$ along a forward diffusion
path and uses it in a reverse SDE to draw samples from $p(\theta \mid x_i)$.  In the
multi-observation setting, the tempting shortcut is to aggregate these individual scores
into a composite score and run the same reverse SDE to target $p(\theta \mid x_{1:n})$.  This
shortcut is, however, theoretically unjustified.  The reverse SDE must be driven by the
\emph{exact} score of the joint forward diffusion of $p(\theta \mid x_{1:n})$, i.e.\
$\nabla_\theta \log p_t(\theta \mid x_{1:n})$, which satisfies the Fokker--Planck equation
of that specific process.  A composite of individual scores does \emph{not} satisfy this
equation: it is the score of a different distribution $\pi_t \neq p_t(\theta \mid x_{1:n})$
that does not lie on the forward diffusion path of the target.  Using such a score in a
reverse SDE therefore produces samples from the wrong distribution, as detailed in Appendix~\ref{app::compar_diffusion}, and no amount of
additional reverse steps can correct this structural mismatch \citep{du2023reduce}. This also explains why some works \citep{song2021score, timeseriesgloecker} attempt to correct this bias by incorporating additional \textit{Langevin} refinement steps into the standard reverse diffusion process. However, this comes at the cost of sacrificing the efficiency of diffusion as a fast sampling method while keeping the challenge of tuning the Langevin hyperparameters.

Several works recognized this problem and proposed composite scores that, while still mismatched with the true forward diffusion,
deliberately define their own tractable sequence of bridging densities \citep{geffner, arruda2025compositionalamortizedinferencelargescale,linhart2024diffusion,timeseriesgloecker}. In this paper, we focus on the following two main composite score formulations considered in the literature.
\citet{geffner} use
\begin{equation}
  \sG(\theta \mid x_{1:n})
  = (1-n)\,\nabla_\theta \log \lambda_t(\theta)
  + \textstyle\sum_{i=1}^n \nabla_\theta \log p_t(\theta \mid x_i),
  \label{eq::geffner_composite_score}
\end{equation}
whose associated bridging density $\piG$ is well-defined and tractable.
\citet{linhart2024diffusion} incorporate a second-order Gaussian approximation of the
backward diffusion kernels to obtain $\sL$, whose bridging density $\piL$ is closer to
the true diffused posterior $p_t(\theta \mid x_{1:n})$, at the cost of estimating
per-observation and prior covariance matrices $\Sigma_{t,i}$, $\Sigma_{t,\lambda}$. Their composite score reads
\begin{equation}
\label{eq::linhart_composite_score}
    s_t^L(\theta\mid x_{1:n})=\Lambda_t^{-1}\Big(\sum_{i=1}^n \Sigma_{t,i}^{-1}\nabla_\theta\log p_t(\theta\mid x_i)+(1-n)\Sigma_{t,\lambda}^{-1}\nabla_\theta\log \lambda_t(\theta)\Big),
 \end{equation} where $\Lambda_t$ is a linear combination of $\Sigma_{t,i}$, $\Sigma_{t,\lambda}$ and $\lambda_t(\theta)$ denotes the diffused prior. $s_t^G$ and $s_t^L$ coincide with the score of the target multi-observation posterior when $t=0$ but deviate from its diffused versions for $t>0$. In both cases, the key insight is to stop trying
to target $p_t(\theta \mid x_{1:n})$ through a reverse SDE, and instead \emph{define}
$\piG$ or $\piL$ as the actual sampling target: the question becomes how to efficiently
draw samples from these bridging densities given access to their scores.

Annealed Langevin dynamics is one theoretically sound alternative \citep{song_annealed_langevin}.  It runs $T$
successive Unadjusted Langevin Algorithm (ULA) chains, each with step size $h_{t_p}$ and $k_{t_p}$ number of steps,
\begin{equation}
  \theta^{(j+1)} = \theta^{(j)}
  + h_{t_p}\,\nabla_\theta \log \pi_{t_p}\!(\theta^{(j)})
  + \sqrt{2h_{t_p}}\,z^{(j)}, \quad z^{(j)} \sim \mathcal{N}(0,I_d), \ \ j=0,\ldots,k_{t_p}-1
  \label{eq:ula}
\end{equation}
on bridging densities $\{\pi_{t_p}\}_{p=0}^T$ interpolating from
$\pi_{t_T}=\mathcal{N}(0,I_d)$ to the target $\pi_{t_0}$ (Figure~\ref{fig:schema_ula}). It requires
only the score of each $\pi_{t_p}$, which the composite scores evaluated at time $t_p$ provide by construction, and
its gradual annealing improves mixing for concentrated posteriors at large $n$. Since
\citet{song2021score} established the SDE interpretation of diffusion models, the field has
largely shifted to reverse-SDE and DDPM/DDIM samplers \citep{ho2020denoising}, which
require far fewer score evaluations. Those samplers, however, are theoretically grounded
only when the score is computed or trained to match the \emph{exact} forward score $\nabla_\theta\log p_t(\theta\mid x_{1:n})$, a condition violated
here by construction.  Annealed Langevin avoids this mismatch entirely, at the cost of
more score evaluations and the need to choose, at each level $t_p$, a step size $h_{t_p}$, a
number of steps $k_{t_p}$, and a total number of levels $T$.  Selecting these hyperparameters in a principled way, with guarantees on the resulting sampling quality, has remained an open problem; both \citet{geffner} and \citet{linhart2024diffusion} resort to empirical tuning. In the context of image generation, \citet{improved_techniques_annealed_langevin} propose some techniques to choose these hyperparameters, most relying on empirical validation, but leaving theoretical guarantee on the sampling quality unaddressed. Our paper closes that gap.

\begin{figure}[h]
    \centering
    \includegraphics[width=\linewidth]{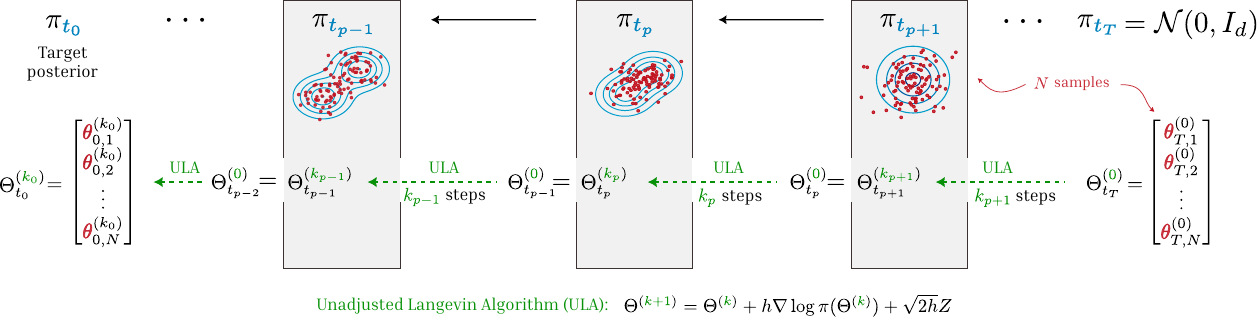}
    \caption{\small The Annealed Langevin algorithm consists in $T$ ULA steps in succession. It allows to sample from bridging densities $\{\pi_{t_p}\}_{p=0}^{T-1}$, represented in blue: the first ULA yields samples (in red) from $\pi_{t_{T-1}}$, the second from $\pi_{t_{T-2}}$, etc... and the last one from the target density $\pi_{t_0}$. Green dashed arrows stand for the $T$ ULA: each of them has to define its own constant step size $h_{t_p}$ and number of steps $k_{t_p}$ (simply denoted $k_p$ for readability concerns). The first process starts from Gaussian samples; then the start point of each process corresponds to the last iterates of the previous process, which is symbolized by the equal signs on the scheme.}
    \label{fig:schema_ula}
\end{figure}

\textbf{Contributions.} We provide the first theoretically grounded guidelines for selecting the hyperparameters of annealed Langevin dynamics in compositional SBI.

\textbf{(i) Wasserstein bounds (Section~\ref{sec:theo_bounds}).} Building on \citet{dalalyan}, we derive a per-level bound for a single ULA chain accounting for both composite score error and the Wasserstein distance between consecutive bridging densities (Proposition~\ref{prop::intermediate_bound}); composing $T$ such bounds yields a global sampling guarantee (Proposition~\ref{prop::global bound}). We show that log-concavity and smoothness of the target are inherited by both Geffner's and Linhart's bridging densities, in closed form for Gaussians (Lemma~\ref{lemma::preservation_cst}) and under mild conditions in general (Lemma~\ref{lemma::preservation_cst_nongaussian}).

\textbf{(ii) Hyperparameter decision rule.} Given a target accuracy $\gamma>0$
and a bias--variance split $\omega\in(0,1)$, we derive explicit formulas for the hyperparameters $h_{t}$ and $k_{t}$ ensuring the Wasserstein error stays below $\gamma$ at every annealing level $t\in\{t_0,\ldots,t_T\}$, with the choice of compositional score entering directly as a design parameter (Section~\ref{sec::fine_tune}). We illustrate the practical use of our decision rule in Gaussian and non-Gaussian settings (resp. in Sections~\ref{sec::theo_comparaison} and \ref{sec::empirical_comparaison}).

\textbf{(iii) Geffner vs.\ Linhart.} In the Gaussian setting, we show that Linhart's algorithm always admits larger step sizes and its bridging densities move less between consecutive levels near $t = 1$ (Section~\ref{sec::theo_comparaison}), requiring fewer total steps. We extend the comparison to non-Gaussian settings empirically (Section~\ref{sec::empirical_comparaison}). In summary, Linhart's composite score yields a more efficient sampler.

\section{Wasserstein error bounds for annealed Langevin dynamics}\label{sec:theo_bounds}
To measure the sampling quality and the rate of convergence of ULA, different works derive upper-bounds of the Wasserstein distance \citep{durmus_asymptotic_conv_ula} or the KL divergence \citep{kl_div_ula} between the target distribution $\pi_{t_0}$ and the distribution of the $k^{\text{th}}$ iterate $\theta^{(k)}$ generated by the algorithm: these bounds rely on assumptions on $\pi_{t_0}$ (smoothness, log-concavity or log-Sobolev inequality) and are valid only with an exact score function. Theorem 4 of \citet{dalalyan} is one of the first results to establish a bound when the score is not known accurately. Based on these results, we are able to upper-bound the Wasserstein error of an annealed Langevin algorithm with $T$ annealing steps, as we show below.

We denote by $\mathcal{L}(\theta^{(k)})$ the distribution of the $k^{\text{th}}$ iterate of a Langevin process. A density $\pi$ is $m$-log concave and $M$-smooth when $mI_d\preceq -\nabla^2_\theta\log \pi(\theta)\preceq MI_d$ for all $\theta$ in the support of $\pi$. It is said to be Gaussian log-Lipschitz if it admits a density with respect to the Gaussian measure whose logarithm is Lipschitz continuous. 
The first proposition bounds the Wasserstein error at a \textit{given} annealing level $t_p\in\{t_0,\ldots,t_T\}$. 
\begin{prop}[Intermediate Wasserstein error bound]\label{prop::intermediate_bound}
    Suppose that $\pi_{t_p}$ is $m_{p}$-log concave and $M_{p}$-smooth in $\mathbb{R}^d$. Let $k_p$ be the number of steps, $h_p$ the constant step size such that $h_{p}<\frac{2}{M_{p}+m_{p}}$, $B_{p}=1.65\frac{M_p}{m_p}(h_pd)^{0.5}+\frac{\epsilon_\text{DSM}}{m_p}$ and $\epsilon_\text{DSM}$ upper bounds the $L_2$-score-matching error. Then,
    \begin{equation*}
        \mathcal{W}_2(\mathcal{L}(\theta_{t_{p}}^{(k_{p})}),\pi_{t_{p}})\leq (1-m_{p}h_{p})^{k_{p}}  \left(\mathcal{W}_2(\mathcal{L}(\theta_{{t_{p+1}}}^{(k_{p+1})}),\pi_{{t_{p+1}}})+\mathcal{W}_2(\pi_{{t_{p+1}}},\pi_{t_p})\right)+B_{p}.
    \end{equation*}
\end{prop}
Proof is given in Appendix~\ref{proof::intermediate_bound}. For readability concerns, we replace the index $t_p$ simply by $p$ for the constants $M, m, h, k, B$, but they remain level-dependent.
With respect to Theorem 4 in \citet{dalalyan}, we have included a new term based on the Wasserstein distance between two consecutive intermediate densities $\mathcal{W}_2(\pi_{t_{p+1}},\pi_{t_p})$ and also simplified the role of the bias of the score approximation to consider a biased but deterministic score estimate controlled by the score-matching error $\epsilon_{\text{DSM}}$. Note that in the case of compositional scores, we only have access to the score-matching errors of individual score estimates $s_\phi(\theta,x_i,t)\approx \nabla_\theta\log p_t(\theta\mid x_i)$: it is, however, possible to obtain a bound of the compositional score error by composing these individual errors (see Section~\ref{sec::compositional_error}).

We also emphasize that $B_p$ depends only on the choice of step size $h_p$, while the number of steps $k_p$ influences only the Wasserstein distances: this observation will  motivate the guidelines we propose in Section~\ref{sec::fine_tune}. The following proposition gives a theoretical Wasserstein guarantee for the annealed Langevin algorithm by composing the intermediate bound $T$ times, while keeping the same notations as before. Proof is given in Appendix~\ref{proof::global_bound}.
\begin{prop}[Global Wasserstein error bound]\label{prop::global bound}
Suppose that $\pi_{t_p}$ is $m_p$-log concave and $M_p$-smooth for all $p\in\{0,\ldots,T-1\}$. Let $h_p$ be the constant step size at the $p^{\text{th}}$ annealing level such that $h_p<\frac{2}{M_p+m_p}$ and $k_p$ be the number of steps. Then,
\begin{equation*}
    \mathcal{W}_2(\mathcal{L}(\theta_{t_0}^{(k_0)}),\pi_{t_0})\leq \sum_{p=0}^{T-1}\left(\prod_{s=0}^p (1-m_sh_s)^{k_s}\right)\mathcal{W}_2(\pi_{t_{p+1}},\pi_{t_p})+ B_{p}\prod_{s=0}^{p-1} (1-m_sh_s)^{k_s},
\end{equation*}
where $B_p=1.65\frac{M_p}{m_p}(h_pd)^{0.5}+\frac{\epsilon_\text{DSM}}{m_p}$ for all $p\in\{0,\ldots, T-1\}$.    
\end{prop}
 Proposition~\ref{prop::global bound} highlights that an efficient Langevin sampling is governed by the quantities $\{m_{p},M_p,\mathcal{W}_2(\pi_{t_{p+1}},\pi_{t_p})\}_{p \in \{0, \cdots, T-1\}}$, which depend on the choice of the intermediate distributions. We provide in Lemma~\ref{lemma::fisher_div} (Appendix~\ref{proof::bound_wass_fisher}) additional bounds for $\mathcal{W}_2(\pi_{t_{p+1}},\pi_{t_p})$ for both $\pi^L_t$ and $\pi^G_t$ showing that while $\mathcal{W}_2(\pi^G_{t_{p+1}},\pi^G_{t_p})$ is governed by the $L^2$ proximity of the scores, $\mathcal{W}_2(\pi^L_{t_{p+1}},\pi^L_{t_p})$ takes into account the covariance of the backward diffusion kernels. 

Although Proposition~\ref{prop::global bound} relies on strong assumptions, requiring all intermediate densities $\pi_{t_p}$ to retain log-concavity and smoothness properties of the target density, especially when $t_p\to t_0$, 
Lemma~\ref{lemma::preservation_cst_nongaussian} relaxes these assumptions in the following training setting used throughout the rest of the paper.

\paragraph{Diffusion setting.} From now on, we consider training individual scores using the Variance-Preserving (VP) diffusion scheme \citep{song2021score}. The corresponding forward transition kernel is of the form $q(\theta_t\mid \theta_0)=\mathcal{N}(\theta_t;\sqrt{\alpha_t}\theta_0,v_tI_d)$ where $v_t=1-\alpha_t$ and $\alpha_t$ is a strictly decreasing function of time such that $\alpha_0=1$ and $\alpha_t\underset{t\to1}{\to}0$. In this setting, we consider $T$ annealing levels spaced over the interval $[0,1]$, with the convention that $t_0=0$ and $t_T=1$. This choice is motivated by the fact that, under VP training, the intermediate densities $\pi_t^L$ and $\pi_t^G$ converge to a standard Gaussian as $t\to 1$.  
\begin{lemma}[Preservation of smoothness and log-concavity - General case]\label{lemma::preservation_cst_nongaussian}
    Let $\sigma_\text{min}(A)$ (resp. $\sigma_\text{max}(A)$) denote the minimal (resp. maximal) eigenvalue of any matrix $A$. Let $\Sigma_{t,i}$, $\Sigma_{t,\lambda}$ denote the covariance matrices of the Gaussian backward kernels in Linhart's compositional score (Equation~\ref{eq::linhart_composite_score}). Assume that the prior and each individual posterior are Gaussian log-Lipschitz, then they become smooth along the VP scheme, which is equivalent to:
    \begin{align*}
        m_{t,\lambda} I_d&\preceq -\nabla^2_\theta\log \lambda_t(\theta)\preceq M_{t,\lambda} I_d,\\
    m_{t,i} I_d&\preceq -\nabla^2_\theta\log p_t(\theta\mid x_i)\preceq M_{t,i} I_d \quad \forall \ t\in [0,1],
    \end{align*} where $m_{t,\lambda}, m_{t,i}\in \mathbb{R}$ and $M_{t,\lambda}, M_{t,i}\in \mathbb{R}_+$. Moreover, $\pi_t^G$ and $\pi_t^L$ are also smooth, with smoothness constants having an explicit formula given in Appendix~\ref{proof::smoothness_preserve}. 
    If $\frac 1n\sum_{i=1}^n \frac{1}{\gamma_{t,i}}m_{t,i}\geq \frac{n-1}{n}\frac{1}{\gamma_{t,\lambda}}M_{t,\lambda}$ for one given $t \in [0,1]$, then $\pi_t^G$ is log-concave when $\gamma_{t,i}=\gamma_{t,\lambda}=1$; under additional symmetry conditions, $\pi_t^L$ is also log-concave, when $\gamma_{t,i}=\sigma_{\text{min}}(\Sigma_{t,i})$ if $m_{t,i}\leq 0$, $\gamma_{t,i}=\sigma_{\text{max}}(\Sigma_{t,i})$ if $m_{t,i}\geq 0$ and $\gamma_{t,\lambda}=\sigma_{\text{min}}(\Sigma_{t,\lambda})$.
\end{lemma} 
Proof is derived in Appendix~\ref{proof::smoothness_preserve}. It highlights that mild first-order regularity conditions on initial distributions can yield a stronger smoothness property for the intermediate densities over time. Furthermore, if the individual posteriors and the prior are only smooth (not log-concave) at some time $t$, their negative smoothness constants can be compensated by positive contributions for sufficiently large $n$, making $\pi_t^G$ and $\pi_t^L$ smooth and log-concave. This is also consistent with Bernstein-von Mises theorem, under which the multi-observation posterior is asymptotically Gaussian (thus smooth and log-concave) as $n$ increases \citep{bernstein}. 
In addition, \citet{log_concavity_strengthen} prove that log-concavity is strengthened along the forward diffusion path, driving smoothness constants $m_{t,i}$ and $m_{t,\lambda}$ towards less negative values provided that the initial distributions are weakly-log-concave: thus, the condition that ensures log-concavity for $\pi_t^G$ and $\pi_t^L$ becomes increasingly mild as $t\to1$.

Therefore, Lemma~\ref{lemma::preservation_cst_nongaussian} allows one to use the theoretical bounds in Propositions~\ref{prop::intermediate_bound} and \ref{prop::global bound} under milder assumptions, when first order regularity conditions on individual posteriors and prior induce smoothness \textit{and} log-concavity properties for the target multi-observation posterior and bridging densities $\pi_t^G$ and $\pi_t^L$. This is notably the case for multimodal individual posteriors, which when considered on their own may be only smooth, without being log-concave. In the following section, we propose a method to fine-tune the hyperparameters of the annealed Langevin algorithm using the theoretical bounds previously derived and under the relaxed smoothness assumptions of Lemma~\ref{lemma::preservation_cst_nongaussian}.

\section{Guidelines for fine-tuning annealed Langevin hyperparameters}
\label{sec::fine_tune}
Despite its apparent freedom to choose intermediate densities, annealed Langevin dynamics remains hyperparameter-heavy. Previous works using this sampling scheme typically fix the number of steps a priori and tune the constant step size at each annealing level using taming rules \citep{linhart2024diffusion, geffner} or some heuristics \citep{improved_techniques_annealed_langevin} to improve sampling quality \citep{tames_ula}. In this paper, we propose to choose these hyperparameters so that the annealed Langevin algorithm achieves a theoretically guaranteed small Wasserstein error. Our decision rule is based on the intermediate theoretical bound derived in Proposition~\ref{prop::intermediate_bound}: assuming that the Wasserstein error of the $(p+1)^{\text{th}}$ Langevin process is controlled by a threshold $\gamma>0$, i.e. $\mathcal{W}_2(\mathcal{L}(\theta_{t_{p+1}}^{(k_{p+1})}),\pi_{t_{p+1}})\leq \gamma$, the idea is to choose $h_{t_p}$ and $k_{t_p}$ so that the Wasserstein error of the following process remains lower than $\gamma$, i.e. $\mathcal{W}_2(\mathcal{L}(\theta_{t_p}^{(k_{p})}),\pi_{t_p})\leq \gamma$. Therefore, the final Wasserstein error of the algorithm given in Proposition~\ref{prop::global bound} will also be bounded by $\gamma$. Let $\omega\in(0,1)$ be a weight parameter, such that $\omega\gamma>\frac{\epsilon_\text{DSM}}{m_t}$ then choosing 
\begin{align*}
    0\leq h_{t_p}&\leq \min\left(\frac{(\omega\gamma-\frac{\epsilon_\textbf{DSM}}{m_{t_p}})^2}{d(1.65)^2}\frac{m_{t_p}^2}{M_{t_p}^2},\frac{2}{m_{t_p}+M_{t_p}}\right),\\
    k_{t_p}&\geq \log\left( \frac{(1-\omega)\gamma}{\gamma+\mathcal{W}_2(\pi_{t_{p+1}},\pi_{t_p})}\right)\frac{1}{\log(1-m_{t_p}h_{t_p})}\geq0, 
\end{align*}
achieves this objective. With these choices, the bias term $B_{t_p}$ in Proposition~\ref{prop::intermediate_bound} becomes lower than $\omega\gamma$, while the rest of the upper bound is lower than $(1-\omega)\gamma$ (see Appendix~\ref{proof::fine-tune} for more details). Note that the condition on $\omega\gamma$ seems natural: the intermediate bound in Proposition~\ref{prop::intermediate_bound} shows a linear dependence between the score error $\epsilon_\text{DSM}$ and the Wasserstein error; therefore, it seems difficult to achieve a high sampling quality (small $\gamma$) using a very biased score estimate (high $\epsilon_\text{DSM}$).

We highlight that the choice of bridging densities $\pi_{t_p}$ (Geffner-like or Linhart-like) significantly affects hyperparameter selection under this decision rule: indeed, the choice of step sizes mainly involves the ratio $m_{t_p}/M_{t_p}$ that depends on each bridging density $\pi_{t_p}$, while the choice of the number of steps depends on the Wasserstein distance between two consecutive bridging densities $\mathcal{W}_2(\pi_{t_{p+1}},\pi_{t_p})$ and the log-concavity constant $m_{t_p}$. These theoretical quantities may be difficult to compute in practice. In the following section, we consider Gaussian distributions, yielding analytical expressions for these quantities. We compare the step sizes and numbers of steps chosen with this decision rule when the bridging densities are derived from Geffner or Linhart. 
\section{Theoretical comparison of guidelines in a Gaussian setting}\label{sec::theo_comparaison}
In this section, we consider a $d$-dimensional standard Gaussian prior $\lambda(\theta)=\mathcal{N}(\theta;0, I_d)$ and a Gaussian likelihood $p(x\mid\theta)=\mathcal{N}(x; \theta,\Sigma)$. Given $n$ IID observations $x_1,\ldots,x_n$, the goal of Bayesian inference is to sample from the (Gaussian) multi-observation posterior $p(\theta\mid x_{1:n}):=\pi_0(\theta)$. In this Gaussian setting, the (diffused) individual posteriors $p_t(\theta\mid x_i)$ and the bridging densities ($\pi_t^G$ or $\pi_t^L$) are also Gaussian and analytically tractable for all $t\in[0,1]$. To simplify the comparison between both algorithms in terms of step sizes and number of steps, we consider a null score error, i.e. $\epsilon_\text{DSM}=0$ by analytically computing all the scores involved.
\subsection{Geffner's and Linhart's intermediate densities}\label{sec::analytical_densities}
We start with the densities $\pi_t^L$ defined by \citet{linhart2024diffusion} for $t\in [0,1]$. As all the assumptions of their work are fulfilled in the Gaussian setting, the distributions underlying their composite scores (Equation~\ref{eq::linhart_composite_score}) correspond to the diffused versions of the target multi-observation posterior along the VP forward process and have the following Gaussian form :
\begin{equation*}
    \pi_t^L(\theta\mid x_{1:n}) = p_t(\theta\mid x_{1:n})=\mathcal{N}(\theta;\mu_t^L, \Sigma_t^L) \quad \text{with} \quad  \left\{
\begin{aligned}
\Sigma_t^L&=\alpha_t(n\Sigma^{-1}+I_d)^{-1}+v_tI_d \\
\mu_t^L &=\sqrt{\alpha_t}(nI_d+\Sigma)^{-1}\sum_{i=1}^n x_i.
\end{aligned}
\right.
\end{equation*}
Let $\Sigma_t=\alpha_t(\Sigma^{-1}+I_d)^{-1}+v_tI_d$ be the covariance matrix of any diffused individual posterior $p_t(\theta\mid x_i)$. Then, the bridging densities $\pi_t^G$ defined by \citet{geffner} are also Gaussian, with the following first two moments:
\begin{equation*}
        \pi_t^G(\theta \mid x_{1:n})=\mathcal{N}(\theta;\mu_t^G,\Sigma_t^G) \quad \text{with} \quad \left\{
\begin{aligned}
\Sigma_t^G&=\left((1-n)I_d+n \Sigma_t^{-1}\right)^{-1} \\ \mu_t^G&=\sqrt{\alpha_t}\Sigma_t^G\Sigma_{t}^{-1}(I_d+\Sigma)^{-1}\sum_{i=1}^n x_i.
\end{aligned}
\right. 
\end{equation*}
Note that the covariance of Geffner's densities is well defined for all time $t\in [0,1]$ as the covariance matrices of the prior and the likelihood commute (see Appendix~\ref{proof::analyt_densities}).
 In the Gaussian case, these intermediate densities naturally retain the log-concavity and smoothness properties of the target distribution $\pi_0$ over time. Lemma~\ref{lemma::preservation_cst} provides an affine relation between the log-concavity constant $m_0$ (resp. smoothness constant $M_0$) of $\pi_0$ and Linhart's bridging densities constants $m_t^L$ (resp. $M_t^L$) or Geffner's counterparts $m_t^G$ (resp. $M_t^G$) for all $t\in [0,1]$. Proof is given in Appendix~\ref{proof::preservation_cst}.
\begin{lemma}[Preservation of smoothness and log-concavity - Gaussian case]\label{lemma::preservation_cst}
    Let the target posterior distribution $\pi_0$ be $M_0$-smooth and $m_0$-log-concave. Let $\sigma_{\max}$ (resp. $\sigma_{\min}$) denote the maximal (resp. minimal) eigenvalue of the likelihood covariance matrix $\Sigma$. Then $M_t^G$ and $M_t^L$ (resp. $m_t^G$ and $m_t^L$) are time-dependent affine functions of $M_0$ (resp. $m_0$) as per
    \begin{align*}
        M_t^L = M_0\frac{\sigma_{\min}}{\sigma_{\min}+nv_t}\ \quad &\text{and} \quad m_t^L = m_0\frac{\sigma_{\max}}{\sigma_{\max}+nv_t},\\
        M_t^G=M_0\frac{\sigma_{\min}}{\sigma_{\min}+v_t}+\frac{(1-n)v_t}{\sigma_{\min}+v_t}\quad &\text{and} \quad m_t^G=m_0\frac{\sigma_{\max}}{\sigma_{\max}+v_t}+\frac{(1-n)v_t}{\sigma_{\max}+v_t}.
    \end{align*}
\end{lemma} 

\subsection{Comparison of step sizes}
\label{sec:comparison}
With the decision rule provided in Section~\ref{sec::fine_tune}, the choice of step sizes depends mainly on the ratio $\frac{m_t}{M_t}$. In the Gaussian setting, the log-concavity constant $m_t$ and smoothness constant $M_t$ respectively correspond to the minimal and maximal eigenvalue of the precision matrix of the intermediate density $\pi_t^G$ or $\pi_t^L$. The following proposition shows that these constants are always higher for Geffner's densities than for Linhart's counterparts.
\begin{prop}[Link between $M_t^G$ and $m_t^G$ with $M_t^L$ and $m_t^L$]\label{prop::comparison_cst}
Let $\sigma_\text{min}$ and $\sigma_\text{max}$ denote the minimal and maximal eigenvalue of the likelihood covariance matrix $\Sigma$. Then, we have
\begin{equation*}
    M^G_t = M^L_t + \frac{n(n-1)\alpha_tv_t}{D(\sigma_\text{min},v_t)} \quad \text{and} \quad 
    m^G_t = m^L_t + \frac{n(n-1)\alpha_tv_t}{D(\sigma_\text{max},v_t)},
\end{equation*}
where $D(\sigma,v_t)=(\sigma+v_t)(\sigma+nv_t)$ and $t\in[0,1]$.
\end{prop}
Proof is provided in Appendix~\ref{proof::link_cst}.
 Proposition~\ref{prop::comparison_cst} gives us insights into the ratio of both constants, therefore into the choice of step sizes for each algorithm presented in Section~\ref{sec::fine_tune}.
\begin{corollary}[Comparison of the ratios $m_t^G/M_t^G$ and $m_t^L/M_t^L$]\label{cor::comparison_ratios}
Let $M_t^L, m_t^L$ (resp. $M_t^G, m_t^G$) denote the smoothness and log-concavity constants of Linhart's (resp. Geffner's) intermediate densities. Then we have :
\begin{equation*}
    \frac{m_t^G}{M_t^G}\leq \frac{m_t^L}{M_t^L} \quad \forall t\in[0,1].
\end{equation*}
Consequently, the step sizes $h_t^G$ used with Geffner's composite score in annealed Langevin are smaller than their counterparts $h_t^L$ used with Linhart's score for any time $t\in [0,1]$.
\end{corollary}
Proof is given in Appendix~\ref{proof::comparison_ratio}.
Choosing Linhart's composite score instead of Geffner's counterpart allows one to use larger step sizes at each annealing level without degrading the final sampling quality, measured in Wasserstein distance. The following section provides a comparison of the number of steps for both algorithms.

\subsection{Comparison of the number of steps}
For a given level $t_p\in[0,1]$, the choice of the number of steps $k_{t_p}$ prescribed in Section~\ref{sec::fine_tune} depends on the Wasserstein distance between intermediate densities $\mathcal{W}_2(\pi_{t_p},\pi_{t_{p+1}})$ and the log-concavity constant $m_{t_p}$. When $t_p\to 0$ or $t_p\to 1$, the contribution of $m_{t_p}$ in the fine-tuning objective is negligible, since this term is essentially the same for both algorithms. As a result, the weighting of the Wasserstein distance in the choice of $k_{t_p}$ is the only factor that enables a meaningful comparison between the two approaches.

As intermediate distributions $\pi_{t_p}^L$ or $\pi_{t_p}^G$ are Gaussian for all $t_p\in[0,1]$, the Wasserstein distance has a closed form expression including a $L_2$-mean difference and a Bures distance between covariance matrices \citep{bures_distance}. The following proposition compares the intermediate Wasserstein distances for both algorithms when $t_p\to 0$ and $t_p\to1$.
\begin{restatable}[Comparison of Wasserstein distances between intermediate densities]{prop}{CompWass}\label{prop::comparaison_wass}
Let $\pi_{t_p}^L$ (resp. $\pi_{t_p}^G$) denote the (Gaussian) intermediate density defined by Linhart (resp. by Geffner) at level $t_p$. Then, we have :
\begin{align*}
    \text{As} \ t_p\to 0, \quad \mathcal{W}_2^2(\pi_{t_p}^G,\pi_{t_{p+1}}^G)&\leq \mathcal{W}_2^2(\pi_{t_p}^L,\pi_{t_{p+1}}^L).\\
        \text{As} \ t_p\to 1, \quad \mathcal{W}_2^2(\pi_{t_p}^G,\pi_{t_{p+1}}^G)&\geq \mathcal{W}_2^2(\pi_{t_p}^L,\pi_{t_{p+1}}^L).
\end{align*}
Therefore, the number of Langevin steps  $k_{t_p}^G$ (chosen as in Section~\ref{sec::fine_tune}) used with Geffner's composite score in annealed Langevin is greater than its counterpart $k_{t_p}^L$ used with Linhart's score as $t_p\to1$, and conversely as $t_p\to 0$.
\end{restatable}
Proof is given in Appendix~\ref{proof::comparaison_wass}. We quantify computational complexity by the total number of score function evaluations\footnote{The computational complexity does not include the additional cost of estimating the covariance matrices required by Linhart's composite score $s_t^L$, as we treat this as a \textit{shared} cost, as detailed in Appendix~\ref{proof::technical_details}}, that is also equal to the total number of Langevin steps $\sum_{p=0}^Tk_{t_p}$. Proposition~\ref{prop::comparaison_wass} gives us insight into the complexity of each algorithm in two different extreme regimes. As $t\to 0$, the annealed Langevin algorithm tends to sample from intermediate distributions close to the target density $\pi_0$: in this case, using Geffner's composite score yields less number of steps for an equivalent sampling quality than Linhart's composite score. Conversely, when $t\to 1$, the annealed Langevin algorithm tends to sample from distributions close to a standard Gaussian: in this case, using Linhart's score seems beneficial, as the number of steps is lower and the step size higher.

A theoretical characterization of the evolution of the number of steps between these two regimes remains an open question. Empirical tests indicate that there are only two regimes between $t=0$ and $t=1$ (see Figure~\ref{fig:evol_nb_steps_gaussian}): there seems to exist a time level $\tilde t$ such that on $[0,\tilde t]$, $k_t^L$ is always higher than $k_t^G$, and on $[\tilde t,1]$ the situation is reverse. This switching time level varies depending on the number of observations $n$, the eigenvalues of the likelihood covariance matrix, and the dimension of the space $d$ in a complex manner. Though, we observe that $\tilde t$ is often close to $0$ regardless of dimension $d\geq3$ : exceptions appear when $d=2$ where the switching point may be later (see Figure~\ref{fig:exception_dim_2_gaussian} in Appendix~\ref{app::additional_figure_gaussian_case}). This unclear evolution may also stem from the contribution of the term $h_tm_t$ in the choice of $k_t$ which may outweigh that of the Wasserstein distance: in extreme regimes, this term is very similar for Geffner's and Linhart's densities, making $k_t$ depend on the intermediate Wasserstein distances only.

Globally, the theoretical results accompanied by the empirical observations in the $d$-dimensional Gaussian setting indicate that achieving a Wasserstein error comparable to Geffner's method, requires Linhart's algorithm to use larger step sizes and fewer Langevin steps for most annealing levels, thus a lower total number of steps.  Figure~\ref{fig:wass_validation_gaussian} (Appendix~\ref{app::empirical_validation_wass}) provides an empirical validation of the Wasserstein error given by the prescribed hyperparameters for both compositional score choices.
\begin{figure}[t]
    \centering
\includegraphics[width=\linewidth]{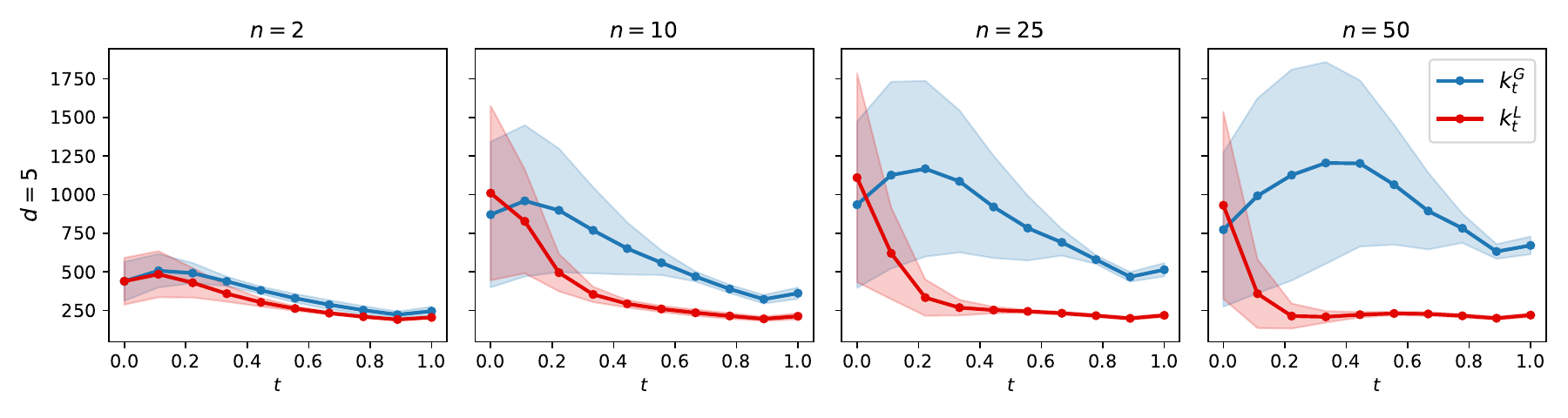}
    \caption{\small Evolution of the number of Langevin steps in the $5$-dimensional Gaussian setting for different numbers $n$ of conditional observations. Blue (resp. red) curves represent the mean number of steps $k_t^G$ (resp. $k_t^L$) at each annealing level when using Geffner's (resp. Linhart's) compositional score. We choose $T=10$ annealing levels uniformly spaced over $[0,1]$. The mean (and std) is computed over $5$ different Gaussian models, each with a different likelihood matrix. Empirical results for $d=2,10$ are provided in Appendix~\ref{app::additional_figure_gaussian_case}.}
    \label{fig:evol_nb_steps_gaussian}
\end{figure}
\section{Hyperparameter selection in practice}\label{sec::empirical_comparaison}
In this section, we illustrate empirically that the hyperparameter guidelines developed in Section~\ref{sec::fine_tune} are useful beyond the purely Gaussian case. The main difficulty to generalize these guidelines is to compute the smoothness ($M_t^G, M_t^L$) and log-concavity constants ($m_t^G, m_t^L$) that drive the choice of the step sizes $h_t^G$ and $h_t^L$, as well as the Wasserstein distances between consecutive densities that affect the choice of the number of steps per annealing level $k_t^G$ and $k_t^L$. We therefore consider a second-order approximation of each individual posterior and the prior, which allows us to get an analytical Gaussian approximation of the target multi-observation posterior by composing individual posteriors and prior as per $p(\theta\mid x_{1:n})\propto\lambda(\theta)^{1-n}\prod_{i=1}^n p(\theta\mid x_i)$. Note that this approximation becomes increasingly reasonable according to the Bernstein-von Mises theorem \citep{bernstein}, as the number of observations $n$ increases (see Appendix~\ref{app::discuss_gaussian_approx} for more details). Intermediate densities $\pi_t^G$ or $\pi_t^L$ also admit analytical expressions, similar to those in Section~\ref{sec::analytical_densities}. This makes their smoothness and log-concavity constants, together with Wasserstein distances between intermediate densities, tractable and enables the selection of ($h_t^G, k_t^G$) or ($h_t^L, k_t^L$) according to our decision rule. Importantly, we do not introduce any additional Gaussianity assumptions here but instead rely on those used by \citet{linhart2024diffusion} when estimating the covariance matrices in Equation~\ref{eq::linhart_composite_score} with their \texttt{GAUSS} algorithm. We also note that we cannot provide a strong theoretical upper bound on the Wasserstein error under such approximations, but we expect them not to compromise the theoretical Wasserstein control, especially for large $n$.

\paragraph{General setting.}\label{sec::compositional_error}For all the following experiments, we use $T$ annealing levels uniformly spaced over the interval $[0,1]$, so that $t_p=\frac{p}{T}$ for $p\in \{0,\ldots,T\}$. We train a time-dependent amortized score estimate $s_\phi$ with score matching on samples of the joint distribution $(\theta,x)\sim\lambda(\theta)p(x\mid \theta)$ (see Appendix~\ref{proof::technical_details} for details). This score is an estimate of the \emph{exact} score of individual posteriors along the VP diffusion scheme, i.e. $s_\phi(\theta,x_i,t) \approx \nabla_\theta \log p_t(\theta\mid x_i)$ for all $t\in[0,1]$ and all $x_i$.  Prior scores can often be derived analytically (see Appendix~\ref{proof::analytical_scores}). As prescribed in \citet{linhart2024diffusion} for their \texttt{GAUSS} algorithm, we use the score $s_\phi$ to sample individual posteriors using a reverse diffusion scheme and compute empirical covariances, involved in Equation~\ref{eq::linhart_composite_score}. Contrary to the previous section, we need to take into account the score-matching error $\epsilon_\text{DSM}$ of each compositional score in the choice of step sizes: as we compose individual scores according to Equation~\ref{eq::geffner_composite_score} or \ref{eq::linhart_composite_score}, we only have access to the score-matching errors of individual
score estimates. However, it is possible to obtain a bound of this compositional score error for both algorithms: \citet{compositional_error} provide a bound of the mean squared error of Linhart's compositional score $s_t^L$ based on individual score errors. Adopting the authors' notations, we derive a bound for Geffner's score error in Lemma~\ref{lemma::composit_error}, where $\tilde s_t^G$ is the estimate of $s_t^G$ (Equation~\ref{eq::geffner_composite_score}) with inexact individual scores. Proof is given in Appendix~\ref{proof::compositional_error}. Appendix~\ref{app::estimation_score_error} provides additional technical details on score error estimation.
\begin{lemma} \label{lemma::composit_error}Let $s_\phi$ (resp. $s_\lambda$) be an amortized neural-based approximation of the individual posteriors $p(\theta\mid x_i)$ (resp. the prior $\lambda(\theta)$). Assume that $\mathrm{E}_\theta\left[\Vert \nabla_{\theta}\log \lambda_t(\theta)-s_\lambda(\theta,t)\Vert^2\right]\leq \epsilon_{\text{DSM},\lambda}^2$ and $\mathrm{E}_\theta\left[\Vert \nabla_{\theta}\log p_t(\theta\mid x_i)-s_\phi(\theta,x_i,t)\Vert^2\right]\leq \epsilon_{\text{DSM}}^2$ for all $i=1,\ldots,n$. Then,
\begin{equation*}
     \mathrm{E}_\theta\left[\Vert \nabla_{\theta}\log \pi_t^G(\theta\mid x_{1;n})-\tilde s_t^G(\theta\mid x_{1:n})\Vert^2\right]\leq \Big((n-1)\epsilon_{\text{DSM},\lambda}+n\epsilon_{\text{DSM}}\Big)^2.
\end{equation*}
\end{lemma}
\paragraph{Numerical applications.}\label{sec::numerical_applcations}
We consider the following four inference tasks, where $d$ denotes the dimension of the parameter space:

\textbf{GMM prior} ($d=2$): The prior is a mixture of two Gaussians with tunable means, and the likelihood is Gaussian. The resulting posterior is a mixture of Gaussians, whose modes become more separated as the prior means move far apart. Note that as $n$ grows, the posterior becomes approximately Gaussian and unimodal. This task allows us to test the robustness of our decision rule for multimodal target posteriors, where non-Gaussianity keeps mastered (see Figure~\ref{fig:gmm_prior_additional_results} in Appendix~\ref{app::gmm_prior_additional_figures}).

\textbf{GMM likelihood} ($d=10$): The prior is a standard Gaussian, while the likelihood is a mixture of two Gaussians, yielding a Gaussian mixture posterior. This task is inspired by \citet{geffner} and deviates slightly more from Gaussianity than the first one.

\textbf{SIR} ($d=2$): This epidemiological model tracks the evolution of susceptible (S), infected
(I) and recovered (R) populations. The prior is Log-Normal, the simulator is based on a set of differential equations that outputs samples from a Binomial distribution. The resulting posterior is unimodal. 

\textbf{Lotka-Volterra} ($d=4$): This predator-prey model describes the evolution of the populations of two interacting species. The prior is Log-Normal, the simulator is based on a set of differential equations that outputs samples from a Log-Normal distribution. The resulting posterior is unimodal.

For the first two tasks, analytical individual posterior scores can be derived (see Appendix~\ref{proof::analytical_scores}). The last two tasks, inspired by the SBI benchmark \citep{lueckmann2021benchmarking}, allow us to test the robustness of our decision rule when individual posterior scores need to be learned and the compositional error estimated for different parameter dimensions $d$.
 The hyperparameters chosen according to our decision rule allow us to run annealed Langevin \textit{without} intermediate clipping steps for all tasks. Figure~\ref{fig:evol_wass_non_gaussian_case} (in Appendix~\ref{proof::technical_details}) further shows that they yield a \textit{similar} final Wasserstein error, chosen to be \textit{small}, for both choices of intermediate densities: this empirical observation also holds for different choices of threshold $\gamma$ and number of annealing levels $T$, as detailed in Appendices~\ref{app::different_gamma} and \ref{app::different_T}. Overall, Linhart’s densities use larger step sizes and require fewer steps than Geffner’s counterparts across most inference tasks (see Figure~\ref{fig:results_non_gaussian}): Linhart's composite score thus yields a more efficient sampler in terms of score function evaluations.
 \begin{figure}[t]
    \centering
\includegraphics[width=\linewidth]{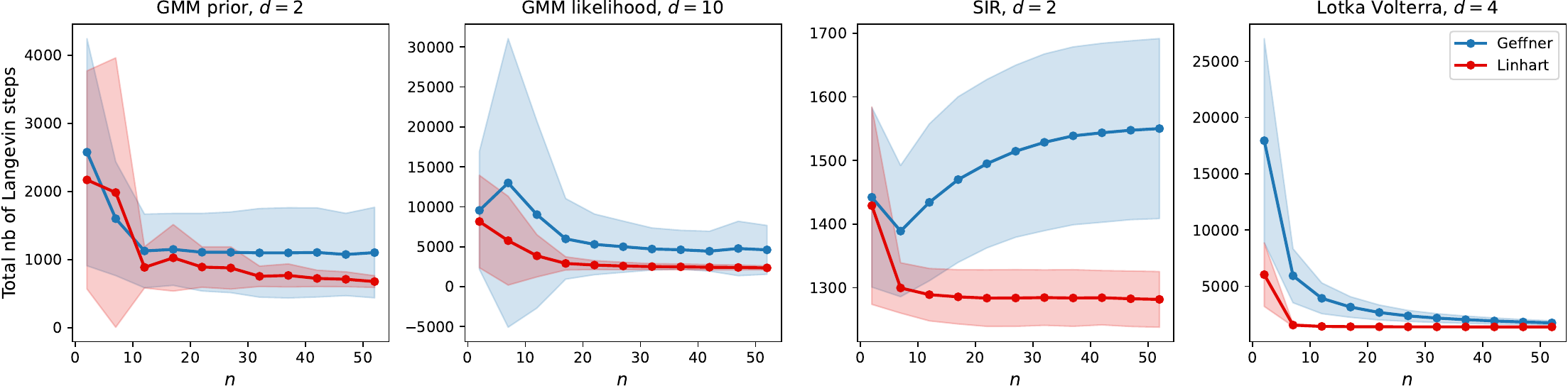}
    \caption{\small Evolution of the mean total number of Langevin steps $\sum_{p=0}^Tk_{t_p}$ for different numbers of conditional observations $n$. These choices achieve a prescribed low final Wasserstein error for both compositional score formulations (see Appendix~\ref{proof::technical_details}). Blue curves stand for Geffner's complexity while red curves represent that of Linhart. Mean and std are computed over $5$ different seeds.}
    \label{fig:results_non_gaussian}
\end{figure}
\section{Conclusion}
We propose the first theoretically grounded guidelines to fine-tune the hyperparameters of annealed Langevin algorithm in compositional SBI, supported by Wasserstein error bounds under smoothness and log-concavity assumptions. These guidelines remain valid under weaker smoothness assumptions on individual posteriors, particularly for large $n$, thereby extending their applicability to broader settings. We highlight the central role of intermediate densities that underlie the compositional scores used in annealed Langevin, in the hyperparameter fine-tuning. We provide explicit hyperparameter prescriptions in the Gaussian case and show that they can be effectively transferred to non-Gaussian settings, thus avoiding the need for manual tuning while achieving a prescribed low Wasserstein error. In addition, we show that Linhart's compositional score yields a more efficient sampler than Geffner's one in both settings in terms of score function evaluations.

Our analysis could serve as a theoretically grounded starting point for practitioners using compositional score-based approaches for designing novel intermediate densities that improve the annealed Langevin sampling performance.
Although our theoretical guarantees provide useful guidance for hyperparameter tuning, they still rely on smoothness and (approximate) log-concavity assumptions that may not hold in highly multimodal regimes; extending these guarantees to more general distributions, e.g. those satisfying a log-Sobolev inequality only, is a key direction for future work.

\section*{Acknowledgements}
PLCR and JA were supported by national grants managed by the French National Research Agency (Agence Nationale de la Recherche) attributed to the SBI4C project (PLCR) and the THUNDER project (JA) of the MIAI Cluster, under the reference ANR-23-IACL-0006. 
JA was supported by the ANR project Bayes-Duality (grant ANR-21-JSTM-0001). 
GVC was partially funded by the Geolearning chair of Mines Paris -- PSL, funded by Andra, BNP Paribas, CCR and the SCOR Foundation for Science. 
\newpage

\small
\bibliographystyle{apalike}
\bibliography{biblio}

\newpage

\normalsize

\appendix
\tableofcontents

\section{Theoretical results}
\subsection{Proof of Proposition~\ref{prop::intermediate_bound}}\label{proof::intermediate_bound}
Suppose that we want to sample from a distribution $\pi_{t_p}$ that is $m_{t_p}$-log concave and $M_{t_p}$-smooth, for a given $p\in \{0,\ldots,T-1\}$. $\pi_{t_p}$ stands for one bridging density in the annealed Langevin process (defined by Geffner or Linhart). For readability concerns, we replace the index $t_p$ simply by $p$ for constants $m, M, B, h, k$. Suppose that the constant step size is such that $h_p<\frac{2}{m_p+M_p}$, then according to \citet{dalalyan} (Theorem 4), the following holds:
\begin{align*}
    \mathcal{W}_2(\mathcal{L}(\theta_{t_p}^{(k_p)}),\pi_{t_p})&\leq (1-m_ph_p)^{k_p}  \mathcal{W}_2(\mathcal{L}(\theta_{t_p}^{(0)}),\pi_{t_p})+B_p\\
    &\leq (1-m_ph_p)^{k_p}  (\mathcal{W}_2(\mathcal{L}(\theta_{t_p}^{(0)}),\pi_{t_{p+1}})+\mathcal{W}_2(\pi_{t_{p+1}},\pi_{t_p}))+B_p \\&\quad \quad \text{by triangular inequality for Wasserstein distance}\\
    &= (1-m_ph_p)^{k_p}  \mathcal{W}_2(\mathcal{L}(\theta_{t_p}^{(0)}),\pi_{t_{p+1}})+(1-m_ph_p)^{k_p}\mathcal{W}_2(\pi_{t_{p+1}},\pi_{t_p})+B_p\\
    &\boxed{= (1-m_ph_p)^{k_p}  \mathcal{W}_2(\mathcal{L}(\theta_{t_{p+1}}^{(k_{p+1})}),\pi_{t_{p+1}})+(1-m_ph_p)^{k_p}\mathcal{W}_2(\pi_{t_{p+1}},\pi_{t_p})+B_p}\\
    &\quad \text{since} \ \theta^{(k_{{p+1}})}_{t_{p+1}}=\theta^{(0)}_{t_p},
\end{align*} where $B_p=1.65\frac{M_p}{m_p}(h_pd)^{0.5}+\frac{\delta \sqrt{d}}{m_p}+\frac{\sigma^2(h_pd)^{0.5}}{1.65M_p+\sigma\sqrt{m_p}}$, $\delta$ bounds the bias of the score estimate, $\sigma$ bounds its variance and $d$ is the dimension of the space.

Let's now simplify the bias terms $B_p$ when the score estimate is obtained by denoising score matching. In this case, we obtain a deterministic \textit{but} biased version of the exact score. 
Recall that the score estimate in \citet{dalalyan} $s_k=\nabla_{\theta}\log \pi(\theta_k)+\zeta_k$
 has a bounded bias  at any iterate $\theta_k$ of the LMC algorithm as per 
\begin{equation*}
    \mathrm{E}[\Vert\mathrm{E}(\zeta_k\mid\theta_k)\Vert^2_2]\leq \delta^2d.
\end{equation*}
If $\epsilon_\text{DSM}^2$ controls the mean square error (MSE) between the true score and its estimate in average over the LMC iterates, then we have :
\begin{align*}
    \mathrm{E}_{\theta_k}[\Vert\mathrm{E}(\zeta_k\mid\theta_k)\Vert^2_2]&=\mathrm{E}_{\theta_k}[\Vert\mathrm{E}(s_k-\nabla_{\theta}\log \pi(\theta_k)\mid\theta_k)\Vert^2_2]\\
    &\leq \mathrm{E}_{\theta_k}[\mathrm{E}(\Vert s_k-\nabla_{\theta}\log \pi(\theta_k)\Vert^2_2\mid\theta_k)]\quad \text{by Jensen inequality}\\
    &=\mathrm{E}_{\theta_k,s_k}[\Vert s_k-\nabla_{\theta}\log \pi(\theta_k)\Vert^2_2]\\
    &=\mathrm{E}_{\theta_k}[\Vert s(\theta_k)-\nabla_{\theta}\log \pi(\theta_k)\Vert^2_2] \\
    &\quad \text{if} \ s_k=s(\theta_k) \ \text{i.e. for each point the score estimate is deterministic}\\
    &\leq \epsilon_\text{DSM}^2.
\end{align*}
This is valid for any bridging distribution $\pi_t$. If $\epsilon_\text{DSM}^2$ upperbounds the maximum MSE over time (especially over levels $t_0,\ldots,t_T$), then we can set $\delta=\frac{\epsilon_\text{DSM}}{\sqrt{d}}$. Moreover, training a score estimate with denoising score-matching provides a deterministic score estimate: we can thus assume a null variance of the score estimate ($\sigma=0$) in the theoretical bound. In \citet{dalalyan}, the variance is bounded as follows:
\begin{align*}
    \mathrm{E}[\Vert \zeta_k - \mathrm{E}(\zeta_k\mid\theta_k)\Vert^2_2] \leq \frac{\sigma^2}{d}.
\end{align*} In our setting, we obtain
\begin{align*}
    \mathrm{E}_{\theta_k}[\Vert \zeta_k - \mathrm{E}(\zeta_k\mid\theta_k)\Vert^2_2]&=\mathrm{E}_{\theta_k}[\Vert s_k-\nabla_{\theta}\log \pi(\theta_k) - \mathrm{E}(s_k-\nabla_{\theta}\log \pi(\theta_k)\mid\theta_k)\Vert^2_2]\\
    &=\mathrm{E}_{\theta_k}[\Vert s(\theta_k)-\nabla_{\theta}\log \pi(\theta_k) - \mathrm{E}(s(\theta_k)-\nabla_{\theta}\log \pi(\theta_k)\mid\theta_k)\Vert^2_2] \\
    &\quad \text{when we assume that} \ s_k=s(\theta_k) \ \text{is deterministic}\\
    &=\mathrm{E}_{\theta_k}[\Vert s(\theta_k)-\nabla_{\theta}\log \pi(\theta_k) - s(\theta_k)-\nabla_{\theta}\log \pi(\theta_k)\Vert^2_2]\\
    &=0.
\end{align*}
So the bias terms simplify to $B_p=1.65\frac{M_p}{m_p}(h_pd)^{0.5}+\frac{\epsilon_\text{DSM}}{m_p}$ in the case where the score estimate is deterministic but biased, as in denoising score matching. Note that to estimate the bias of the score and compute $\epsilon_\text{DSM}$ in practice, we need to bound the score matching error in average over the annealed Langevin iterates at each annealing level.
\newpage
\subsection{Proof of Proposition~\ref{prop::global bound}}\label{proof::global_bound}
To get the global error bound in Wasserstein distance of an annealed Langevin algorithm with $T$ annealing steps, it is sufficient to compose the previous intermediate bound $T$ times, assuming that each intermediate distribution $\pi_t$ is $M_t$-smooth and $m_t$-concave for all $t=t_0,\ldots,t_{T-1}$. Here again, for readability concerns, we replace the time level $t_p$ simply by $p$ for the constants $m,M,k,h,B$. As we start from Gaussian samples $\theta_{t_T}^{(0)}\sim \mathcal{N}(0,I_d)$, we use the convention that $k_T=0$, therefore, $\theta_{t_{T-1}}^{(0)}=\theta_{t_T}^{(k_T)}=\theta_{t_T}^{(0)}$. Suppose that $h_p<\frac{2}{m_p+M_p}$ for all $p=0,\ldots,T-1$, then we get:
\begin{align*}
    \mathcal{W}&_2(\mathcal{L}(\theta_0^{(k_0)}),\pi_{t_0})\\
    &\leq (1-m_0h_0)^{k_0}  \mathcal{W}_2(\mathcal{L}(\theta_{t_1}^{(k_1)}),\pi_{t_1})+(1-m_0h_0)^{k_0}\mathcal{W}_2(\pi_{t_1},\pi_{t_0})+B_0\\
    &\leq (1-m_0h_0)^{k_0}  \left((1-m_1h_1)^{k_1}  \mathcal{W}_2(\mathcal{L}(\theta_{t_2}^{(k_2)}),\pi_{t_2})+(1-m_1h_1)^{k_1}\mathcal{W}_2(\pi_{t_2},\pi_{t_1})+B_1\right)\\
    &\quad +(1-m_0h_0)^{k_0}\mathcal{W}_2(\pi_{t_1},\pi_{t_0})+B_0\\
    &=(1-m_0h_0)^{k_0}(1-m_1h_1)^{k_1}  \mathcal{W}_2(\mathcal{L}(\theta_{t_2}^{(k_2)}),\pi_{t_2})+(1-m_0h_0)^{k_0}(1-m_1h_1)^{k_1} \mathcal{W}_2(\pi_{t_2},\pi_{t_1})\\
    &\quad +(1-m_0h_0)^{k_0} B_1 +(1-m_0h_0)^{k_0}\mathcal{W}_2(\pi_{t_1},\pi_{t_0})+B_0\\
    &=\ldots \quad \text{after several annealing steps}\\
    &\leq \prod_{s=0}^{T-2} (1-m_sh_s)^{k_s} \mathcal{W}_2(\mathcal{L}(\theta_{t_{T-1}}^{(k_{T-1})}),\pi_{t_{T-1}})+\sum_{p=0}^{T-2}\left(\prod_{s=0}^p (1-m_sh_s)^{k_s}\mathcal{W}_2(\pi_{t_{p+1}},\pi_{t_p})+\underbrace{\prod_{s=0}^{p-1} (1-m_sh_s)^{k_s}}_{=1 \ \text{if} \ p=0}B_{p}\right)\\
    &\leq \prod_{s=0}^{T-2} (1-m_sh_s)^{k_s} \left((1-m_{T-1}h_{T-1})^{k_{T-1}} \underset{=0}{\underbrace{\mathcal{W}_2(\mathcal{L}(\theta_{t_{T}}^{(k_T)}),\pi_{t_T})}} +(1-m_{T-1}h_{T-1})^{k_{T-1}}\mathcal{W}_2(\pi_{T-1},\pi_T)+B_{T-1}\right)\\
    &\quad +\sum_{p=0}^{T-2}\left(\prod_{s=0}^p (1-m_sh_s)^{k_s}\mathcal{W}_2(\pi_{t_{p+1}},\pi_{t_p})+\prod_{s=0}^{p-1} (1-m_sh_s)^{k_s}B_p\right)\\
    &\quad \text{using the bound of Proposition 1 and since $\theta_{t_T}^{(k_T)}\sim \pi_{t_T}$}\\
    &=\boxed{\sum_{p=0}^{T-1}\left(\prod_{s=0}^p (1-m_sh_s)^{k_s}\right)\mathcal{W}_2(\pi_{t_{p+1}},\pi_{t_p})+ \left(\prod_{s=0}^{p-1} (1-m_sh_s)^{k_s}\right)B_p},
\end{align*}
where $B_p=1.65\frac{M_p}{m_p}(h_pd)^{0.5}+\frac{\delta \sqrt{d}}{m_p}+\frac{\sigma^2(h_pd)^{0.5}}{1.65M_p+\sigma\sqrt{m_p}}$ for all $p=0,\ldots,T-1$. If we consider taking the same step size and number of Langevin steps per Langevin process, i.e. $h_p=h$ and $k_p=k$ for all $p$ such that $h<\min_p\frac{2}{m_p+M_p}$ and we set $m=\min_pm_p$, then we get a simplified bound :
\begin{equation*}
    \mathcal{W}_2(\mathcal{L}(\theta_0^{(k)}),\pi_{t_0})\leq\sum_{p=0}^{T-1}(1-mh)^{k(p+1)}\mathcal{W}_2(\pi_{t_{p+1}},\pi_{t_p})+(1-mh)^{kp}B_p.
\end{equation*}

\subsection{Bound of intermediate Wasserstein distances}\label{proof::bound_wass_fisher}
\begin{lemma}\label{lemma::fisher_div}
    Let $\Vert x \Vert_{A^{-1}}:=\sqrt{x^TA^{-1}x}$ be the Mahalanobis norm induced by a SPD matrix $A$. Then,
    \begin{align*}
        \mathcal{W}_2(\pi_{t+1}^G,\pi_t^G)&\leq \frac{1}{m_t^G}\Big((n-1)\sqrt{\mathrm{E}_{\theta}\big \Vert \nabla_{\theta}\log \lambda_t(\theta)-\nabla_{\theta}\log \lambda_{t+1}(\theta)\Vert^2}\\
        &\quad +\sum_{i=1}^n \sqrt{\mathrm{E}_{\theta}\Vert\nabla_{\theta}\log p_t(\theta\mid x_i)-\nabla_{\theta}\log p_{t+1}(\theta\mid x_i)\Vert^2}\Big),\\
        \mathcal{W}_2(\pi_{t+1}^L,\pi_t^L)&\leq \frac{1}{m_t^L}\Big((n-1)\sqrt{\mathrm{E}_{\theta}\big \Vert \nabla_{\theta}\log \lambda_t(\theta)-B_{t,\lambda}\nabla_{\theta}\log \lambda_{t+1}(\theta)\Vert^2_{\Psi_{t,\lambda}^{-1}}}\\
        &\quad +\sum_{i=1}^n \sqrt{\mathrm{E}_{\theta}\Vert\nabla_{\theta}\log p_t(\theta\mid x_i)-B_{t,i}\nabla_{\theta}\log p_{t+1}(\theta\mid x_i)\Vert^2_{\Psi_{t,i}^{-1}}}\Big),
    \end{align*} where $\Psi_{t,k}=\Sigma_{t,k}\Lambda_t\Lambda_t^T\Sigma_{t,k}^T$ and $B_{t,k}=\Sigma_{t,k}\Lambda_{t}\Lambda_{{t+1}}^{-1}\Sigma_{{t+1,k}}^{-1}$ with $k\in\{\lambda,0,\ldots,n\}$.
\end{lemma}
\textit{Proof}: We replace the classical level index $t_p$ by $t$ for readability concerns, with the abuse of notation that $\pi_{t+1}$ denotes the consecutive density to $\pi_t$, usually called $\pi_{t_{p+1}}$. In Proposition~\ref{prop::global bound}, we require intermediate densities $\pi_t$ to be $m_t$-log-concave for any $t=t_0,\ldots,t_T$. They thus satisfy a log-Sobolev inequality and the Talagrand inequality as per:
\begin{align*}
    KL(\pi_{t+1}\Vert \pi_{t})&\leq \frac{1}{2m_t}\mathrm{E}_{\pi_{t+1}}[\Vert \nabla_\theta \log \pi_{t+1}(\theta)-\nabla_\theta\log \pi_t(\theta)\Vert^2],\\
    \mathcal{W}_2^2(\pi_{t+1},\pi_t)&\leq \frac{2}{m_t}KL(\pi_{t+1}\Vert \pi_t),
\end{align*} yielding therefore
\begin{equation*}
    \mathcal{W}_2^2(\pi_{t+1},\pi_t)\leq \frac{1}{m_t^2}\underbrace{\mathrm{E}_{\pi_{t+1}}[\Vert \nabla_\theta \log \pi_{t+1}(\theta)-\nabla_\theta\log \pi_t(\theta)\Vert^2]}_{\text{Fisher divergence}}.
\end{equation*} Let's now bound the Fisher divergence between two consecutive intermediate densities by the Fisher divergence between individual posteriors at two consecutive annealing levels. We denote by $\Vert\cdot\Vert_{A^{-1}}$ the Mahalanobis norm induced by a SPD matrix $A$, defined as $\Vert x\Vert_{A^{-1}}=\sqrt{x^TA^{-1}x}$.
We start with the densities defined by \citet{geffner}:
\begin{align*}
    \mathrm{E}_{\theta}&\left[\Vert \nabla_{\theta}\log \pi_t^G(\theta\mid x_{1:n})-\nabla_\theta \log \pi^G_{t+1}(\theta\mid x_{1:n})\Vert^2\right]\\
    &=\mathrm{E}_{\theta}\left[\Vert(1-n)\left(\nabla_{\theta}\log \lambda_t(\theta)-\nabla_{\theta}\log \lambda_{t+1}(\theta)\right)+\sum_{i=1}^n\nabla_{\theta}\log p_t(\theta\mid x_i)-\nabla_{\theta}\log p_{t+1}(\theta\mid x_i)\Vert^2\right]\\
    &=\mathrm{E}_{\theta}\Vert(1-n)(\nabla_{\theta}\log \lambda_t(\theta)-\nabla_{\theta}\log \lambda_{t+1}(\theta))\Vert^2+\sum_{i=1}^n \mathrm{E}_{\theta}\Vert\nabla_{\theta}\log p_t(\theta\mid x_i)-\nabla_{\theta}\log p_{t+1}(\theta\mid x_i)\Vert^2\\
    &\ \ \ +2\mathrm{E}_{\theta}\big \langle (1-n)(\nabla_{\theta}\log \lambda_t(\theta)-\nabla_{\theta}\log \lambda_{t+1}(\theta)),\sum_{i=1}^n\nabla_{\theta}\log p_t(\theta\mid x_i)-\nabla_{\theta}\log p_{t+1}(\theta\mid x_i)\big \rangle\\
    &\quad +\mathrm{E}_{\theta}\sum_{i\neq j}\langle \nabla_{\theta}\log p_t(\theta\mid x_i)-\nabla_{\theta}\log p_{t+1}(\theta\mid x_i),\nabla_{\theta}\log p_t(\theta\mid x_j)-\nabla_{\theta}\log p_{t+1}(\theta\mid x_j)\rangle\\
    &\leq(1-n)^2\mathrm{E}_{\theta}\Vert\nabla_{\theta}\log \lambda_t(\theta)-\nabla_{\theta}\log \lambda_{t+1}(\theta)\Vert^2+\sum_{i=1}^n \mathrm{E}_{\theta}\Vert\nabla_{\theta}\log p_t(\theta\mid x_i)-\nabla_{\theta}\log p_{t+1}(\theta\mid x_i)\Vert^2\\
    &\ \ \ +2(n-1)\sum_{i=1}^n\mathrm{E}_{\theta}\big \Vert \nabla_{\theta}\log \lambda_t(\theta)-\log \lambda_{t+1}(\theta)\Vert\Vert\nabla_{\theta}\log p_t(\theta\mid x_i)-\nabla_{\theta}\log p_{t+1}(\theta\mid x_i)\big \Vert \\
    &\quad +\sum_{i\neq j}\mathrm{E}_{\theta}\Vert \nabla_{\theta}\log p_t(\theta\mid x_i)-\nabla_{\theta}\log p_{t+1}(\theta\mid x_i)\Vert \Vert\nabla_{\theta}\log p_t(\theta\mid x_j)-\nabla_{\theta}\log p_{t+1}(\theta\mid x_j)\Vert\\
    &\quad \text{by Cauchy-Schwarz and} \ n\geq 1\\
    &\leq(1-n)^2\mathrm{E}_{\theta}\Vert\nabla_{\theta}\log \lambda_t(\theta)-\nabla_{\theta}\log \lambda_{t+1}(\theta)\Vert^2+\sum_{i=1}^n \mathrm{E}_{\theta}\Vert\nabla_{\theta}\log p_t(\theta\mid x_i)-\nabla_{\theta}\log p_{t+1}(\theta\mid x_i)\Vert^2\\
    &\ \ \ +2(n-1)\sum_{i=1}^n\sqrt{\mathrm{E}_{\theta}\big \Vert \nabla_{\theta}\log \lambda_t(\theta)-\nabla_{\theta}\log \lambda_{t+1}(\theta)\Vert^2}\sqrt{\mathrm{E}_{\theta}\Vert\nabla_{\theta}\log p_t(\theta\mid x_i)-\nabla_{\theta}\log p_{t+1}(\theta\mid x_i)\big \Vert^2} \\
    &\quad +\sum_{i\neq j}\sqrt{\mathrm{E}_{\theta}\Vert \nabla_{\theta}\log p_t(\theta\mid x_i)-\nabla_{\theta}\log p_{t+1}(\theta\mid x_i)\Vert^2} \sqrt{\mathrm{E}_{\theta}\Vert\nabla_{\theta}\log p_t(\theta\mid x_j)-\nabla_{\theta}\log p_{t+1}(\theta\mid x_j)\Vert^2}\\
    &\quad \text{by Cauchy-Schwarz a second time}\\
    &= \Big((n-1)\sqrt{\mathrm{E}_{\theta}\big \Vert \nabla_{\theta}\log \lambda_t(\theta)-\nabla_{\theta}\log \lambda_{t+1}(\theta)\Vert^2}+\sum_{i=1}^n \sqrt{\mathrm{E}_{\theta}\Vert\nabla_{\theta}\log p_t(\theta\mid x_i)-\nabla_{\theta}\log p_{t+1}(\theta\mid x_i)\Vert^2}\Big)^2.
\end{align*}
For \citet{linhart2024diffusion}, we get the following bound:
\begin{align*}
    \mathrm{E}&_{\theta}\left[\Vert \right. \nabla_{\theta}\left. \log \pi_t^L(\theta\mid x_{1:n})-\nabla_\theta \log \pi^L_{t+1}(\theta\mid x_{1:n})\Vert^2\right]\\
    &=\mathrm{E}_{\theta}\left[\Vert(1-n)\Lambda_t^{-1} \Sigma_{t,\lambda}^{-1}\left(\nabla_{\theta}\log \lambda_t(\theta)-B_{t,\lambda}\nabla_{\theta}\log \lambda_{t+1}(\theta)\right) \color{white} \sum_{i=1}^n\Lambda_t^{-1}\right.\\
    &\quad +\left.\sum_{i=1}^n\Lambda_t^{-1}\Sigma_{t,i}^{-1}(\nabla_{\theta}\log p_t(\theta\mid x_i)-B_{t,i}\nabla_{\theta}\log p_{t+1}(\theta\mid x_i))\Vert^2\right]\\
    &\leq \Big((n-1)\sqrt{\mathrm{E}_{\theta}\big \Vert \Lambda_t^{-1}\Sigma_{t,\lambda}^{-1}(\nabla_{\theta}\log \lambda_t(\theta)-B_{t,\lambda}\nabla_{\theta}\log \lambda_{t+1}(\theta))\Vert^2}\\
    &+\sum_{i=1}^n \sqrt{\mathrm{E}_{\theta}\Vert\Lambda_t^{-1}\Sigma_{t,i}^{-1}(\nabla_{\theta}\log p_t(\theta\mid x_i)-B_{t,i}\nabla_{\theta}\log p_{t+1}(\theta\mid x_i))\Vert^2}\Big)^2\\
    &\quad \text{using the previous result on Geffner's densities}\\
    &=\left((n-1)\sqrt{\mathrm{E}_{\theta}\big \Vert \nabla_{\theta}\log \lambda_t(\theta)-B_{t,\lambda}\nabla_{\theta}\log \lambda_{t+1}(\theta)\Vert^2_{\Psi_{t,\lambda}^{-1}}}\right.
    \\
    &\left.+\sum_{i=1}^n \sqrt{\mathrm{E}_{\theta}\Vert\nabla_{\theta}\log p_t(\theta\mid x_i)-B_{t,i}\nabla_{\theta}\log p_{t+1}(\theta\mid x_i)\Vert^2_{\Psi_{t,i}^{-1}}}\right)^2,
\end{align*}where $\Psi_{t,k}=\Sigma_{t,k}\Lambda_t\Lambda_t^T\Sigma_{t,k}^T$ and $B_{t,k}=\Sigma_{t,k}\Lambda_{t}\Lambda_{{t+1}}^{-1}\Sigma_{{t+1,k}}^{-1}$ with $k\in\{\lambda,0,\ldots,n\}$.
\subsection{Fine-tuning of hyperparameters}\label{proof::fine-tune}
We want to find hyperparameters $\{h_t\}_t$ and $\{k_t\}_t$ for each level $t\in\{t_0,\ldots,t_T\}$ such that the sampling quality is lower (in Wasserstein distance) than a given threshold $\gamma>0$. We first assume that the score error is bounded by $\epsilon_\text{DSM}^2$ for all levels $t$. We will use the intermediate bound in Proposition~\ref{prop::intermediate_bound} and assume that the sampling quality of the previous Langevin process is controlled by $\gamma$, i.e. $\mathcal{W}_2(\mathcal{L}(\theta_{t_{p+1}}^{(k_{p+1})}),\pi_{t_{p+1}})\leq \gamma$. Then we get:
\begin{align*}
     \mathcal{W}_2(\mathcal{L}(\theta&_{t_p}^{(k_p)}),\pi_{t_p})\\
     &\leq \color{C0}(1-m_ph_p)^{k_p}  \left(\mathcal{W}_2(\mathcal{L}(\theta_{t_{p+1}}^{(k_{p+1})}),\pi_{t_{p+1}})+\mathcal{W}_2(\pi_{t_{p+1}},\pi_{t_p})\right)\color{black}+\color{newred}1.65\frac{M_p}{m_p}(h_pd)^{0.5}+\frac{\epsilon_\text{DSM}}{m_p}\\
     &\leq \color{C0}(1-m_ph_p)^{k_p}  \left(\gamma+\mathcal{W}_2(\pi_{t_{p+1}},\pi_{t_p})\right)\color{black}+\color{newred}1.65\frac{M_p}{m_p}(h_pd)^{0.5}+\frac{\epsilon_\text{DSM}}{m_p}.
\end{align*}
To ensure that $\mathcal{W}_2(\mathcal{L}(\theta_{t_p}^{(k_p)}),\pi_{t_p})$ is also lower than $\gamma$, it is sufficient to choose a step size $h_p$ such that the \color{newred}red \color{black}part is lower than $\omega \gamma$ with weight $\omega\in (0,1)$, then choose a number of Langevin steps $k_p$ such that the \color{C0}blue \color{black}part is lower than $(1-\omega)\gamma$. This yields 
\begin{align*}
    0\leq h_p&\leq \min\left(\frac{(\omega\gamma-\frac{\epsilon_\text{DSM}}{m_p})^2}{d(1.65)^2}\frac{m_p^2}{M_p^2},\frac{2}{m_p+M_p}\right),\\
    k_p&\geq \log\left( \frac{(1-\omega)\gamma}{\gamma+\mathcal{W}_2(\pi_{t_{p+1}},\pi_{t_p})}\right)\frac{1}{\log(1-m_ph_p)}\geq0,
\end{align*} where we replace the index $t_p$ by $p$ for readability concerns.
Note that the step size is increasing in $\omega$, while the number of steps is decreasing in $(1-\omega)$. This means that we can achieve the same Wasserstein accuracy $\gamma$ by choosing $\omega$ sufficiently large (e.g. $0.8$) so that the number of steps becomes relatively small and the step size relatively high. However, we need to ensure that $\omega \gamma>\frac{\epsilon_\text{DSM}}{m_p}$ to find a valid positive step size $h_p$.
\subsection{Proof of Lemma~\ref{lemma::preservation_cst_nongaussian} - Smoothness and log-concavity preservation}\label{proof::smoothness_preserve}
This section is split into two parts: the first one shows how smoothness is preserved over time for individual posteriors and intermediate densities under some mild assumptions; the second one studies the preservation of log-concavity.
\subsubsection{Smoothness}
Let $p$ denote a density measure sufficiently smooth. We first recall the following equivalence:
\begin{align*}
    p \ \text{is (log)} \ L\text{-smooth} &\Leftrightarrow \Vert\nabla \log p(\theta)-\nabla\log p(\tilde \theta)\Vert\leq L\Vert\theta-\tilde\theta\Vert\\
    &\Leftrightarrow -LI_d\preceq-\nabla^2\log p(\theta)\preceq LI_d,\\
    p \ \text{is log} \ L\text{-Lipschiz} &\Leftrightarrow \Vert \log p(\theta)-\log p(\tilde \theta)\Vert\leq L\Vert\theta-\tilde\theta\Vert\\
    &\Leftrightarrow -LI_d\preceq-\nabla\log p(\theta)\preceq LI_d,
\end{align*} where $L>0$ and $\theta$ belongs to the support of $p$. We also say that $p$ is Gaussian log-Lipschitz if it admits a density with respect to the Gaussian measure that is log-Lipschitz.
\begin{lemma}
    Assume that each individual posterior distribution $p(\cdot\mid x_i)$ (resp. the prior distribution $\lambda(\cdot)$) admits a density with respect to the Gaussian measure that is $L_i$-log-Lipschitz (resp. $L_\lambda$-log-Lipschitz). Then, their diffused version along the VP forward scheme $p_t(\cdot \mid x_i)$ (resp. $\lambda_t(\cdot)$) become $L_{t,i}$ (resp. $L_{t,\lambda}$) -smooth where 
\begin{align*}
    L_{t,i}&=(R_{t,i}+L_i^2\alpha_t+1) \quad \text{with} \quad R_{t,i}=5L_i\alpha_t^{\frac 3 2}\left(L_i+\frac{1}{\sqrt{\frac 1 2 \log(\frac{1}{\alpha_t})}}\right),\\
    L_{t,\lambda}&=(R_{t,\lambda}+L_\lambda^2\alpha_t+1) \quad \text{with} \quad R_{t,\lambda}=5L_\lambda\alpha_t^{\frac 3 2}\left(L_\lambda+\frac{1}{\sqrt{\frac 1 2 \log(\frac{1}{\alpha_t})}}\right).
\end{align*} Moreover, the bridging densities $\pi_t^G$ and $\pi_t^L$ (under additional symmetry conditions) are also smooth with constants depending on $L_{t,\lambda}$ and $L_{t,i}$ and defined in the proof.
\end{lemma}
Recall that individual posterior scores are learned with denoising score-matching before being used in compositional scores. Let $\nu$ denote one individual posterior and $p_0$ be its density with respect to the Lebesgue measure, i.e. $\frac{d\nu}{d\theta}=p_0$. Samples $\theta_0$ from $\nu$ are noised during a VP forward diffusion process. This defines a Markov chain $(\theta_t)_{t\geq0}$ such that 
\begin{align*}
    \theta_0&\sim \nu,\\
    \theta_t&=r_t\theta_0+\sqrt{1-r_t^2}Z \ \ \ \text{with}\ \ \ Z\sim\mathcal{N}(0,I),
\end{align*} where $r_t^2:=\alpha_t=e^{-\int_0^t\beta(s)ds}\in[0,1]$ is a VP diffusion coefficient and a strictly decreasing function of time such that $\alpha_0=1$ and $\alpha_t\to_{t\to1} 0$. The forward transition kernel can thus be written as $q_{t|0}(\theta_t \mid \theta_0)=\mathcal{N}(r_t\theta_0,(1-r_t^2)I)$. Let $f(\theta)=\frac{d\nu}{d\gamma_d}(\theta)$ denote the density of $\nu$ with respect to the standard $d$-dimensional Gaussian measure. We notice that :
\begin{align*}
    p_t(\theta_t) &= \int q_{t|0}(\theta_t \mid \theta_0) p_0(\theta_0) d \theta_0\\
    &= (1-r_t^2)^{-\frac d 2}\int \phi\left(\frac{\theta_t - r_t \theta_0}{\sqrt{1-r_t^2}}\right) p_0(\theta_0) d \theta_0 \ \ \ \text{where} \ \phi \ \text{is the density of a standard Gaussian,}\\
    &= (1-r_t^2)^{-\frac d 2}\int \phi\left(\frac{\theta_t - r_t \theta_0}{\sqrt{1-r_t^2}}\right) d\nu(\theta_0)\\
    &= (1-r_t^2)^{-\frac d 2}\int \phi\left(\frac{\theta_t - r_t \theta_0}{\sqrt{1-r_t^2}}\right) \frac{d\nu}{d\gamma_d}(\theta_0)d\gamma_d(\theta_0)\\
    &= (1-r_t^2)^{-\frac d 2}\int \phi\left(\frac{\theta_t - r_t \theta_0}{\sqrt{1-r_t^2}}\right) f(\theta_0)\phi(\theta_0)d\theta_0 \ \ \ \text{since} \ \phi(\theta)=\frac{d\gamma_d}{d\theta}(\theta)\\
    &= (1-r_t^2)^{-\frac d 2}C^{-2d}\int f(\theta_0)e^{-\frac 1 2 \frac{(\theta_t - r_t \theta_0)^2}{1-r_t^2}} e^{-\frac 1 2 \theta_0^2}d\theta_0 \\
    &\quad \text{where} \ C=\sqrt{2 \pi} \ \text{is the normalizing constant of} \ \phi,\\
    &= (1-r_t^2)^{-\frac d 2}C^{-2d}\int f(\theta_0)e^{-\frac 1 2 \left(\frac{\theta_t^2 -2r_t\theta_t\theta_0+ r_t^2 \theta_0^2}{1-r_t^2}+\theta_0^2\right)}d\theta_0\\
    &= (1-r_t^2)^{-\frac d 2}C^{-2d}\int f(\theta_0)e^{-\frac 1 2 \left(\frac{\theta_t^2 -2r_t\theta_t\theta_0+ \theta_0^2}{1-r_t^2}\right)}d\theta_0\\
    &= (1-r_t^2)^{-\frac d 2}C^{-2d}e^{-\frac 1 2 \frac{\theta_t^2}{1-r_t^2}}e^{\frac 1 2 \frac{r_t^2\theta_t^2}{1-r_t^2}}\int f(\theta_0)e^{-\frac 1 2 \left(\frac{(\theta_0-r_t\theta_t)^2}{1-r_t^2}\right)}d\theta_0\\
    &= (1-r_t^2)^{-\frac d 2}C^{-2d}e^{-\frac 1 2 \theta_t^2}\int f(\theta_0)e^{-\frac 1 2 \left(\frac{(\theta_0-r_t\theta_t)^2}{1-r_t^2}\right)}d\theta_0\\
    &= C^{-2d}e^{-\frac 1 2 \theta_t^2}\int f\left(r_t\theta_t+\sqrt{1-r_t^2}u\right)e^{-\frac 1 2 u^2}du \\
    &\quad \text{by the change of variable} \ u=\frac{\theta_0-r_t\theta_t}{\sqrt{1-r_t^2}},\\
    &= \phi(\theta_t)\int f\left(r_t\theta_t+\sqrt{1-r_t^2}u\right)\phi(u)du\\
    &= \phi(\theta_t)\int f\left(r_t\theta_t+\sqrt{1-r_t^2}u\right)d\gamma_d(u)\\
    &= \phi(\theta_t)\mathrm{E}_{Z\sim \gamma_d}\left[f\left(r_t\theta_t+\sqrt{1-r_t^2}Z\right)\right]\\
    &=\phi(\theta_t)\mathcal{Q}_tf(\theta_t),
\end{align*}
where $\mathcal{Q}_th(x):=\mathrm{E}_{Z\sim N(0,I)}\left[h(r_tx+\sqrt{1-r_t^2}Z)\right]$. We can use Corollary C.1 from \citet{gao2024gaussianinterpolationflows} with their coefficients as per $\beta_t=r_t\geq0$ and $\alpha_t=\sqrt{1-r_t^2}\in [0,1]$. As $f$ is assumed to be $L$-Lipschitz, then we get:
\begin{align*}
    (-R_t-L^2r_t^2)I_d \preceq \nabla^2\log \mathcal{Q}_tf(\theta_t)\preceq R_t I_d,
\end{align*}
where $R_t:=5Lr_t^3\left(L+\frac{1}{\sqrt{\log(\frac{1}{r_t})}}\right)$.

Then we note that:
\begin{align*}
    -\nabla_\theta^2\log p_t(\theta_t)&=-\nabla_\theta^2\log \phi(\theta_t)-\nabla_\theta^2\log\mathcal{Q}_tf(\theta_t)\\
    &=I_d-\nabla_\theta^2\log\mathcal{Q}_tf(\theta_t).
\end{align*}
Using the previous bounds, we get :
\begin{equation*}
   (1-R_t)I_d \preceq -\nabla_\theta^2\log p_t(\theta_t) \preceq (R_t+L^2r_t^2+1)I_d.
\end{equation*}
Replacing $r_t$ by $\sqrt{\alpha_t}$ in the context of VP diffusion process, we get bounds for the Hessian of all diffused individual posteriors $p_t(\cdot\mid x_i)$, meaning that they are all $L_{t,i}$ smooth where $L_{t,i}$ depends on $L_i$ as follows:
\begin{align*}
R_{t,i}&=5L_i\alpha_t^{\frac 3 2}\left(L_i+\frac{1}{\sqrt{\frac 1 2 \log(\frac{1}{\alpha_t})}}\right),\\
    L_{t,i}&=(R_{t,i}+L_i^2r_t^2+1):=(R_{t,i}+L_i^2\alpha_t+1).
\end{align*}
Let's now recall that the bridging densities $\pi_t^G$ defined by Geffner are of the form $\pi_t^G(\theta\mid x_{1:n})\propto e^{\int s^G_t(\theta\mid x_{1:n})d\theta}$ where $s^G_t$ is defined as follows:
\begin{equation*}
    s^G_t(\theta\mid x_{1:n})=(1-n)\nabla_\theta\log\lambda_t(\theta)+\sum_{i=1}^n\nabla_\theta \log p_t(\theta\mid x_i).
\end{equation*}
Consequently, we have:
\begin{align*}
    -\nabla_\theta^2\log \pi_t^G(\theta\mid x_{1:n})=-\nabla_\theta s^G_t(\theta\mid x_{1:n})=-(1-n)\nabla^2_\theta\log\lambda_t(\theta)-\sum_{i=1}^n\nabla^2_\theta \log p_t(\theta\mid x_i).
\end{align*}
Then, as $1-n\leq 0$, we get the following bounds:
\begin{align*}
    -\nabla_\theta^2\log \pi_t^G(\theta\mid x_{1:n})&\succeq (1-n)L_{t,\lambda}-\sum_{i=1}^n L_{t,i},\\
    -\nabla_\theta^2\log \pi_t^G(\theta\mid x_{1:n})&\preceq (n-1) L_{t,\lambda}+\sum_{i=1}^n L_{t,i}.
\end{align*}
So $\pi_t^G$ is $\tilde L_t^G$-smooth with constant $\tilde L_t^G=(n-1) L_{t,\lambda}+\sum_{i=1}^n L_{t,i}$. Let's now turn to Linhart's bridging densities $\pi_t^L(\theta\mid x_{1:n})$ that underlie the compositional score $s_t^L$ defined as follows:
\begin{equation*}s^L_t(\theta\mid x_{1:n})=\Lambda^{-1}\left(\sum_{i=1}^n\Sigma^{-1}_{t,i}\nabla_{\theta}\log p_t(\theta\mid x_i)+(1-n)\Sigma^{-1}_{t,\lambda}\nabla_{\theta}\log \lambda_{t}(\theta)\right),
\end{equation*}
where $\Lambda$, $\Sigma_{t,\lambda}$ and $\Sigma_{t,i}$ for all $i$ are symmetric positive definite matrices. We can then bound each of them as follows :
\begin{align*}
    \frac{1}{\sigma_{\text{max}}(\Lambda)}I \preceq \ &\Lambda^{-1} \preceq \frac{1}{\sigma_{\text{min}}(\Lambda)}I,\\
    \frac{1}{\sigma_{\text{max}}(\Sigma_{t,i})}I \preceq \ &\Sigma_{t,i}^{-1} \preceq \frac{1}{\sigma_{\text{min}}(\Sigma_{t,i})}I \quad \text{for all} \ i,\\
    \frac{1}{\sigma_{\text{max}}(\Sigma_{t,\lambda})}I \preceq \ &\Sigma_{t,\lambda}^{-1} \preceq \frac{1}{\sigma_{\text{min}}(\Sigma_{t,\lambda})}I,
\end{align*} where $\sigma_{\text{min}}(A)$ (resp. $\sigma_{\text{max}}(A)$) denotes the minimal (resp. maximal) eigenvalue of $A$. Then, 
\begin{equation*}
    -\nabla^2_{\theta}\log \pi_t^L(\theta\mid x_{1:n})=-\nabla_\theta s^L_t(\theta\mid x_{1:n}) =-\Lambda^{-1}\left(\sum_{i=1}^n\Sigma^{-1}_{t,i}\nabla^2_{\theta}\log p_t(\theta\mid x_i)+(1-n)\Sigma^{-1}_{t,\lambda}\nabla^2_{\theta}\log \lambda_{t}(\theta)\right).
\end{equation*}
If we assume that $\Sigma_{t,i}^{-1}\nabla_\theta^2 \log p_t(\theta\mid x_i)$ is symmetric for all $i$ (which is equivalent to say that both matrices commute) and as $1-n\leq 0$, we get:
\begin{align*}
     -\nabla_\theta^2\log \pi_t^L(\theta)&\succeq-\frac{1}{\sigma_{\text{min}}(\Lambda)}\left(\sum_{i=1}^n \frac{1}{\sigma_{\text{min}}(\Sigma_{t,i})}L_{t,i}+(n-1)\frac{1}{\sigma_{\text{min}}(\Sigma_{t,\lambda})}L_{t,\lambda}\right)I,\\
     -\nabla_\theta^2 \log \pi_t^L(\theta)&\preceq \frac{1}{\sigma_{\text{min}}(\Lambda)}\left(\sum_{i=1}^n \frac{1}{\sigma_{\text{min}}(\Sigma_{t,i})}L_{t,i}+(n-1)\frac{1}{\sigma_{\text{min}}(\Sigma_{t,\lambda})}L_{t,\lambda}\right)I.
\end{align*}
So $\pi_t^L$ is $\tilde L_t^L$-smooth with $\tilde L_t^L=\frac{1}{\sigma_{\text{min}}(\Lambda)}\left(\sum_{i=1}^n \frac{1}{\sigma_{\text{min}}(\Sigma_{t,i})}L_{t,i}+(n-1)\frac{1}{\sigma_{\text{min}}(\Sigma_{t,\lambda})}L_{t,\lambda}\right)$ for all $t$.
\subsubsection{Log-concavity}
\begin{lemma}
    Assume that each individual posterior (resp. the prior) is Gaussian log-Lipschitz, then it becomes smooth along the forward diffusion scheme, which is equivalent to:
    \begin{align*}
        m_{t,\lambda} I_d&\preceq -\nabla^2_\theta\log \lambda_t(\theta)\preceq M_{t,\lambda} I_d,\\
    m_{t,i} I_d&\preceq -\nabla^2_\theta\log p_t(\theta\mid x_i)\preceq M_{t,i} I_d,
    \end{align*} where $m_{t,\lambda},m_{t,i}\in \mathbb{R}$ and $M_{t,\lambda},M_{t,i}\in\mathbb{R}_+$.
    Let $\Sigma_{t,i}$ (resp. $\Sigma_{t,\lambda}$) denote the covariance matrices of the Gaussian backward kernels in Linhart's assumption. Then, assuming $\frac 1n\sum_{i=1}^n \frac{1}{\gamma_{t,i}}m_{t,i}\geq \frac{n-1}{n}\frac{1}{\gamma_{t,\lambda}}M_{t,\lambda}$ is sufficient for Geffner's bridging densities to be log-concave with $\gamma_{t,i}=\gamma_{t,\lambda}=1$. Under additional symmetry conditions, it is also sufficient for Linhart's bridging densities to be log-concave with $\gamma_{t,i}=\sigma_{\text{min}}(\Sigma_{t,i})$ (resp. $\sigma_{\text{max}}(\Sigma_{t,i})$) if $m_{t,i}\leq 0$ (resp. $\geq 0$) and $\gamma_{t,\lambda}=\sigma_{\text{min}}(\Sigma_{t,\lambda})$.
\end{lemma} 
Let's first assume that the prior $\lambda
(\cdot)$ and all individual posteriors $p(\cdot\mid x_i)$ are log-concave with constant $m_\lambda$ and $m_i$. Then their diffused versions correspond to the convolution of the scaled target density with the Gaussian transition kernel $q_{t\mid 0}(\theta_t\mid \theta_0)$:
\begin{align*}
    \lambda_t(\theta_t) &= \int q_{t|0}(\theta_t \mid \theta_0) \lambda(\theta_0) d \theta_0\\
    &= \int e^{-\frac 12 \frac{(\theta_t - \sqrt{\alpha_t} \theta_0)^2}{1-\alpha_t}}\lambda(\theta_0) d \theta_0\\
    &=\alpha_t^{-\frac{d}{2}}\int e^{-\frac 12 \frac{(\theta_t - u)^2}{1-\alpha_t}}\lambda\left(\frac{u}{\sqrt{\alpha_t}}\right) du \quad \text{by the change of variables} \ u=\sqrt{\alpha_t}\theta_0.
\end{align*}
The scaling function $x\to x\alpha_t^p$ is an affine transformation (for any $p\in\mathbb{R}$) that preserves the log-concavity of $\lambda$. Then, the convolution of two log-concave densities remains log-concave \citep{logconcavity_preservation}. So $\lambda_t$ is log-concave for all $t\in[0,1]$. The same reasoning holds for each individual posterior $p(\cdot\mid x_i)$.
So, log-concavity is preserved along the forward path. \citet{log_concavity_strengthen} also show that the weak log-concavity of the data distribution evolves into log-concavity over time during the forward diffusion scheme.\\

Assume now that, for a given time $t$, some individual posteriors are only smooth, meaning that some eigenvalues of the Hessian of their scores are still \textbf{negative}: the previous lemma indeed shows that smoothness is strengthened along the forward diffusion scheme, which translates as per:
\begin{align*}
    m_{t,\lambda} I_d&\preceq -\nabla^2_\theta\log \lambda_t(\theta)\preceq M_{t,\lambda} I_d,\\
    m_{t,i} I_d&\preceq -\nabla^2_\theta\log p_t(\theta\mid x_i)\preceq M_{t,i} I_d,
\end{align*} where $m_{t,i}\geq 0$ for $i\in P$, $m_{t,i}\leq 0$ for $i\in N$ with $N \cup P=\{1,\ldots,n\}$, $P \cap N=\emptyset$ and $m_{t,\lambda}\in \mathbb{R}$.

If we assume that $\sum_{i=1}^nm_{t,i}\geq (n-1)M_{t,\lambda}$ then using the bounds derived to prove the previous lemma, we get
\begin{align*}
    -\nabla^2_\theta\log \pi_t^G(\theta\mid x_{1:n})\succeq\left(\underbrace{(1-n)M_{t,\lambda}+\sum_{i=1}^nm_{t,i}}_{\geq0}\right)I_d,
\end{align*} which shows that Geffner's bridging densities are log-concave over time. Note that the previous condition does not require all individual posteriors to be log-concave: for $n$ sufficiently large, positive constants $m_{t,i}$ may catch up for negative ones in the sum, so that their mean remain positive.\\
Let's now take a look at Linhart's intermediate densities with the same smoothness conditions on each individual posterior and the prior. Let's assume that 
\begin{equation*}
    \sum_{i\in P} \frac{1}{\sigma_{\text{max}}(\Sigma_{t,i})}m_{t,i}+\sum_{i\in N} \frac{1}{\sigma_{\text{min}}(\Sigma_{t,i})}m_{t,i}\geq \frac{n-1}{\sigma_{\text{min}}(\Sigma_{t,\lambda})}M_{t,\lambda},
\end{equation*} then we have:
\begin{align*}
    -\nabla_\theta\log \pi_t^L(\theta)\succeq\frac{1}{\sigma_\text{max}(\Lambda)}\left(\underbrace{\sum_{i\in P}\frac{m_{t,i}}{\sigma_\text{max}(\Sigma_{t,i})}+\sum_{i\in N}\frac{m_{t,i}}{\sigma_\text{min}(\Sigma_{t,i})}+(1-n)\frac{M_{t,\lambda}}{\sigma_\text{min}(\Sigma_{t,\lambda})}}_{\geq 0}\right)I_d.
\end{align*}
Again, the previous condition does not require all individual posteriors to be log-concave (meaning that $m_{t,i}$ are positive) for Linhart's bridging density to be log-concave: negative constants may catch up for negative ones in the sum.
\subsection{Analytical bridging densities in a Gaussian setting}\label{proof::analyt_densities}
We derive here the analytical formula of the bridging densities in a Gaussian setting provided in Section~\ref{sec::analytical_densities}. Recall that we use a VP forward process to train the individual scores: therefore, all annealing levels $t=t_0,\ldots, t_T$ live in $[0,1]$ with the convention that $t_0=0$ and $t_T=1$. We start with the intermediate densities defined by \citet{linhart2024diffusion}.
We first recall the formula of the individual posteriors $p(\theta\mid x_i)=\mathcal{N}(\theta;\mu_\text{post}(x_i),\Sigma_\text{post})$ and the multi-observation posterior $p(\theta\mid x_{1:n})=\mathcal{N}(\theta;\mu_{\text{post},n}(x_{1:n}),\Sigma_{\text{post},n})$ in the Gaussian setting \citep{bishop}, where:
\begin{align*}
    \Sigma_\text{post}=(\Sigma^{-1}+I_d)^{-1} \quad &\text{and} \quad \Sigma_{\text{post},n}=(n\Sigma^{-1}+I_d)^{-1},\\
    \mu_\text{post}(x_i)=\Sigma_\text{post}\Sigma^{-1}x_i \quad &\text{and} \quad \mu_{\text{post},n}(x_{1:n})=\Sigma_{\text{post},n}\Sigma^{-1}\sum_{i=1}^nx_i.
\end{align*}
In the Gaussian setting, as all the assumptions of Linhart's work are fulfilled, their compositional score exactly corresponds to the diffused scores arising from a classical diffusion process. Thus, the underlying distributions $\pi_t^L$ are the diffused versions of the target multi-observation posterior, i.e. $\pi_t^L(\theta\mid x_{1:n})=p_t(\theta\mid x_{1:n})=\mathcal{N}(\theta,\mu_t^L, \Sigma_t^L)$ where:
\begin{align*}
\Sigma_t^L&=\alpha_t\Sigma_\text{post,n}+v_tI_d=\alpha_t(n\Sigma^{-1}+I_d)^{-1}+v_tI_d,\\
    \mu_t^L &= \sqrt{\alpha_t}\mu_\text{post}(x_{1:n})=\sqrt{\alpha_t}\Sigma_\text{post,n}\Sigma^{-1}\sum_{i=1}^n x_i=\sqrt{\alpha_t}(nI_d+\Sigma)^{-1}\sum_{i=1}^n x_i.
\end{align*}
We then consider the bridging densities defined by \citet{geffner} :
\begin{equation*}
     \pi_t^G(\theta)=\lambda_t(\theta)^{1-n}\prod_{i=1}^n p_t(\theta\mid x_i),
\end{equation*} where the diffused distributions are Gaussian of the form
\begin{align*}
    \lambda_t(\theta)&=\mathcal{N}(\theta;0,I_d) \quad \forall t\in(0,1),\\
    p_t(\theta\mid x_i)&=\mathcal{N}(\theta;\sqrt{\alpha_t}\mu_\text{post}(x_i),\underbrace{\alpha_t(\Sigma^{-1}+I_d)^{-1}+v_tI_d}_{\Sigma_{t}})\quad \text{diffused posterior under VP.}
    \end{align*}
    Then all the bridging densities are also Gaussian of the form $\pi_t^G(\theta)=\mathcal{N}(\theta;\mu_t^G,\Sigma_t^G)$
where
    \begin{align*}
 \Sigma_t^G&=\left((1-n)I_d+\sum_{i=1}^n \Sigma_{t}^{-1}\right)^{-1}= \left((1-n)I_d+n \Sigma_{t}^{-1}\right)^{-1},\\ \mu_t^G&=\Sigma_t^G\left(\sum_{i=1}^n\Sigma_{t}^{-1}\sqrt{\alpha_t}\mu_\text{post}(x_i)\right)=\sqrt{\alpha_t}\Sigma_t^G\Sigma_{t}^{-1}\Sigma_\text{post}\Sigma^{-1}\left(\sum_{i=1}^nx_i\right)\\
 &=\sqrt{\alpha_t}\Sigma_t^G\Sigma_{t}^{-1}(I_d+\Sigma)^{-1}\left(\sum_{i=1}^nx_i\right).
\end{align*}
For a well-defined Gaussian density, we need to be sure that the covariance matrix $\Sigma_t^G$ is positive definite. In our case, $\Sigma_t$ and $I_d$ commute, so do $\Sigma$ and $I_d$. Thus, we can diagonalize all matrices in a common eigenbasis and compute the eigenvalues of any linear combination of them. Let $\sigma(A)$ denotes one eigenvalue of the matrix $A$ and recall that $\sigma(A^{-1})=\frac{1}{\sigma(A)}$. Then,
\begin{align*}
    \sigma((\Sigma_t^G)^{-1})=(1-n)+n\sigma(\Sigma_{t}^{-1})&=1-n+\frac{n}{\alpha_t\frac{\sigma(\Sigma)}{\sigma(\Sigma)+1}+v_t}\\
    &=1-n+\frac{n(\sigma(\Sigma)+1)}{\sigma(\Sigma)+v_t}\\
    &=1+n\left(\frac{1-v_t}{v_t+\sigma(\Sigma)}\right)\\
    &\geq 1 \quad \text{since} \ v_t\in(0,1) \ \text{and} \ \sigma(\Sigma)>0.
\end{align*}
When $t\to 1$, $v_t\to1$ and the eigenvalues of $\Sigma_t^G$ tend to $1$, which is natural since $\pi_t^G$ tends to a standard Gaussian. When $t\to0$, then $v_t\to0$ and the eigenvalue of $\Sigma_t^G$ are equal to those of the multi-observation target posterior $\pi_0$, namely $\frac{\sigma(\Sigma)}{\sigma(\Sigma)+n}$. The covariance matrix of $\pi_t^G$ remains well defined if one chooses another prior distribution with a covariance matrix that commutes with the likelihood one.
\subsection{Proof of Lemma~\ref{lemma::preservation_cst}}\label{proof::preservation_cst}
Recall that the log-concavity constant $m_t$ corresponds to the minimal eigenvalue of the precision matrix $(\Sigma_t)^{-1}$ of the intermediate density $\pi_t$ while the smoothness constant $M_t$ is the maximal eigenvalue of the same matrix. We start with the densities defined by Linhart.
\begin{align*}
    m_0&=\min_\sigma \frac{n+\sigma}{\sigma}=\frac{n+\sigma_{\max}}{\sigma_{\max}} \quad \text{as the function decreases in}\ \sigma,\\
    m_t^L &= \min_\sigma\frac{n+\sigma}{\sigma+nv_t}=\frac{n+\sigma_{\max}}{\sigma_{\max}+nv_t} \quad \text{as the function decreases in}\ \sigma\\
    &=\frac{n+\sigma_{\max}}{\sigma_{\max}}\left(\frac{1}{1+\frac{nv_t}{\sigma_{\max}}}\right)\\
    &=m_0\left(\frac{1}{1+\frac{nv_t}{\sigma_{\max}}}\right),\\
    M_0&=\max_\sigma \frac{n+\sigma}{\sigma}=\frac{n+\sigma_{\min}}{\sigma_{\min}},\\
    M_t^L&=M_0\left(\frac{1}{1+\frac{nv_t}{\sigma_{\min}}}\right) \quad \text{with the same reasoning.}
\end{align*}
Then in Geffner's case we have
\begin{align*}
    m_t^G &= \min_\sigma\frac{\sigma+n+(1-n)v_t}{\sigma+v_t}=\frac{\sigma_{\max}+n+(1-n)v_t}{\sigma_{\max}+v_t}\quad \text{since the function decreases in}\ \sigma,\\
    &=\frac{n+\sigma_{\max}}{\sigma_{\max}}\left(\frac{1}{1+\frac{v_t}{\sigma_{\max}}}\right)+\frac{(1-n)v_t}{\sigma_{\max}+v_t}\\
    &=m_0\left(\frac{1}{1+\frac{v_t}{\sigma_{\max}}}\right)+\frac{(1-n)v_t}{\sigma_{\max}+v_t},\\
    M_t^G&=M_0\left(\frac{1}{1+\frac{v_t}{\sigma_{\min}}}\right)+\frac{(1-n)v_t}{\sigma_{\min}+v_t}\quad \text{with the same reasoning}.
\end{align*}
\subsection{Proof of Proposition~\ref{prop::comparison_cst}}\label{proof::link_cst}
For Gaussian distributions, the log-concavity constant $m_t^G$ or $m_t^L$ (resp. smoothness constant $M_t^G$ or $M_t^L$) corresponds to the minimal (resp. maximal) eigenvalue of its precision matrix $(\Sigma_t^G)^{-1}$ or $(\Sigma_t^L)^{-1}$. We first note that $\Sigma$, $\Sigma_t^L$ and $\Sigma_t^G$ commute: therefore, they share an eigen basis in which they are diagonal. This allows us to easily derive the eigenvalues of any linear combination of them, particularly the ones of $(\Sigma_t^G)^{-1}-(\Sigma_t^L)^{-1}$. Let $\sigma(A)$ denote one eigenvalue of the matrix $A$ and $\sigma$ denote one eigenvalue of the likelihood covariance.
\begin{align*}
    \sigma((\Sigma_t^G)^{-1})-\sigma((\Sigma_t^L)^{-1})&=\frac{\sigma+v_t+n\alpha_t}{\sigma+v_t}-\frac{n+\sigma}{\sigma+nv_t}\\
    &=\frac{(\sigma+v_t+n(1-v_t))(\sigma+nv_t)-(\sigma+v_t)(n+\sigma)}{(\sigma+v_t)(\sigma+nv_t)}\\
    &=\frac{nv_t^2-nv_t+n^2(v_t-v_t^2)}{(\sigma+v_t)(\sigma+nv_t)}\\
    &=\frac{n(n-1)v_t(1-v_t)}{(\sigma+v_t)(\sigma+nv_t)}\\
    &= \frac{n(n-1)v_t\alpha_t}{(\sigma+v_t)(\sigma+nv_t)}\geq 0 \quad \text{since} \ v_t=1-\alpha_t\\
    &>0 \quad \text{for} \ t\in(0,1).
\end{align*}
This shows that $(\Sigma_t^G)^{-1}-(\Sigma_t^L)^{-1} \succeq 0 \Leftrightarrow (\Sigma_t^G)^{-1}\succeq(\Sigma_t^L)^{-1}$, with equality only if $t=0$ or $t=1$. This also shows that the eigenvalues of $(\Sigma_t^G)^{-1}$ (denoted $\lambda_i^G$) are always higher than the ones of $(\Sigma_t^L)^{-1})$ (denoted $\lambda_i^L$) and that we can relate them analytically as per : $\lambda_i^G=\lambda_i^L+\frac{n(n-1)v_t\alpha_t}{(\sigma_i+v_t)(\sigma_i+nv_t)}$. Let's now study the denominator in the previous computation $D(\sigma,v_t)=(\sigma+v_t)(\sigma+nv_t)$ that is a polynomial of degree $2$ in $\sigma$:
\begin{align*}
    D&=(\sigma+v_t)(\sigma+nv_t)=\sigma^2+(n+1)v_t\sigma+nv_t^2,\\
    \Delta_{D} &= (n+1)^2v_t^2-4nv_t^2=(n-1)^2v_t^2\geq0,\\
    r_{D} &= \frac{-v_t(n+1)\pm v_t(n-1)}{2}=\left\{
\begin{array}{l}
-nv_t \leq 0\\
-v_t \leq 0
\end{array}.
\right.
\end{align*}
As $\sigma$ should be positive, this means that $D$ is positive and increases over $[0,\infty[$. So the function $g(\sigma):=\frac{n(n-1)\alpha_tv_t}{D(\sigma,v_t)}$ reaches its minimum in the maximal eigenvalue of the likelihood covariance $\sigma_\text{max}$ and its maximum in the minimal eigenvalue of the likelihood covariance $\sigma_\text{min}$. We can also check that the eigenvalues of $(\Sigma_t^L)^{-1}$ are a decreasing function $f(\sigma):=\frac{n+\sigma}{\sigma+nv_t}$ of $\sigma\in]0,+\infty[$:
\begin{equation*}
    \frac{\partial}{\partial\sigma}f(\sigma)=\frac{\sigma+nv_t-n-\sigma}{(\sigma+nv_t)^2}=\frac{n(v_t-1)}{(\sigma+nv_t)^2}\leq 0\quad \text{since} \ v_t\in [0,1].
\end{equation*} $f$ reaches its maximal and minimal values at the same points that $g$.
Therefore,
\begin{align*}
    M_t^G &= \max_\sigma \left(\frac{n+\sigma}{\sigma+nv_t}+\frac{n(n-1)v_t\alpha_t}{D(\sigma,v_t)}\right)\\
    &=\max_\sigma \frac{n+\sigma}{\sigma+nv_t}+\max_\sigma \frac{n(n-1)v_t\alpha_t}{D(\sigma,v_t)} \quad \text{since the functions reach their maximal value at the same point}\\
    &=M_t^L+\frac{n(n-1)\alpha_tv_t}{D(\sigma_\text{min},v_t)}\quad \text{since} \ D \ \text{is decreasing in} \ \sigma.
\end{align*}
Similarly,
\begin{align*}
    m_t^G &= \min_\sigma \left(\frac{n+\sigma}{\sigma+nv_t}+\frac{n(n-1)v_t\alpha_t}{D(\sigma,v_t)}\right)\\
    &=\min_\sigma \frac{n+\sigma}{\sigma+nv_t}+\min_\sigma \frac{n(n-1)v_t\alpha_t}{D(\sigma,v_t)}\\
    &\quad \text{since the functions reach their minimal value at the same point}\\
    &=m_t^L+\frac{n(n-1)\alpha_tv_t}{D(\sigma_\text{max},v_t)}\quad \text{since} \ D \ \text{is decreasing in} \ \sigma.
\end{align*}
\newpage
\subsection{Proof of Corollary~\ref{cor::comparison_ratios}}\label{proof::comparison_ratio}
Proposition~\ref{prop::comparison_cst} shows that Geffner's constants $m_t^G$ and $M_t^G$ are linked to their Linhart's counterparts. Let $N_t=n(n-1)\alpha_tv_t$, $D_{\min}=D(\sigma_{\min})$ and $D_{\max}=D(\sigma_{\max})$. Then, we have:
\begin{align*}
    \frac{M_t^G}{m_t^G}&=\frac{M_t^L+\frac{N_t}{D_\text{min}}}{m_t^L+\frac{N_t}{D_\text{max}}}\\
    &=\frac{M_t^L}{m_t^L+\frac{N_t}{D_\text{max}}}+\frac{\frac{N_t}{D_\text{min}}}{m_t^L+\frac{N_t}{D_\text{max}}}\\
    &=\frac{M_t^L}{m_t^L}\frac{1}{1+\frac{N_t}{D_\text{max}m_t^L}}+\frac{1}{m_t^L}\frac{\frac{N_t}{D_\text{min}}}{1+\frac{N_t}{m_t^LD_\text{max}}},\\
    \frac{\frac{M_t^G}{m_t^G}}{\frac{M_t^L}{m_t^L}}&=\frac{1}{1+\frac{N_t}{m_t^LD_\text{max}}}+\frac{1}{M_t^L}\frac{\frac{N_t}{D_\text{min}}}{1+\frac{N_t}{m_t^LD_\text{max}}}\\
    &=\frac{1}{1+\frac{N_t}{m_t^LD_\text{max}}}\left(1+\frac{N_t}{M_t^LD_\text{min}}\right),\\
    \frac{\frac{M_t^G}{m_t^G}}{\frac{M_t^L}{m_t^L}}\geq 1 &\Leftrightarrow 1+\frac{N_t}{m_t^LD_\text{max}}\leq1+\frac{N_t}{M_t^LD_\text{min}}\\
    &\Leftrightarrow m_t^LD_\text{max} \geq M_t^LD_\text{min}.
\end{align*}
Let's now prove that the last condition is always true for any time $t$. For a fixed time $t$, we consider the function $l(\sigma)= \frac{n+\sigma}{\sigma+nv_t}D(\sigma,v_t)=(n+\sigma)(\sigma+v_t)$ after simplification. 
\begin{align*}
    l(\sigma)&=(n+\sigma)(\sigma+v_t)=\sigma^2+\sigma(v_t+n)+nv_t,\\
    \frac{\partial}{\partial \sigma}l(\sigma)&=2\sigma+n+v_t\geq 0 \ \forall \sigma\in[0,+\infty[.
\end{align*}So $l$ is increasing in $\sigma$ over $[0,+\infty[$. Recall that $m_t^L=\min_\sigma \frac{n+\sigma}{\sigma+nv_t}=\frac{n+\sigma_{\max}}{\sigma_{\max}+nv_t}$ and $M_t^L=\max_\sigma \frac{n+\sigma}{\sigma+nv_t}=\frac{n+\sigma_{\min}}{\sigma_{\min}+nv_t}$ since the function $f:\sigma\mapsto \frac{n+\sigma}{\sigma+nv_t}$ is decreasing over $[0,+\infty[$ (see Appendix~\ref{proof::link_cst} for details). Thus, $m_t^LD(\sigma_{\max})=l(\sigma_{\max})$ and $M_t^LD(\sigma_{\min})=l(\sigma_{\min})$. As $l$ is increasing then the condition is always true. Therefore, $\frac{M_t^G}{m_t^G}$ is higher than $\frac{M_t^L}{m_t^L}$ for any time $t\in[0,1]$. and the inverse ratios are conversely ordered.
As $h_t^G$ and $h_t^L$ (chosen as in Section~\ref{sec::fine_tune}) linearly depend on these inverse ratios, it is obvious that $h_t^G\leq h_t^L$ for any time $t\in[0,1]$.
\subsection{Proof of Proposition~\ref{prop::comparaison_wass}}\label{proof::comparaison_wass}
Recall that the bridging densities $\pi_t^G$ and $\pi_t^L$ are Gaussian for all times $t\in[0,1]$. The Wasserstein distance between two consecutive densities is analytically tractable and involves a mean difference term and the Bures distance between covariance matrices, that simplifies to a Fröbenius norm as prior and likelihood covariance matrices commute: 
\begin{align*}
    \mathcal{W}_2^2(\mathcal{N}(\mu_{t_p},\Sigma_{t_p}),\mathcal{N}(\mu_{t_{p+1}},\Sigma_{t_{p+1}}))&=\underbrace{\Vert\mu_{t_p}-\mu_{t_{p+1}}\Vert^2_2}_{\text{mean difference}}+\underbrace{\operatorname{tr}(\Sigma_{t_p}+\Sigma_{t_{p+1}}-2(\Sigma_{t_p}^{\frac 12}\Sigma_{t_{p+1}}\Sigma_{t_p}^{\frac 12})^{\frac 12})}_{\text{Bures distance}}\\
    &=\Vert\mu_{t_p}-\mu_{t_{p+1}}\Vert^2_2+\underbrace{\Vert\Sigma_{t_p}^{\frac 12}-\Sigma_{t_{p+1}}^{\frac 12}\Vert_F^2}_{\text{Fröbenius norm}}.
\end{align*}
To demonstrate Proposition~\ref{prop::comparaison_wass}, we will first prove two intermediate lemmas: Lemma~\ref{lemma::compar_bures} compares the Bures distance between covariance matrices for both algorithms, while Lemma~\ref{lemma::compar_mean} studies the mean difference terms between algorithms. 
The reader is referred to Appendix~\ref{proof::final_compar_wass} for the final proof of Proposition~\ref{prop::comparaison_wass}. \\
\textbf{Notation convention}: in the rest of Appendix~\ref{proof::comparaison_wass}, we use the abuse of notation $\pi_{t+1}$ (instead of the usual $\pi_{t_{p+1}}$) to denote the intermediate density following $\pi_t$ during the annealed sampling scheme for one given $t\in\{t_0,\ldots,t_T\}\cap [0,1]$.
\subsubsection{Comparison of Bures distances}
\begin{lemma}[Comparison of Bures distances]\label{lemma::compar_bures}
    Let $\Sigma_t^L$ (resp. $\Sigma_t^G$) denote the covariance of the intermediate density defined by Linhart (resp. Geffner) for a given time $t$ and $d_B$ denote the Bures distance. Then, we have
    \begin{align*}
        \text{As} \ t\to 0, \quad d_B(\Sigma_t^G,\Sigma_{t+1}^G)&\leq d_B(\Sigma_t^L,\Sigma_{t+1}^L).\\
        \text{As} \ t\to 1, \quad d_B(\Sigma_t^G,\Sigma_{t+1}^G)&\geq d_B(\Sigma_t^L,\Sigma_{t+1}^L).
    \end{align*}
\end{lemma}
We first note that in the Gaussian setting, the intermediate densities $\pi_t$ and $\pi_{t+1}$ are Gaussian with commuting covariance matrices, regardless of the algorithm (Linhart or Geffner). Thus, the Bures distance between their corresponding covariance matrices simplify to a Fröbenius norm as per:
\begin{align*}
    Tr(\Sigma_t+\Sigma_{t+1}-2(\Sigma_t^{\frac 12}\Sigma_{t+1}\Sigma_t^{\frac 12})^{\frac 12})&=Tr(\Sigma_t+\Sigma_{t+1}-2(\Sigma_t\Sigma_{t+1})^{\frac 12})\\
    &=Tr((\Sigma_t^{\frac 12}-\Sigma_{t+1}^{\frac 12})^T(\Sigma_t^{\frac 12}-\Sigma_{t+1}^{\frac 12}))\\
    &=\Vert\Sigma_t^{\frac 12}-\Sigma_{t+1}^{\frac 12}\Vert_F^2\\
    &=\sum_{i=1}^d \left(\sqrt{\sigma_i(\Sigma_t)}-\sqrt{\sigma_i(\Sigma_{t+1})}\right)^2,
\end{align*} where $d$ is the dimension of the space and $\sigma_i(A)$ denotes the $i^{th}$ eigenvalue of the matrix $A$ in the common eigenbasis.
To compare these Bures distances, we study the sign of their difference over time as follows: 
\begin{align*}
    \sum_{i=1}^d &\left(\sqrt{\sigma_i(\Sigma_t^G)}-\sqrt{\sigma_i(\Sigma_{t+1}^G)}\right)^2-\sum_{i=1}^p \left(\sqrt{\sigma_i(\Sigma_t^L)}-\sqrt{\sigma_i(\Sigma_{t+1}^L)}\right)^2\\
    &=\sum_{i=1}^d \left(\sqrt{\sigma_i(\Sigma_t^G)}-\sqrt{\sigma_i(\Sigma_{t+1}^G)}\right)^2-\left(\sqrt{\sigma_i(\Sigma_t^L)}-\sqrt{\sigma_i(\Sigma_{t+1}^L)}\right)^2\\
    &=\sum_{i=1}^d \left(\sqrt{\sigma_i(\Sigma_t^G)}-\sqrt{\sigma_i(\Sigma_{t+1}^G)}+\sqrt{\sigma_i(\Sigma_t^L)}-\sqrt{\sigma_i(\Sigma_{t+1}^L)}\right)\\
    &\quad \cdot \left(\sqrt{\sigma_i(\Sigma_t^G)}-\sqrt{\sigma_i(\Sigma_{t+1}^G)}-\sqrt{\sigma_i(\Sigma_t^L)}+\sqrt{\sigma_i(\Sigma_{t+1}^L)}\right)\\
    &=\sum_{i=1}^d \left(\sqrt{\sigma_i(\Sigma_t^G)}+\sqrt{\sigma_i(\Sigma_t^L)}-\left(\sqrt{\sigma_i(\Sigma_{t+1}^G)}+\sqrt{\sigma_i(\Sigma_{t+1}^L)}\right)\right)\\
    &\quad \cdot\left(\sqrt{\sigma_i(\Sigma_t^G)}-\sqrt{\sigma_i(\Sigma_t^L)}-\left(\sqrt{\sigma_i(\Sigma_{t+1}^G)}-\sqrt{\sigma_i(\Sigma_{t+1}^L)}\right)\right)\\
    &=\sum_{i=1}^d \left(\color{newred}\sqrt{\sigma_i(\Sigma_t^L)}\left(1+\frac{1}{\sqrt{1+\sigma_i(\Sigma_t^L)\Omega_{t,i}}}\right)-\sqrt{\sigma_i(\Sigma_{t+1}^L)}\left(1+\frac{1}{\sqrt{1+\sigma_i(\Sigma_{t+1}^L)\Omega_{t+1,i}}}\right)\color{black}\right)\\
&\quad \cdot\left(\color{C0}\sqrt{\sigma_i(\Sigma_t^L)}\left(\frac{1}{\sqrt{1+\sigma_i(\Sigma_t^L)\Omega_{t,i}}}-1\right)-\sqrt{\sigma_i(\Sigma_{t+1}^L)}\left(\frac{1}{\sqrt{1+\sigma_i(\Sigma_{t+1}^L)\Omega_{t+1,i}}}-1\right)\color{black}\right), \tag{*}
\end{align*} where $\Omega_{t,i}=\frac{n(n-1)\alpha_tv_t}{(\sigma_i+v_t)(\sigma_i+nv_t)}\geq 0$. Recall that all matrices commute, their eigenvalues are then ordered in the common eigen basis.
\paragraph{Behavior when $t\to1$.}
We will first show that the \color{newred}red \color{black} term is negative whatever the value of $\sigma_i$ and $\Omega_i$ from a specific time step $t_0$. To do so, we will consider the function of time $t\mapsto\sqrt{\sigma(\Sigma_t^L)}\left(1+\frac{1}{\sqrt{1+\sigma(\Sigma_t^L)\Omega_{t}}}\right)$ and show that it increases in time on $[t_0,1]$.

Recall that $\sigma(\Sigma_t^L)=\frac{\sigma+nv_t}{n+\sigma}$. Thus for $n\geq1$, the eigenvalues of the posterior ($\frac{\sigma}{n+\sigma}$) are lower than $1$ and $\sigma(\Sigma_t^L)$ is \textbf{increasing} over time from $\frac{\sigma}{n+\sigma}$ to $1$. Indeed
\begin{align*}
    \frac{\partial}{\partial t}{\sigma}(\Sigma_t^L)&=\frac{n\dot v_t}{n+\sigma}\geq 0 \ \text{since} \ \dot{v_t}\geq 0 \ \forall t.
\end{align*}
Concerning the function $\Omega_t=\frac{N_t}{D_t}$ we note that :
\begin{align*}
    \frac{N_t}{D_t}&=\frac{n(n-1)\alpha_tv_t}{(\sigma+v_t)(\sigma+nv_t)}> 0 \ \forall t\in(0,1),\\
    \frac{N_1}{D_1}&=\frac{N_0}{D_0}=0 \ \text{since} \ v_0=0 \ \text{and} \  \alpha_1=0.
\end{align*}
Let $\boxed{f(t):=\sigma(\Sigma_t^L)\Omega_t}$. $f$ is a continuous and derivable function of time over the closed interval $[0,1]$, it reaches a global maximum at some time point $t_0$. Let's show that $f$ is decreasing from this time point.
\begin{align*}
    \frac{\partial}{\partial t}f(t)&=\frac{\partial}{\partial t}\left(\sigma(\Sigma_t^L)\Omega_t\right)\\
    &=\frac{\partial}{\partial t}\frac{n(n-1)\alpha_tv_t}{(n+\sigma)(v_t+\sigma)}\\
    &=\frac{n(n-1)}{n+\sigma}\frac{\partial}{\partial t}\frac{\alpha_tv_t}{v_t+\sigma}\\
    &=\frac{n(n-1)}{n+\sigma}\frac{(\dot{\alpha_t}v_t+\alpha_t\dot{v_t})(v_t+\sigma)-\dot{v_t}\alpha_tv_t}{(v_t+\sigma)^2}\\
    &=\frac{n(n-1)}{(n+\sigma)(v_t+\sigma)^2}\left[(\dot{\alpha_t}v_t-\alpha_t\dot{\alpha_t})(v_t+\sigma)+\dot{\alpha_t}\alpha_tv_t\right]\quad\text{since} \ \dot \alpha_t=-\dot v_t\\
    &=\frac{\dot{\alpha_t}n(n-1)}{(n+\sigma)(v_t+\sigma)^2}\left[(v_t-\alpha_t)(v_t+\sigma)+\alpha_tv_t\right]\\
    &=\underbrace{\frac{\dot{\alpha_t}n(n-1)}{(n+\sigma)(v_t+\sigma)^2}}_{\leq 0}\underbrace{\left[v_t^2+2\sigma v_t-\sigma\right]}_{R(v_t)}\quad \text{since} \ \alpha_t=1-v_t \ \text{and} \ \dot \alpha_t\leq 0,\\
    \Delta_R&=4\sigma^2+4\sigma\geq0,\\
    r_1&=-\sigma-\sqrt{\sigma(\sigma+1)}\leq 0 \ \text{and} \ r_2=-\sigma+\sqrt{\sigma(\sigma+1)}\in [0,\frac 12] \ \forall\sigma>0\\
    \Rightarrow& \ R(v_t)\geq 0 \ \text{for} \ v_t\geq r_2\\
\Rightarrow& \ \frac{\partial}{\partial t}f(t)\left\{
\begin{aligned}
&\leq0 \quad \text{if} \quad 1\geq v_t\geq r_2\\
&\geq 0 \quad \text{if} \quad 0\leq v_t\leq r_2
\end{aligned}
\right..
\end{align*}
As $v_t$ is strictly increasing there exists a bijection between $t\in[0,1]$ and $v_t\in[0,1]$, therefore there exists a time step $t_0$ such that $v_{t_0}=r_2$. $f(t)$ is \textbf{non negative} and thus \textbf{increases} on $[0,t_0]$ then \textbf{decreases} on $[t_0,1]$. Therefore, $t\mapsto \frac{1}{\sqrt{1+f(t)}}+1=\frac{1}{\sqrt{1+\sigma(\Sigma_t^L)\Omega_t}}+1$ is \textbf{increasing} and \textbf{non negative} on $[t_0,1]$.
We can then deduce that for $t\in[t_0,1]$:
\begin{align*}
    0\leq\frac{1}{\sqrt{1+\sigma(\Sigma_{t}^L)\Omega_{t}}}+1&\leq \frac{1}{\sqrt{1+\sigma(\Sigma_{t+1}^L)\Omega_{t+1}}}+1,\\
    \sqrt{\sigma(\Sigma_{t}^L)}\left(\frac{1}{\sqrt{1+\sigma(\Sigma_{t}^L)\Omega_{t}}}+1\right)&\leq \sqrt{\sigma(\Sigma_{t}^L)}\left(\frac{1}{\sqrt{1+\sigma(\Sigma_{t+1}^L)\Omega_{t+1}}}+1\right)\quad \text{since} \ \sigma(\Sigma_{t}^L)\geq0 \ \forall t\\
    &\leq \sqrt{\sigma(\Sigma_{t+1}^L)}\left(\frac{1}{\sqrt{1+\sigma(\Sigma_{t+1}^L)\Omega_{t+1}}}+1\right)\\\
    &\quad \text{since} \ \sigma(\Sigma_{t+1}^L) \ \text{is increasing over} \ [0,1].
\end{align*}
Thus the \color{newred}red \color{black}term is \textbf{negative} whatever the item $i$ of the sum on the interval $[t_0,1]$.

Let's now study the \color{C0} blue \color{black}term. First recall that for $t\to1$, $\sigma(\Sigma_t^L)\to 1$ and $\Omega_t\to 0$. We can derive a first order expansion at a given point $t$ close to $1$:
\begin{align*}
\sqrt{\sigma(\Sigma_t^L)}\left(\frac{1}{\sqrt{1+\sigma(\Sigma_t^L)\Omega_{t}}}-1\right)&=\sqrt{\sigma(\Sigma_t^L)}\left(1-\frac 12 \sigma(\Sigma_t^L)\Omega_{t}-1+o(\sigma(\Sigma_t^L)\Omega_t)\right)\\
&=-\frac 12\sigma(\Sigma_t^L)^\frac{3}{2}\Omega_t +o(\sigma(\Sigma_t^L)^\frac 32\Omega_t).
\end{align*}
Locally around $t$ sufficiently close to $1$, the function $\sqrt{\sigma(\Sigma_t^L)}\left(\frac{1}{\sqrt{1+\sigma(\Sigma_t^L)\Omega_{t}}}-1\right)$ behaves like $\boxed{-\frac 12g(t):=-\frac 12\sigma(\Sigma_t^L)^\frac 32\Omega_t}$. Let's study $g$ for time point near $1$.
\begin{align*}
    g(t)&:=\sigma(\Sigma_t^L)^\frac 32\Omega_t\\
    &=\frac{n(n-1)\alpha_tv_t}{(\sigma+v_t)(\sigma+nv_t)}\frac{(\sigma+nv_t)^\frac 32}{(\sigma+n)^\frac 32}\\
    &=\frac{n(n-1)\alpha_tv_t\sqrt{\sigma+v_tn}}{(\sigma+v_t)(\sigma+n)^\frac 32}\\
    &=\underbrace{\frac{n(n-1)}{(\sigma+n)^\frac 32}}_{:=C\geq0}\frac{(\alpha_t-\alpha_t^2)\sqrt{\sigma+v_tn}}{\sigma+v_t},\\
    C^{-1}\frac{\partial}{\partial t}g(t)&=\frac{(\sigma+v_t)\left[(\dot{\alpha_t}-2\dot{\alpha_t}\alpha_t)\sqrt{\sigma+nv_t}+(\alpha_t-\alpha_t^2)\frac{n\dot{v_t}}{2\sqrt{\sigma+nv_t}}\right]-\dot{v_t}(\alpha_tv_t\sqrt{\sigma+nv_t})}{(\sigma+v_t)^2}\\
    &=\frac{1}{2\sqrt{\sigma+nv_t}(\sigma+v_t)^2}\left[2(\dot{\alpha_t}-2\dot{\alpha_t}\alpha_t)(\sigma+nv_t)(\sigma+v_t)+\alpha_tv_tn\dot{v_t}(\sigma+v_t)-2\alpha_tv_t\dot{v_t}(\sigma+nv_t)\right]\\
    &=\underbrace{\frac{\dot{\alpha_t}}{2\sqrt{\sigma+nv_t}(\sigma+v_t)^2}}_{\leq 0}\underbrace{\left[2(1-2\alpha_t)(\sigma+nv_t)(\sigma+v_t)-\alpha_tv_tn(\sigma+v_t)+2\alpha_tv_t(\sigma+nv_t)\right]}_{Q(v_t)}.
    \end{align*}
    Let's study the polynmial $Q$ over time:
\begin{align*}
    Q(v_t)&=2(2v_t-1)(\sigma+nv_t)(\sigma+v_t)-v_t(1-v_t)n(\sigma+v_t)+2v_t(1-v_t)(\sigma+nv_t)\\
    &=3nv_t^3+v_t^2(5n\sigma+2\sigma-n)+v_t(4\sigma^2-3n\sigma)-2\sigma^2,\\
    Q(0)&=-2\sigma^2\leq 0 \quad \text{and} \quad Q(1)=2n+2n\sigma+2\sigma^2+2\sigma\geq 0.
\end{align*}
As $Q$ is a polynomial in $v_t$ of order $3$ it has at most $3$ roots. Moreover, with regard to its values at the extreme points of the interval $[0,1]$, $Q$ has exactly $1$ or $3$ roots in $[0,1]$. Let $q_3$ be the largest root in $[0,1]$ of $Q$. $Q$ is necessarily positive from this root, otherwise, there would exist a largest root in $[0,1]$. Let $t_1\in[0,1]$ such that $v_{t_1}=q_3$. Then, $Q$ is positive on $[t_1,1]$, i.e. the function $g$ is decreasing on $[t_1,1]$ and $g'(t)\neq 0$ on $(t_1,1)$. Thus, $t\mapsto -\frac 12 g(t)$ increases on $[t_1,1]$. So for $t$ sufficiently close to $1$, we have $t\geq t_1$ and the function $t\mapsto\sqrt{\sigma(\Sigma_t^L)}\left(\frac{1}{\sqrt{1+\sigma(\Sigma_t^L)\Omega_{t}}}-1\right)$ locally behaves like $t\mapsto -\frac 12 g(t)$ thus increases locally around $t$. As this reasoning is valid for any point in a neighborhood of $1$, denoted $[t_2,1]$, we can deduce that the latter function follows an increasing trend on $[t_2,1]$. Therefore, the \color{C0} blue \color{black} term is \textbf{negative} on $[t_2,1]$. This means that there exists a time point $\max(t_0,t_2)\leq 1$ such that all terms in the sum of equation $(*)$ are non negative. Therefore, \textbf{the Bures distance for Geffner's covariance matrices $d_B(\Sigma_t^G,\Sigma_{t+1}^G)$ is always greater than the Bures distance for Linhart's covariance matrices $d_B(\Sigma_t^L,\Sigma_{t+1}^L)$ as $t\to1$}. \\
 
\paragraph{Behavior when $t\to0$.}
 Let's look at the function $\boxed{h(t)=\frac{\sigma(\Sigma_t^L)}{1+\sigma(\Sigma_t^L)\Omega_t}}$ on $[0,1]$.
 \begin{align*}
     h(t)&=\frac{\frac{\sigma+v_tn}{\sigma+n}}{1+\frac{n(n-1)\alpha_tv_t}{(\sigma+v_t)(\sigma+nv_t)}\frac{\sigma+v_tn}{\sigma+n}}\\
     &=\frac{(\sigma+v_tn)(\sigma+v_t)}{(\sigma+n)(\sigma+v_t)+n(n-1)\alpha_tv_t}\\
     &=\frac{\sigma^2+v_t\sigma+v_tn\sigma+nv_t^2}{\sigma^2+\sigma v_t+n\sigma+nv_t+n(n-1)\alpha_tv_t},\\
     \frac{\partial}{\partial t}h(t)&=\underbrace{\frac{\dot{v_t}}{...}}_{\geq 0}(\sigma+n\sigma+2nv_t)(\sigma^2+\sigma v_t+n\sigma+nv_t+n(n-1)\alpha_tv_t)\\
     &\quad -\underbrace{\frac{\dot{v_t}}{...}}_{\geq 0}(\sigma^2+v_t\sigma+v_tn\sigma+nv_t^2)(\sigma+n-n(n-1)v_t+n(n-1)\alpha_t)\\
     &=\underbrace{\frac{\dot{v_t}}{...}}_{\geq 0}(v_t^2n^3(1+\sigma)+2v_t\sigma n^2(1+\sigma)+n\sigma^2(1+\sigma))\\
     &\geq 0 \quad \forall t\in [0,1] \ \text{since} \ \sigma>0, n\geq 1, v_t\in[0,1].
 \end{align*}
 We voluntarily omit the denominator in the previous computations as it is always positive and does not change the sign of the derivative.
 The function $h$ is positive and increases over the interval $[0,1]$, as does $\sqrt{h(t)}$. As the function $\sigma(\Sigma_t^L)$ also increases over $[0,1]$, the \color{newred} red \color{black} term is always negative for any value $\sigma_i$ and $\Omega_i$ of the sum.
 Let's now look at the \color{C0} blue \color{black} term when $t\to 0$. We know that the function $f(t)=\sigma(\Sigma_t^L)\Omega_t$ increases on $[0,t_0]$ (see previous analysis when $t\to1$). Thus, the function $t\mapsto \frac{1}{\sqrt{1+\sigma(\Sigma_t^L)\Omega_t}}-1$ is decreasing and negative on $[0,t_0]$. We can then derive on $[0,t_0]$:
 \begin{align*}
     \frac{1}{\sqrt{1+\sigma(\Sigma_{t+1}^L)\Omega_{t+1}}}-1&\leq \frac{1}{\sqrt{1+\sigma(\Sigma_t^L)\Omega_t}}-1\leq 0,\\
     \sqrt{\sigma(\Sigma_{t+1}^L)}\left(\frac{1}{\sqrt{1+\sigma(\Sigma_{t+1}^L)\Omega_{t+1}}}-1\right)&\leq \sqrt{\sigma(\Sigma_{t+1}^L)}\left(\frac{1}{\sqrt{1+\sigma(\Sigma_t^L)\Omega_t}}-1\right)\\
     &\quad \text{since} \ \sigma(\Sigma_{t+1}^L)\geq0 \ \forall t\\
     &\leq \sqrt{\sigma(\Sigma_{t}^L)}\left(\frac{1}{\sqrt{1+\sigma(\Sigma_t^L)\Omega_t}}-1\right)\\
     &\quad \text{since} \ \sigma(\Sigma_{t}^L) \ \text{increases on} \ [0,1].
 \end{align*}
 Then the \color{C0}blue \color{black} term is \textbf{positive} in the sum on $[0,t_0]$ whatever the value of $\sigma_i$ and $\Omega_i$. This means that all terms in the sum are negative for $t\in[0,t_0]$, i.e. \textbf{as $t\to0$, the Bures distance between Geffner's covariance matrices $d_B(\Sigma_t^G, \Sigma_{t+1}^G)$ is lower than the same distance between Linhart's matrices} $d_B(\Sigma_t^L, \Sigma_{t+1}^L)$.
 \subsubsection{Comparison of mean differences}
 \begin{lemma}[Comparison of mean differences]\label{lemma::compar_mean}
Let $\mu_t^L$ (resp. $\mu_t^G$) denote the mean of the intermediate density defined by Linhart (resp. Geffner) for a given time $t$. Then, we have
    \begin{align*}
        \text{As} \ t\to 0, \quad \Vert \mu_t^G-\mu_{t+1}^G\Vert^2_2&\leq \Vert \mu_t^L-\mu_{t+1}^L\Vert^2_2.\\
        \text{As} \ t\to 1, \quad \Vert \mu_t^G-\mu_{t+1}^G\Vert^2_2&\geq \Vert \mu_t^L-\mu_{t+1}^L\Vert^2_2.
    \end{align*}
\end{lemma}
Let $A_t=\sqrt{\alpha_t}(n\Sigma^{-1}+I_d)^{-1}$ and $B_t=\sqrt{\alpha_t}\Sigma_t^G\Sigma_t^{-1}(\Sigma^{-1}+I_d)^{-1}$ be respectively the matrices appearing in the computation of the mean differences derived in Section~\ref{sec::analytical_densities} and Appendix~\ref{proof::analyt_densities}.
\begin{align*}
    \Vert\mu_t^L-\mu_{t+1}^L\Vert^2-\Vert\mu_t^G-\mu_{t+1}^G\Vert^2&=\Vert (A_t-A_{t+1})\underbrace{\Sigma^{-1}\sum_{i=1}^n x_i}_{:=y}\Vert^2-\Vert (B_t-B_{t+1})\underbrace{\Sigma^{-1}\sum_{i=1}^n x_i}_{:=y}\Vert^2\\
    &=y^T(A_t-A_{t+1})^T(A_t-A_{t+1})y-y^T(B_t-B_{t+1})^T(B_t-B_{t+1})y\\
    &=y^T\left((A_t-A_{t+1})^T(A_t-A_{t+1})-(B_t-B_{t+1})^T(B_t-B_{t+1})\right)y.
\end{align*}
We will show that the matrix $(A_t-A_{t+1})^T(A_t-A_{t+1})$ is greater (in the sense of eigenvalues comparison) than $(B_t-B_{t+1})^T(B_t-B_{t+1})$ when $t\to0$ and lower when $t\to 1$.
Recall that:
\begin{align*}
    \sigma(A_t)=\sqrt{\alpha_t}\frac{\sigma}{\sigma+n} \quad &\text{and} \quad \sigma(B_t)=\sqrt{\alpha_t}\frac{\sigma}{\sigma+n+v_t(1-n)}.\\
\end{align*} These eigenvalues can be easily derived as all matrices commute and can be diagonalized in the same eigen basis. If we take the derivative with respect to time, we get:
\begin{align*}
    \frac{\partial}{\partial t}\sigma(A_t)&=\frac{\dot{\alpha_t}}{2\sqrt{\alpha_t}}\frac{\sigma}{\sigma+n}\leq 0 \ \forall t\quad \text{since} \ \dot \alpha_t\leq 0 \ \forall t\\
    \Leftrightarrow\ &A_t-A_{t+1}\succeq0 \ \forall t \ \text{(decreasing trend in time),}\\
    \frac{\partial}{\partial t}\sigma(B_t)&=\underbrace{\frac{\dot{\alpha_t}\sigma}{2\sqrt{\alpha_t}(\sigma+n+v_t(1-n))^2}}_{\leq 0 \ \forall t}(\sigma+2-n+v_t(n-1))\\
    &\leq0 \ \text{if} \ v_t\geq\frac{n-2-\sigma}{n-1}\\
    \Leftrightarrow& \left\{
\begin{array}{l}
B_t-B_{t+1}\preceq0 \ \text{if} \ 0\leq v_t\leq\max(0,\frac{n-2-\sigma}{n-1})\\
B_t-B_{t+1}\succeq0 \ \text{if} \ 1\geq v_t\geq\frac{n-2-\sigma}{n-1}
\end{array}\right.. 
\end{align*}
Note that $\frac{n-2-\sigma}{n-1}\leq 1$ for all $n\geq1$ and $\sigma>0$.
In the following, we will study the sign of $\sigma(A_t-A_{t+1})-\sigma(B_t-B_{t+1})$ over time to determine the behavior of the mean terms over time.\\

\paragraph{Behavior when $t\to1$.}
In this case, $v_t\to1$, thus $A_t-A_{t+1}\succeq0$ and $B_t-B_{t+1}\succeq0$. As $A_t$ and $B_t$ commute, we have 
\begin{align*}
    \sigma(A_t-B_t) &= \sqrt{\alpha_t}\frac{\sigma}{\sigma+n}-\sqrt{\alpha_t}\frac{\alpha_t\sigma+v_t(\sigma+1)}{v_t(\sigma+1)+(1-n)\alpha_t\sigma+n\alpha_t(\sigma+1)}\frac{\sigma+1}{\alpha_t\sigma+v_t(\sigma+1)}\frac{\sigma}{\sigma+1}\\
    &=\sqrt{\alpha_t}\frac{\sigma}{\sigma+n}-\sqrt{\alpha_t}\frac{\sigma+v_t}{\sigma+v_t+n\alpha_t}\frac{\sigma}{\sigma+v_t}\\
    &=\sqrt{\alpha_t}\frac{\sigma}{\sigma+n}-\sqrt{\alpha_t}\frac{\sigma}{\sigma+v_t+n\alpha_t}\\
    &=\frac{\sigma}{\sigma+n}\underbrace{\sqrt{\alpha_t}\left(1-\frac{\sigma+n}{\sigma+v_t+n\alpha_t}\right)}_{:=e(t)},\\
    \frac{\partial}{\partial t}e(t)&=\frac{\dot{\alpha_t}}{2\sqrt{\alpha_t}}\left(1-\frac{\sigma+n}{\sigma+v_t+n\alpha_t}\right)+\sqrt{\alpha_t}(\sigma+n)\frac{\dot{v_t}+n\dot{\alpha_t}}{(\sigma+v_t+n\alpha_t)^2}\\
    &=\underbrace{\frac{\dot{\alpha_t}}{2\sqrt{\alpha_t}(\sigma+v_t+n\alpha_t)^2}}_{\leq 0}\left(\underbrace{(\sigma+v_t+n\alpha_t)^2-(\sigma+n)(\sigma+v_t+n\alpha_t)+2\alpha_t(\sigma+n)(-1+n)}_{S(v_t)}\right),\\
    S(v_t)&=v_t^2(n-1)^2+v_t(3n-3n^2+3\sigma-3n\sigma)+2n^2-2n+2n\sigma-2\sigma\\
    &=v_t^2(n-1)^2-3v_t(n-1)(\sigma+n)+2(n-1)(\sigma+n)\\
    &=(n-1)(v_t^2(n-1)-3v_t(\sigma+n)+2(\sigma+n)),\\
    S(0)&=2(n-1)(\sigma+n)\geq 0,\\
    S(1)&=(1+\sigma)(1-n)\leq 0.
\end{align*}
$S$ is a polynomial of degree $2$ in $v_t$. Given the values reached by $S$ on the bounds of $[0,1]$, it has an unique root in this interval, denoted $r_0$. Let $t_3$ such that $v_{t_3}=r_0$. Then, $S$ is positive on $[0,r_0]$ and negative on $[r_0,1]$. Therefore, the function $e$ decreases on $[0,t_3]$, then increases on $[t_3,1]$. This is equivalent to say that :
\begin{align*}
&\left\{
\begin{array}{l}
\sigma(A_t-B_t)\geq \sigma (A_{t+1}-B_{t+1}) \quad \text{if} \ t\in[0,t_3]\\
\sigma(A_t-B_t)\leq \sigma(A_{t+1}-B_{t+1}) \quad \text{if} \ t\in[t_3,1]
\end{array}\right.\\
&\text{Since all matrices commute, this translates as per}\\
\Leftrightarrow \ &\left\{
\begin{array}{l}
\sigma(A_t-A_{t+1})\geq \sigma (B_{t}-B_{t+1}) \quad \text{if} \ t\in[0,t_3]\\
\sigma(A_t-A_{t+1})\leq \sigma(B_{t}-B_{t+1}) \quad \text{if} \ t\in[t_3,1]
\end{array}\right..
\end{align*}
Thus, when $t\in [\max(t_3,\frac{n-2-\sigma}{n-1}),1]$, we have $0\leq \sigma(A_t-A_{t+1})\leq \sigma(B_t-B_{t+1})$, therefore
\begin{align*}
    0\leq \underbrace{\sigma((A_t-A_{t+1})^T(A_t-A_{t+1}))}_{=(\sigma(A_t-A_{t+1}))^2}&\leq \underbrace{\sigma((B_t-B_{t+1})^T(B_t-B_{t+1}))}_{=(\sigma(B_t-B_{t+1}))^2}\\
    \Leftrightarrow \ 0\preceq(A_t-A_{t+1})^T(A_t-A_{t+1})&\preceq (B_t-B_{t+1})^T(B_t-B_{t+1})\\
    \Leftrightarrow \ \Vert\mu_t^L-\mu_{t+1}^L\Vert^2&\leq \Vert\mu_t^G-\mu_{t+1}^G\Vert^2.
\end{align*}
\textbf{In the neighborhood of $t=1$, Linhart's bridging densities always have a lower mean difference $\Vert \mu_t^L-\mu_{t+1}^L\Vert^2_2$ than Geffner's counterparts} $\Vert \mu_t^G-\mu_{t+1}^G\Vert^2_2$ .\\

\paragraph{Behavior when $t\to0$.}
We consider the time steps $t$ such that $0\leq v_t\leq \max(0,\frac{n-2-\sigma}{n-1})$. Let's denote this time interval $[0,\tilde t_4]$. In this case, $A_t-A_{t+1}\succeq0$ and $B_t-B_{t+1}\preceq0$. Note that if $\max(0,\frac{n-2-\sigma}{n-1})=0$, then we turn back to the previous paragraph and study the case when $t\in[0,t_3]$. We have
\begin{equation*}
    |\sigma(A_t-A_{t+1})|-|\sigma(B_t-B_{t+1})|=\sigma(A_t-A_{t+1})+\sigma(B_t-B_{t+1}).
\end{equation*}
Let's study the function of time $t\mapsto \sigma(A_t+B_t)$ and show that it decreases in the neighborhood of $t=0$.
\begin{align*}
    \sigma(A_t+B_t)&=\sqrt{\alpha_t}\frac{\sigma}{\sigma+n}+\sqrt{\alpha_t}\frac{\sigma}{\sigma+v_t+n\alpha_t}\\
    &=\frac{\sigma}{\sigma+n}\underbrace{\sqrt{\alpha_t}\left(1+\frac{\sigma+n}{\sigma+v_t+n\alpha_t}\right)}_{=d(t)}\geq 0 \ \forall t,\\
    \frac{\partial}{\partial t}d(t)&=\frac{\dot{\alpha_t}}{2\sqrt{\alpha_t}}\left(1+\frac{\sigma+n}{\sigma+v_t+n\alpha_t}\right)-\sqrt{\alpha_t}(\sigma+n)\frac{\dot{v_t}+n\dot{\alpha_t}}{(\sigma+v_t+n\alpha_t)^2}\\
    &=\frac{\dot{\alpha_t}}{2\sqrt{\alpha_t}}\left(1+\frac{\sigma+n}{\sigma+v_t+n\alpha_t}-2\alpha_t(\sigma+n)\frac{-1+n}{(\sigma+v_t+n\alpha_t)^2}\right)\\
    &=\frac{\dot{\alpha_t}}{2\sqrt{\alpha_t}(\sigma+v_t+n\alpha_t)^2}\left((\sigma+v_t+n\alpha_t)^2+(\sigma+n)(\sigma+v_t+n\alpha_t)-2\alpha_t(n-1)(\sigma+n)\right)\\
    &=\underbrace{\frac{\dot{\alpha_t}}{2\sqrt{\alpha_t}(\sigma+v_t+n\alpha_t)^2}}_{\leq 0}(\underbrace{v_t^2(n-1)^2+v_t(\sigma+n)(1-n)+2(\sigma+n)(\sigma+1)}_{T(v_t)}),\\
    T(0)&=2(\sigma+n)(\sigma+1)\geq 0,\\
    T(1)&=1+3\sigma+n+n\sigma+2\sigma^2\geq 0,\\
    \Delta_T&=(n-1)^2(\sigma+n)(n-7\sigma-8).
\end{align*}
$T$ is a polynomial of degree $2$ in $v_t$. If $n-7\sigma-8\leq 0$, then $T$ has at most one root and stays positive on $[0,1]$. Otherwise, there exists $2$ roots $r_1, r_2$; either both belong to the interval $[0,1]$, or neither does. In the second case, $T$ remains positive over $[0,1]$. In the first case, $T$ is positive for $0\leq v_t\leq r_1$ and $1\geq v_t\geq r_2$. In both cases, if $r_1\geq 0$, $T$ is positive on $[0,\min(r_1,1)]$. So there exists a time steps $t_4$ such that $v_{t_4}=\min(r_1,1)$ and the function $d$ is decreasing on $[0,t_4]$. This is equivalent to say that 
\begin{align*}
    \sigma(A_t+B_t)&\geq \sigma(A_{t+1}+B_{t+1})\geq 0 \ \text{if} \ t\in[0,t_4]\\
    \Leftrightarrow \ \sigma(A_t-A_{t+1})&\geq \sigma(B_{t+1}-B_t)\geq 0 \quad \text{since all matrices commute}\\
    \Leftrightarrow \ |\sigma(A_t-A_{t+1})|&\geq |\sigma(B_t-B_{t+1})|\geq 0 \ \text{if} \ t\in[0,\min(t_4,\tilde t_4)]\\
    \Leftrightarrow \ \underbrace{\sigma((A_t-A_{t+1})^T(A_t-A_{t+1}))}_{=|\sigma(A_t-A_{t+1})|^2}&\geq \underbrace{\sigma((B_t-B_{t+1})^T(B_t-B_{t+1}))}_{=|\sigma(B_t-B_{t+1})|^2}\geq 0 \ \text{if} \ t\in[0,\min(t_4,\tilde t_4)]\\
    \Leftrightarrow \ \Vert \mu_t^L-\mu_{t+1}^L\Vert^2&\geq\Vert \mu_t^G-\mu_{t+1}^G\Vert^2.
\end{align*}
\textbf{In the neighborhood of $t=0$, Linhart's bridging densities always have a greater mean difference $\Vert \mu_t^L-\mu_{t+1}^L\Vert^2_2$ than Geffner's counterparts} $\Vert \mu_t^G-\mu_{t+1}^G\Vert^2_2$.
\subsubsection{Final derivation}\label{proof::final_compar_wass}
We first state again the proposition before deriving its proof. The proof uses the same notation convention as in the previous subsections for readability concerns: $\pi_{t+1}$ will denote $\pi_{t_{p+1}}$ and $\pi_t$ stands for $\pi_{t_p}$, for one given $p=0,\ldots,T$.
\CompWass*
\textit{Proof}:\\
\underline{Behavior when $t\to 1$}:

From Lemma~\ref{lemma::compar_bures}, there exists a time step $\max(t_0,t_2)\leq 1$ such that $d_B(\Sigma_t^G,\Sigma_{t+1}^G)\geq d_B(\Sigma_t^L,\Sigma_{t+1}^L)$ on $[\max(t_0,t_2),1]$.\\
From Lemma~\ref{lemma::compar_mean}, there exists a time step $t_3\leq 1$ such that $\Vert \mu_t^G - \mu_{t+1}^G\Vert^2_2\geq\Vert \mu_t^L - \mu_{t+1}^L\Vert^2_2$ on $[t_3,1]$.
Therefore, as we work with Gaussian distributions, for $t\in [\max(t_0,t_2,t_3),1]$, $\mathcal{W}_2^2(\pi_t^G,\pi_{t+1}^G)\geq \mathcal{W}_2^2(\pi_t^L,\pi_{t+1}^L)$.

\underline{Behavior when $t\to 0$}: 

From Lemma~\ref{lemma::compar_bures}, there exists a time step $0\leq t_0\leq 1$ such that $d_B(\Sigma_t^G,\Sigma_{t+1}^G)\leq d_B(\Sigma_t^L,\Sigma_{t+1}^L)$ on $[0,t_0]$.\\
From Lemma~\ref{lemma::compar_mean}, there exists a time step $0\leq \min (\tilde t_4, t_4)\leq 1$ such that $\Vert \mu_t^G - \mu_{t+1}^G\Vert^2_2\leq\Vert \mu_t^L - \mu_{t+1}^L\Vert^2_2$ on $[0, \min (\tilde t_4, t_4)]$.\\
Therefore, as we work with Gaussian distributions, for $t\in [0,\min(t_0,\tilde t_4,t_4)]$, $\mathcal{W}_2^2(\pi_t^G,\pi_{t+1}^G)\leq \mathcal{W}_2^2(\pi_t^L,\pi_{t+1}^L)$.

When $t\to0$ and $t\to1$, the choice of the number of steps prescribed in Section~\ref{sec::fine_tune} mainly depends on the intermediate Wasserstein distances, as the log-concavity constants $m_t^G$ and $m_t^L$ tend to have similar values in these regimes. Suppose that we choose the same step size for both algorithms, denoted $h_t$. Let's now show that the number of Langevin steps follows the growth trend of the Wasserstein distance between bridging densities. When $t\to0$, then $0\leq \mathcal{W}_2(\pi_t^G,\pi_{t+1}^G)\leq \mathcal{W}_2(\pi_t^L,\pi_{t+1}^L)$ and we have
\begin{align*}
0\leq \frac 1 2 \frac{\gamma}{\gamma+\mathcal{W}_2(\pi_{t+1}^L,\pi_t^L)}&\leq \frac 1 2 \frac{\gamma}{\gamma+\mathcal{W}_2(\pi_{t+1}^G,\pi_t^G)}\\
\Leftrightarrow \log\left(\frac 1 2 \frac{\gamma}{\gamma+\mathcal{W}_2(\pi_{t+1}^L,\pi_t^L)}\right)&\leq \log\left(\frac 1 2 \frac{\gamma}{\gamma+\mathcal{W}_2(\pi_{t+1}^G,\pi_t^G)}\right)\leq 0 \\
  \Leftrightarrow \log\left(\frac 1 2 \frac{\gamma}{\gamma+\mathcal{W}_2(\pi_{t+1}^L,\pi_t^L)}\right)\underbrace{\frac{1}{\log(1-m_t^Gh_t)}}_{\leq 0}&\geq \log\left(\frac 1 2 \frac{\gamma}{\gamma+\mathcal{W}_2(\pi_{t+1}^G,\pi_t^G)}\right)\underbrace{\frac{1}{\log(1-m_t^Gh_t)}}_{\leq 0}\geq 0.
\end{align*}
Note that $m_t^G\geq m_t^L$ for all time $t$. Thus,
\begin{align*}
    0\leq 1-m_t^Gh_t&\leq 1-m_t^Lh_t\leq 1,\\
    0\geq \frac{1}{\log(1-m_t^Gh_t)}&\geq \frac{1}{\log(1-m_t^Lh_t)}. 
\end{align*}
Thus, we get
\begin{align*}
    \underbrace{\log\left(\frac 1 2 \frac{\gamma}{\gamma+\mathcal{W}_2(\pi_{t+1}^L,\pi_t^L)}\right)}_{\leq 0}\frac{1}{\log(1-m_t^Lh_t)}&\geq \underbrace{\log\left(\frac 1 2 \frac{\gamma}{\gamma+\mathcal{W}_2(\pi_{t+1}^L,\pi_t^L)}\right)}_{\leq 0}\frac{1}{\log(1-m_t^Gh_t)}\geq 0\\
    \Rightarrow \log\left(\frac 1 2 \frac{\gamma}{\gamma+\mathcal{W}_2(\pi_{t+1}^L,\pi_t^L)}\right)\frac{1}{\log(1-m_t^Lh_t)}&\geq \log\left(\frac 1 2 \frac{\gamma}{\gamma+\mathcal{W}_2(\pi_{t+1}^G,\pi_t^G)}\right)\frac{1}{\log(1-m_t^Gh_t)}\\
    \Rightarrow k_t^L\geq k_t^G.
\end{align*}
Then, when $t\to1$, both log-concavity constants tend to $1$, so consider $m_t=\min(m_t^G,m_t^L)$ so that the choice of the number of steps only depends on the intermediate Wasserstein distances. As $\mathcal{W}_2(\pi_t^G,\pi_{t+1}^G)\geq \mathcal{W}_2(\pi_t^L,\pi_{t+1}^L)$ when $t\to 1$, then we have :
\begin{align*}
    \log\left(\frac 1 2 \frac{\gamma}{\gamma+\mathcal{W}_2(\pi_{t+1}^L,\pi_t^L)}\right)\underbrace{\frac{1}{\log(1-m_th_t)}}_{\leq 0}&\leq \log\left(\frac 1 2 \frac{\gamma}{\gamma+\mathcal{W}_2(\pi_{t+1}^G,\pi_t^G)}\right)\underbrace{\frac{1}{\log(1-m_th_t)}}_{\leq 0}\\
    \Rightarrow k_t^G\geq k_t^L.
\end{align*}
\newpage
\subsection{Proof of Lemma~\ref{lemma::composit_error}}\label{proof::compositional_error}
Let $s_\phi$ (resp. $s_\lambda$) be an amortized neural-based approximation of the score of the individual posteriors $p(\theta\mid x_i)$ (resp. the prior $\lambda(\theta)$). Assume that $\mathrm{E}_\theta\left[\Vert \nabla_{\theta}\log \lambda_t(\theta)-s_\lambda(\theta,t)\Vert^2\right]\leq \epsilon_{\text{DSM},\lambda}^2$ and $\mathrm{E}_\theta\left[\Vert \nabla_{\theta}\log p_t(\theta\mid x_i)-s_\phi(\theta,x_i,t)\Vert^2\right]\leq \epsilon_{\text{DSM}}^2$ for all $i=1,\ldots,n$ and a given $t\in[0,1]$. We denote by $\tilde s_t^G$, the composite score defined in Equation~\ref{eq::geffner_composite_score}, where individual posterior scores are inexact. Then, we have:
\begin{align*}
    \mathrm{E}_{\theta}\Big(
    \big\Vert \nabla_{\theta}&\log \pi_t(\theta\mid x_{1:n})
    -\tilde s_{t}^G(\theta\mid x_{1:n})\big\Vert^2\Big)
    \\
    &=
    \mathrm{E}_{\theta}\Bigg(
    \bigg\Vert
    (1-n)\left(\nabla_{\theta}\log \lambda_t(\theta)-s_{\lambda}(\theta,t)\right)
    +\sum_{i=1}^n
    \left(
    \nabla_{\theta}\log p_t(\theta\mid x_i)-s_\phi(\theta,x_i,t)
    \right)
    \bigg\Vert^2
    \Bigg)\\
    &=
    \mathrm{E}_{\theta}
    \big\Vert
    (1-n)(\nabla_{\theta}\log \lambda_t(\theta)-s_{\lambda}(\theta,t))
    \big\Vert^2
    +\sum_{i=1}^n
    \mathrm{E}_{\theta}
    \big\Vert
    \nabla_{\theta}\log p_t(\theta\mid x_i)-s_\phi(\theta,x_i,t)
    \big\Vert^2\\
    &\quad
    +2\mathrm{E}_{\theta}
    \bigg\langle
    (1-n)(\nabla_{\theta}\log \lambda_t(\theta)-s_{\lambda}(\theta,t)),
    \sum_{i=1}^n
    \left(
    \nabla_{\theta}\log p_t(\theta\mid x_i)-s_\phi(\theta,x_i,t)
    \right)
    \bigg\rangle\\
    &\quad
    +\mathrm{E}_{\theta}
    \sum_{i\neq j}
    \big\langle
    \nabla_{\theta}\log p_t(\theta\mid x_i)-s_\phi(\theta,x_i,t),
    \nabla_{\theta}\log p_t(\theta\mid x_j)-s_\phi(\theta,x_j,t)
    \big\rangle\\
    &\leq
    (1-n)^2
    \mathrm{E}_{\theta}
    \big\Vert
    \nabla_{\theta}\log \lambda_t(\theta)-s_{\lambda}(\theta,t)
    \big\Vert^2
    +\sum_{i=1}^n
    \mathrm{E}_{\theta}
    \big\Vert
    \nabla_{\theta}\log p_t(\theta\mid x_i)-s_\phi(\theta,x_i,t)
    \big\Vert^2\\
    &\quad
    +2(n-1)\sum_{i=1}^n
    \mathrm{E}_{\theta}
    \big\Vert
    \nabla_{\theta}\log \lambda_t(\theta)-s_{\lambda}(\theta,t)
    \big\Vert
    \big\Vert
    \nabla_{\theta}\log p_t(\theta\mid x_i)-s_\phi(\theta,x_i,t)
    \big\Vert\\
    &\quad
    +\sum_{i\neq j}
    \mathrm{E}_{\theta}
    \big\Vert
    \nabla_{\theta}\log p_t(\theta\mid x_i)-s_\phi(\theta,x_i,t)
    \big\Vert
    \big\Vert
    \nabla_{\theta}\log p_t(\theta\mid x_j)-s_\phi(\theta,x_j,t)
    \big\Vert\\
    &\quad \text{by Cauchy--Schwarz and } n\geq 1\\
    &\leq
    (1-n)^2\epsilon_{\mathrm{DSM},\lambda}^2
    +\sum_{i=1}^n \epsilon_{\mathrm{DSM},i}^2\\
    &\quad
    +2(n-1)\sum_{i=1}^n
    \sqrt{
    \mathrm{E}_{\theta}
    \big\Vert
    \nabla_{\theta}\log \lambda_t(\theta)-s_{\lambda}(\theta,t)
    \big\Vert^2
    }
    \sqrt{
    \mathrm{E}_{\theta}
    \big\Vert
    \nabla_{\theta}\log p_t(\theta\mid x_i)-s_\phi(\theta,x_i,t)
    \big\Vert^2
    }\\
    &\quad
    +\sum_{i\neq j}
    \sqrt{
    \mathrm{E}_{\theta}
    \big\Vert
    \nabla_{\theta}\log p_t(\theta\mid x_i)-s_\phi(\theta,x_i,t)
    \big\Vert^2
    }
    \sqrt{
    \mathrm{E}_{\theta}
    \big\Vert
    \nabla_{\theta}\log p_t(\theta\mid x_j)-s_\phi(\theta,x_j,t)
    \big\Vert^2
    }\\
    &\quad \text{by Cauchy--Schwarz}\\
    &\leq
    (1-n)^2\epsilon_{\mathrm{DSM},\lambda}^2
    +\sum_{i=1}^n \epsilon_{\mathrm{DSM},i}^2
    +2(n-1)\epsilon_{\mathrm{DSM},\lambda}
    \sum_{i=1}^n\epsilon_{\mathrm{DSM},i}
    +\sum_{i\neq j}\epsilon_{\mathrm{DSM},i}\epsilon_{\mathrm{DSM},j}\\
    &=
    \Big(
    (n-1)\epsilon_{\mathrm{DSM},\lambda}
    +\sum_{i=1}^n\epsilon_{\mathrm{DSM},i}
    \Big)^2\\
    &\leq
    \Big(
    (n-1)\epsilon_{\mathrm{DSM},\lambda}
    +n\epsilon_{\mathrm{DSM}}
    \Big)^2
    \quad
    \text{if } \epsilon_{\mathrm{DSM},i}\leq \epsilon_{\mathrm{DSM}} \ \forall i .
\end{align*}
\subsection{Analytical derivation of prior and individual posterior scores}\label{proof::analytical_scores}
\paragraph{Multivariate Gaussian.}
The analytical formulas of individual Gaussian posteriors diffused along the VP process are given in Appendix~\ref{proof::analyt_densities}. This allows us to derive analytical scores of diffused individual posteriors:
\begin{equation*}
    \nabla_\theta\log p_t(\theta\mid x_i)=-(\alpha_t\Sigma_\text{post}+v_tI_d)^{-1}(\theta-\sqrt{\alpha_t}\mu_\text{post}(x_i)) \quad \forall t\in[0,1],
\end{equation*} where $\Sigma_\text{post}$ and $\mu_\text{post}(x_i)$ are explicitly defined in Appendix~\ref{proof::analyt_densities}.
For Log-Normal priors, we first apply a log-transform on the samples before computing the score: therefore, it amounts to deriving the score of a Gaussian distribution along the VP process.
\paragraph{GMM prior.} The prior is a mixture of Gaussians of the form $\lambda(\theta)=\frac{1}{2}\mathcal{N}(\theta;\mu_1,\sigma_1^2I_d)+\frac{1}{2}\mathcal{N}(\theta;\mu_2,\sigma_2^2I_d)$ and the likelihood is $p(x\mid \theta)=\mathcal{N}(x;\theta,\Sigma)$. Then, the score of the diffused prior along the VP process is equal to:
\begin{align*}
    \nabla_\theta\log \lambda_t(\theta)&=-\omega_{t,1}(\theta)\Sigma_{t,1}^{-1}(\theta-\sqrt{\alpha_t}\mu_1)-\omega_{t,2}(\theta)\Sigma_{t,2}^{-1}(\theta-\sqrt{\alpha_t}\mu_2),\\
    \text{where} \ \Sigma_{t,j}&=(\alpha_t\sigma_j^2+v_t)I_d \quad \forall j=1,2\\
    \text{and} \ \tilde \omega_{t,j}(\theta)&=\frac 12\mathcal{N}(\theta;\sqrt{\alpha_t}\mu_j, \Sigma_{t,j}), \quad \omega_{t,j}(\theta)=\tilde\omega_{t,j}(\theta)/\sum_{k=1}^2\tilde \omega_{t,k}(\theta).
\end{align*}The corresponding posterior is a mixture of 2 Gaussians:
\begin{equation*}
    p(\theta\mid x)=\omega_{\text{post},1}(x)\mathcal{N}(\mu_{\text{post},1},\Sigma_{\text{post},1})+\omega_{\text{post},2}(x)\mathcal{N}(\mu_{\text{post},2},\Sigma_{\text{post},2}),
\end{equation*}where 
\begin{align*}
    \Sigma_{\text{post},j}&=(\Sigma^{-1}+\frac{1}{\sigma_j^2}I_d)^{-1},\\
    \mu_{\text{post},j}&=\Sigma_{\text{post},j}(\mu_j/\sigma_j^2+\Sigma^{-1}x),\\
    \tilde \omega_{\text{post},j}(x) &= \frac 12\mathcal{N}(x;\mu_j, \sigma_j^2I_d+\Sigma),\\
    \omega_{\text{post},j}(x)&=\tilde \omega_{\text{post},j}(x)/\sum_{k=1}^2\tilde \omega_{\text{post},k}(x) \quad \forall j=1,2.
\end{align*} We have the analytical formula of the score of the diffused posterior along the VP scheme:
\begin{equation*}
    \nabla_\theta \log p_t(\theta\mid x)=-\omega_{\text{post}, t,1}(\theta)\Sigma_{\text{post},t,1}^{-1}(\theta-\sqrt{\alpha_t}\mu_{\text{post},1})-\omega_{\text{post},t,2}(\theta)\Sigma_{\text{post},t,2}^{-1}(\theta-\sqrt{\alpha_t}\mu_{\text{post},2}),
\end{equation*} where 
\begin{align*}
    \Sigma_{\text{post},t,j}&=(\alpha_t\Sigma_{\text{post},j}+v_tI_d),\\
    \tilde \omega_{\text{post},t,j}(\theta)&=\omega_{\text{post},j}(x)\mathcal{N}(\theta;\sqrt{\alpha_t}\mu_{\text{post},j},\alpha_t\Sigma_{\text{post},j}+v_tI_d),\\
    \omega_{\text{post},t,j}(\theta)&=\tilde \omega_{\text{post},t,j}(\theta)/\sum_{k=1}^2 \omega_{\text{post},t,k}(\theta).
\end{align*}
\paragraph{GMM likelihood.} We consider a standard Gaussian prior $\lambda(\theta)=\mathcal{N}(0,I_d)$ and a mixture of two Gaussians as likelihood of the form $p(x\mid \theta)=\frac 12 \mathcal{N}(x;\theta,1/9\Sigma)+\frac 12 \mathcal{N}(x;\theta,2.25\Sigma)$ and $\Sigma$ is a diagonal matrix with values increasing linearly between $0.6$ and $1.4$. Let $\Sigma_1=2.25\Sigma$ and $\Sigma_2=1/9\Sigma$. The individual posterior is a mixture of two Gaussians of the form:
\begin{align*}
    p(\theta\mid x)&=\omega_1(x)\mathcal{N}(\theta;\mu_1,\Sigma_{1,p})+\omega_2(x)\mathcal{N}(\theta;\mu_2,\Sigma_{2,p}),\\
    \text{where} \quad \Sigma_{j,p}&=(\Sigma_j^{-1}+I_d)^{-1},\\
    \mu_{j}&=\Sigma_{j,p}\Sigma_j^{-1}x,\\
    \tilde \omega_j(x)&=\frac 12 \mathcal{N}(x;0,I_d+\Sigma_j),\\
    \omega_j(x)&=\tilde\omega_j(x)/\sum_{k=1}^2\tilde \omega_k(x).
\end{align*} The score of the diffused individual posterior is as follows:
\begin{align*}
    \nabla_\theta\log p_t(\theta\mid x)&=-\omega_{t,1}(\theta)(\alpha_t\Sigma_{1,p}+v_tI_d)^{-1}(\theta-\sqrt{\alpha_t}\mu_1)-\omega_{t,2}(\theta)(\alpha_t\Sigma_{2,p}+v_tI_d)^{-1}(\theta-\sqrt{\alpha_t}\mu_2),\\
    \text{where} \quad \tilde \omega_{t,j}(\theta)&=\omega_j(x)\mathcal{N}(\theta;\sqrt{\alpha_t}\mu_j,\alpha_t\Sigma_{j,p}+v_tI_d),\\
    \omega_{t,j}(\theta)&=\tilde \omega_{t,j}(\theta)/\sum_{k=1}^2 \tilde \omega_{t,k}(\theta).
\end{align*}
\newpage
\section{Empirical results and technical details}
\subsection{Empirical validation of the Wasserstein error - Gaussian setting}\label{app::empirical_validation_wass}
We provide here the final Wasserstein errors obtained with the prescribed hyperparameters (step sizes and number of steps per level) for the $d$-dimensional Gaussian case. Figure~\ref{fig:wass_validation_gaussian} presents an empirical validation of our decision rule for annealed Langevin hyperparameters: indeed, the final Wasserstein error obtained with our tuned annealed Langevin algorithm remains controlled by the chosen threshold $\gamma$, whatever the dimension $d$ or the number of conditional observations $n$.
\begin{figure}[H]
    \centering
    \includegraphics[width=0.8\linewidth]{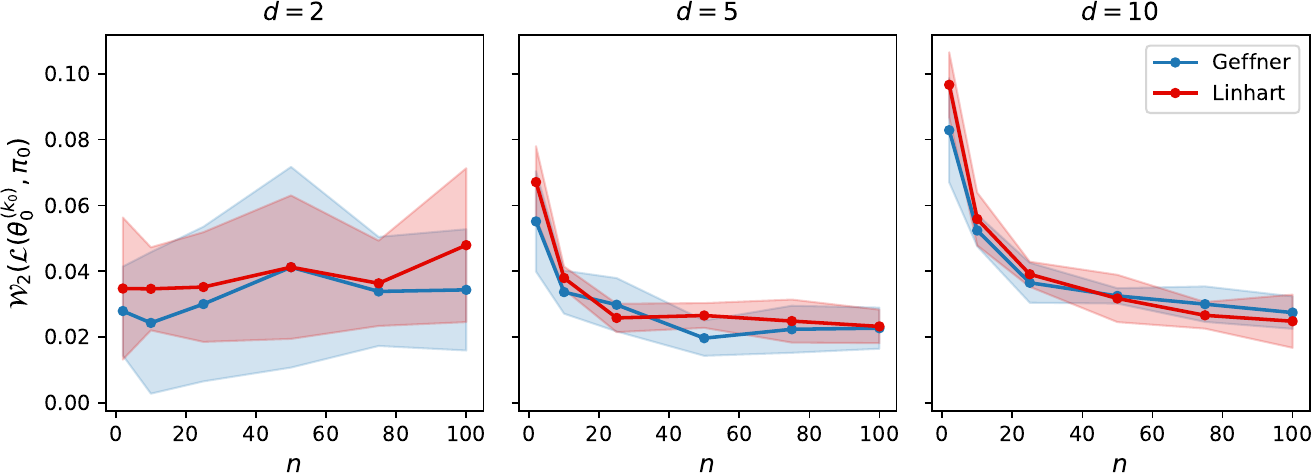}
    \caption{\small Mean final Wasserstein errors obtained in the Gaussian setting with annealed Langevin dynamics and prescribed hyperparameters from Section~\ref{sec::fine_tune}. Blue (resp. red) curves stand for Geffner's (resp. Linhart's) compositional score used in the sampling algorithm. We use $T=10$ annealing levels, $\gamma=0.5$ and $\omega=0.5$ (resp. $0.8$) for $d=2,5$ (resp. $d=10$). The empirical Wasserstein error indeed remains below the fixed threshold $\gamma$ whatever the dimension $d$ and the number of conditional observations $n$. As theoretically expected, it is also similar for both choices of compositional scores. The error does not vary a lot whatever the dimension $d$. Mean and std are computed over $5$ seeds.}
    \label{fig:wass_validation_gaussian}
\end{figure}
\newpage
\subsection{Comparison of the number of steps - Gaussian setting}\label{app::additional_figure_gaussian_case} In the Gaussian setting, we test our results for several parameter dimensions $d=2,5,10$. Figure~\ref{fig:evol_nb_steps_gaussian_more_dim} shows the evolution of the number of steps over time. We observe that there are two regimes separated by a switch point $\tilde t$: on $[0,\tilde t]$, $k_t^L\geq k_t^G$ but this regime is very short; on $[\tilde t,1]$, $k_t^L\leq k_t^G$ and this regime covers most of the annealing levels.
\begin{figure}[H]
    \centering
    \includegraphics[width=\linewidth]{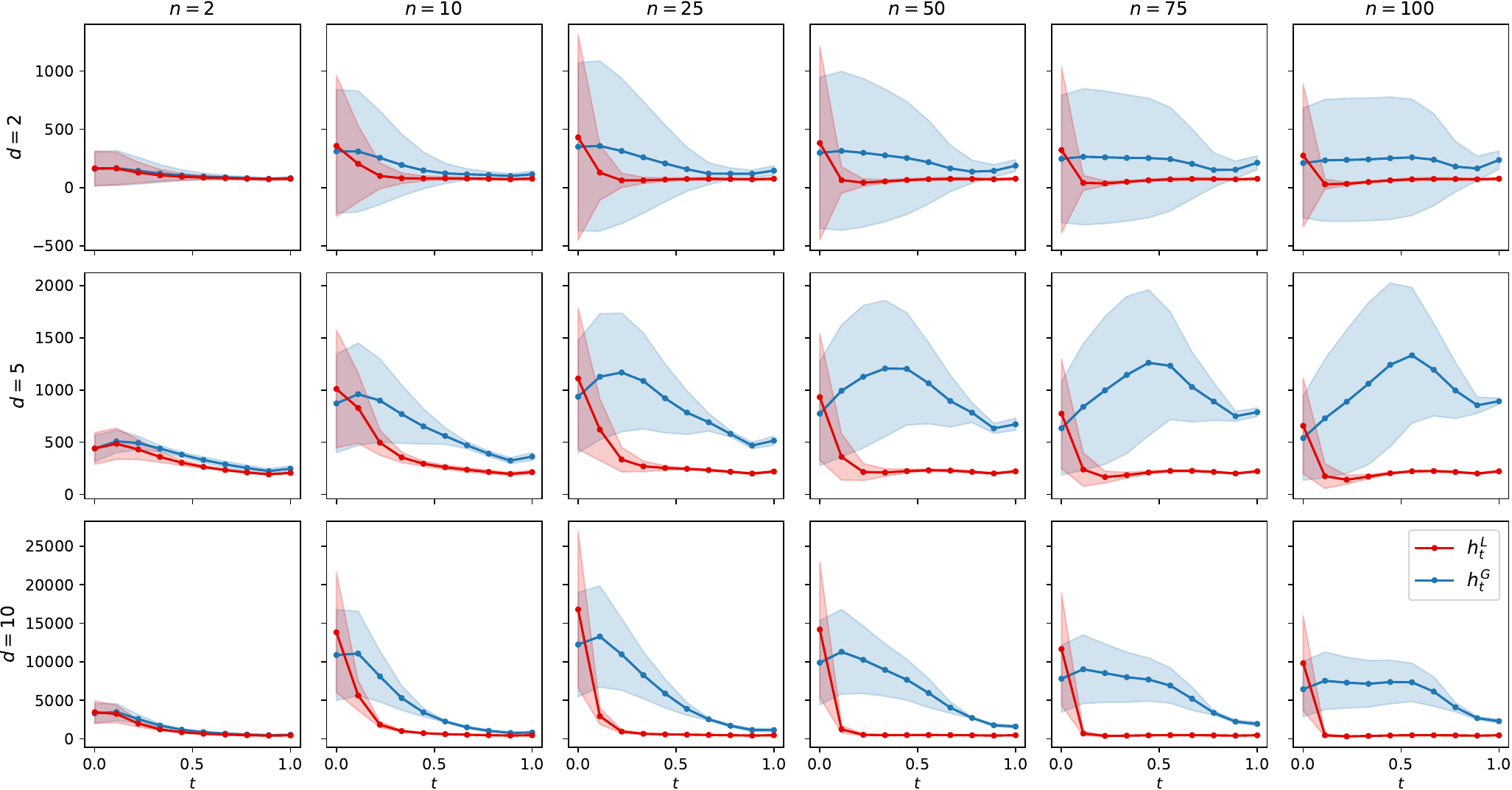}
    \caption{\small Evolution of $k_t^G$ and $k_t^L$ over time for the Gaussian case. The switch point $\tilde t$ (where blue and red curves cross) is unique and often appears from the second annealing level near $t=0$, whatever the dimension of the space $d$ or the number of conditional observations $n$. Mean and std are computed over $5$ seeds, each model differs in its likelihood covariance matrix.}
    \label{fig:evol_nb_steps_gaussian_more_dim}
\end{figure}
Figure~\ref{fig:exception_dim_2_gaussian} shows one example in dimension $2$, where the switch point $\tilde t$ appears later than the second annealing level: in this case, $k_t^L\geq k_t^G$ for more annealing levels than in most cases. Note that this behavior becomes even more rare when the dimension increases.
\begin{figure}[H]
    \centering
    \includegraphics[width=\linewidth]{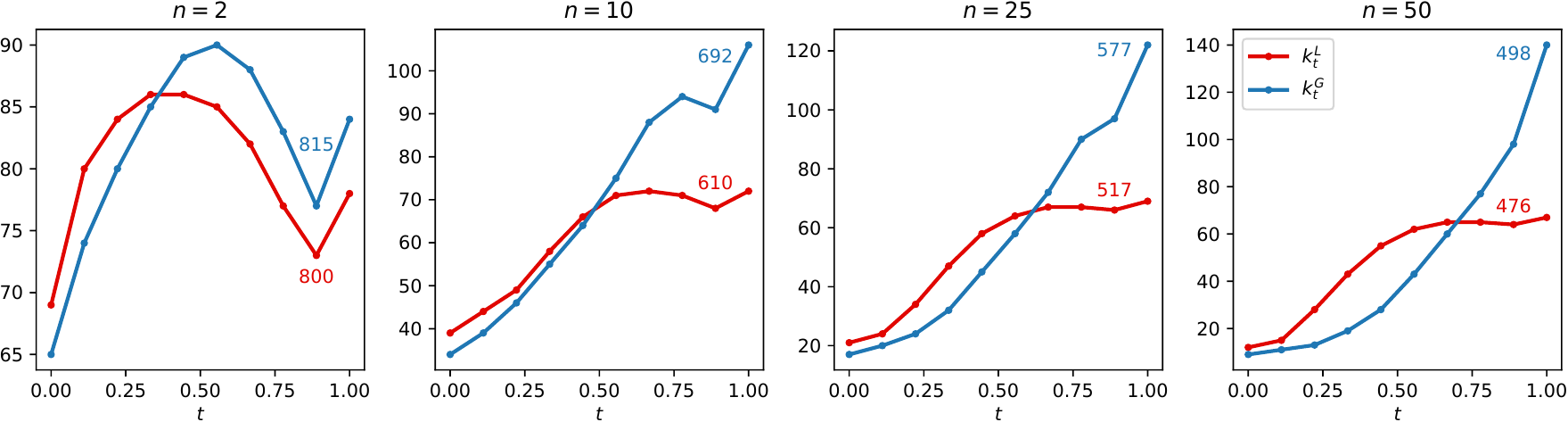}
    \caption{\small Evolution of $k_t^G$ and $k_t^L$ over time: Example of switch point appearing later when $d=2$. The colored figures on the right of each subplot correspond to the total number of steps $\sum_{p=0}^Tk_{t_p}$. In this specific case, blue and red curves cross later than the second annealing level. However, the complexity (in terms of total number of steps) remains lower when choosing Linhart's composite score than that of Geffner.}
    \label{fig:exception_dim_2_gaussian}
\end{figure}
\newpage
\subsection{Additional figures - GMM Prior}\label{app::gmm_prior_additional_figures}
For this specific task, the goal is to deviate from Gaussianity in a controlled manner. The prior is defined as a $2$-dimensional mixture of two Gaussians: $\lambda(\theta)=\frac 12 \mathcal{N}(\mu_1,\sigma_1^2I_d)+\frac 12 \mathcal{N}(\mu_2, \sigma_2^2I_d)$. Instead of moving the means, we fix them to $\mu_1=(0,0)$ and $\mu_2=(1,1)$ and tune the couple of scales $(\sigma_1,\sigma_2)$ to control non Gaussianity. We use the following scale values $(0.05,0.05)$, $(0.05,0.5)$ and $(0.5,0.5)$. As the scales decrease, the individual posteriors become more bimodal. The likelihood is a Gaussian of the form $p(x\mid \theta)=\mathcal{N}(x;\theta,\Sigma)$, with $\Sigma$ a positive definite matrix. Figure~\ref{fig:gmm_prior_additional_results} shows the complexity and final Wasserstein error when the individual posterior is strongly bimodal (prior scales are chosen small as per $(0.05,0.05)$ such that prior means are well separated) and less bimodal (with greater prior scales $(0.5,0.5)$). Our hyperparameter decision rule allows one to reach a similar prescribed low Wasserstein error for both choices of algorithm. As $n$ increases, the multi-observation posterior becomes more Gaussian, thus making the sampling task easier and requiring a lower total number of steps. Overall, Linhart's compositional score yields a more efficient sampler in terms of score function evaluations, especially for large $n$ and less multimodal target posterior (larger scales $\sigma_1,\sigma_2$). The results for the last scale choice, $(0.5,0.05)$, are provided in Figure~\ref{fig:results_non_gaussian} (total number of steps) and the following Appendix (final Wasserstein error): this model strikes a balance between Gaussianity and strong non-Gaussianity since the posterior rapidly converges to a Gaussian as $n$ increases. 
\begin{figure}[H]
    \centering
    \includegraphics[width=0.7\linewidth]{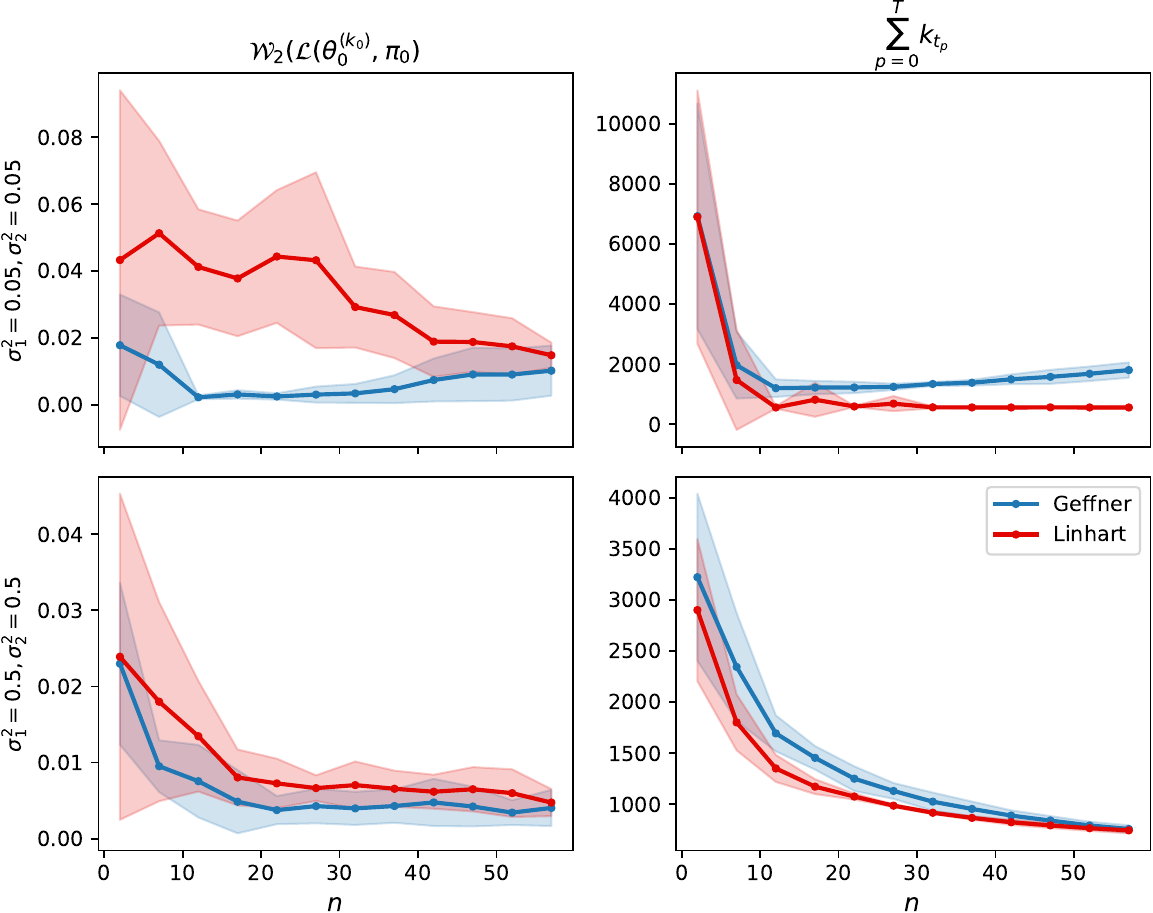}
    \caption{\small The top two subplots show the mean final Wasserstein error (left) and total number of steps (right) for a strong non Gaussian \textbf{GMM prior} model (scales $\sigma_1$ and $\sigma_2$ are very small). The bottom two subplots present the same results (mean Wasserstein error on the left and total number of steps on the right) for a less non Gaussian model (with greater prior scales). We use $T=10$ annealing levels, $\gamma=0.5$ and $\omega=0.5$. The Wasserstein error and the complexity are presented for different numbers of conditional observations $n$. We observe that the empirical Wasserstein error is well controlled by $\gamma$ on both tasks. Mean and std are computed over $5$ seeds.}
    \label{fig:gmm_prior_additional_results}
\end{figure}
\newpage
\subsection{Technical details - Non Gaussian setting}\label{proof::technical_details}
Prior scores can be explicitly derived for all tasks. For \textbf{GMM prior} and \textbf{GMM likelihood} tasks, individual analytical scores can also be derived (see Appendix~\ref{proof::analytical_scores}). For \textbf{SIR} and \textbf{Lotka Volterra}, we train an amortized neural score network $s_\phi(\theta,x,t)$ using the \texttt{sbi} package (training loop, network architecture and default training parameters) \citep{sbi_guide}. The network is a MLP composed of $3$ hidden layers of $256$ hidden features and normalization layers. A Gaussian Fourier embedding is used for time.

For these two tasks, we train the network on log $\theta$'s for more stability, then apply the reverse transform during sampling. An additional embedding network (MLP with $2$ hidden layers of $64$ hidden features) is used to encode observations $x_1,\ldots, x_n$ for the \textbf{Lotka Volterra} task.
We use a training set of $25000$ pairs $(\theta,x)$ for each task with a validation split $90\%-10\%$. We train the score network for maximum $1000$ epochs with early stopping (implemented by default in the \texttt{sbi} package).

We sample from the true target posteriors using MCMC with \texttt{numpyro} \citep{numpyro} and \texttt{jax} \citep{jax}, except for \textbf{GMM prior} where analytical formula can be derived. We compute the Wasserstein error using the \texttt{POT} Python package \citep{pot} using $3000$ samples obtained from the annealed Langevin algorithm and the true distribution. For Linhart's composite score we use their $\texttt{GAUSS}$ algorithm to estimate covariance matrices using a small number ($\approx 100$) of diffusion steps. Importantly, we do not include this Linhart-specific cost (in terms of score evaluations) into the computational comparison: indeed, estimating such matrices is also necessary during fine-tuning (especially to compute $m_t$, $M_t$ and $\mathcal{W}_2(\pi_t,\pi_{t+1})$), \textit{regardless} of the choice of bridging densities. We thus treat this as a \textit{shared cost} between both composite scores that we do not include in our comparison: the complexity only includes total score evaluations needed during sampling, i.e. $\sum_{p=0}^Tk_{t_p}$.

We use $\gamma=0.5$ as threshold for the Wasserstein error for all tasks and $\omega=0.5$ during hyperparameter selection, except when the parameter dimension $d\geq5$: in this case, we instead use $\omega=0.8$ to reduce the number of steps while reaching a similar Wasserstein error.
With respect to our training, we use a maximal value of the compositional error $\epsilon_\text{DSM}=0.1$ for \textbf{SIR} and \textbf{Lotka Volterra}.

To run the experiments, we use a machine with an Intel CPU and $8$ cores. We run the hyperparameter tuning followed by annealed Langevin sampling for different values of $n$ at once: in practice we choose $n$ IID observations $x_1,\ldots,x_n$, then iteratively sample from the posteriors $p(\theta\mid x_{1:j})$ for $j\in[2,\ldots,n]$. Finally, we consider $T=10$ annealing levels (as in \citet{song_annealed_langevin}) uniformly spaced on the interval $[0,1]$: for computational stability, we set $t_0=1e-4$ instead of $0$ for sampling and fine-tuning the hyperparameters.

Figure~\ref{fig:evol_wass_non_gaussian_case} shows that hyperparameters chosen according to our decision rule indeed allow one to control the Wasserstein error of annealed Langevin algorithm for the different inference tasks.
\begin{figure}[H]
    \centering
    \includegraphics[width=\linewidth]{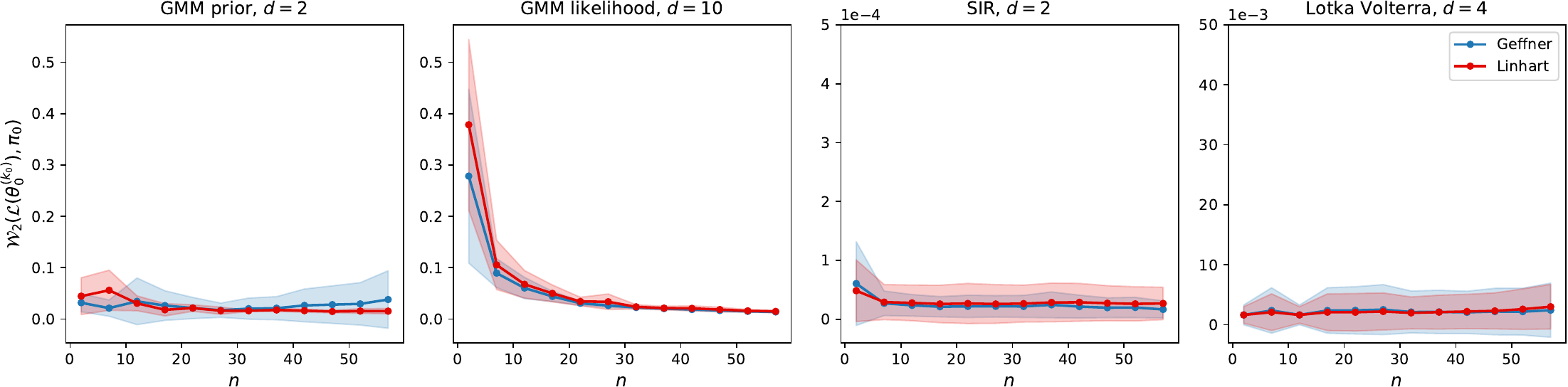}
    \caption{\small Evolution of the (mean) final Wasserstein error $\mathcal{W}_2(\mathcal{L}(\theta_0^{(k_0)}),\pi_0)$ of annealed Langevin algorithm run with the prescribed hyperparameters for the tasks presented in Section~\ref{sec::numerical_applcations}. We use $T=10$ (as in \citet{song_annealed_langevin}) and $\gamma=0.5$. These plots show that the prescribed hyperparameters allow one to control the final empirical Wasserstein error by $\gamma$ whatever the number of conditional observations $n$. The Wasserstein error is also similar for both compositional scores, as expected with our fine-tuning. Importantly, the sampling does not include intermediate clipping steps. Mean and std are computed over $5$ seeds.}
    \label{fig:evol_wass_non_gaussian_case}
\end{figure}
As $n$ grows, our Gaussian approximation of the multi-observation posterior (in Section~\ref{sec::compositional_error}) becomes more accurate, even when each individual posterior is approximated through moment matching (as discussed in Appendix~\ref{app::discuss_gaussian_approx}). Moreover, the necessary condition on the smoothness constants in Lemma~\ref{lemma::preservation_cst_nongaussian} is more likely to hold for larger values of $n$, making the bridging densities $\pi_t^L$ and $\pi_t^G$ satisfy the (smoothness and log-concavity) assumptions required for our decision rule to apply properly. In other words, increasing the number of conditioning variables makes these assumptions more likely to be satisfied, enabling our decision rule to yield well-tuned hyperparameters: this may explain why the Wasserstein error tends to decrease or stabilize in Figure~\ref{fig:evol_wass_non_gaussian_case} for all tasks. However, for large $n$, a slight increase may come from the accumulation of errors when aggregating individual scores into a compositional score, meaning that the compositional error $\epsilon_\text{DSM}$ may be slightly underestimated. Moreover, for strongly multimodal distributions (\textbf{GMM likelihood}), the Wasserstein error remains controlled but closer to the threshold $\gamma$ than for large $n$ where multimodality is progressively smoothed. Overall, we observe that the Wasserstein error remains controlled and little variable for large $n$, which manual tuning does not allow in general.
\newpage
\subsection{Comparison with score-based diffusion sampling and alternative Langevin tuning}
\label{app::compar_diffusion}
In this section, we propose empirical evidence that diffusion sampling based on a compositional score fails to produce samples from the target distribution $\pi_0$, especially as $n$ grows. Sampling quality is measured in Wasserstein distance. For diffusion sampling, we use a reverse VP diffusion scheme with $1000$ steps (as prescribed in \citet{song2021score}), $\beta_\text{min}=0.01$, $\beta_\text{max}=10$ and $t_0=1e-4$.

We also compare the final Wasserstein quality when choosing annealed Langevin hyperparameters according to the empirical tuning in \citet{geffner}: they use $T=400$ levels and the same number of steps per annealing level while adjusting the step size according to a specific taming rule. We adapt their empirical tuning (constant number of step and taming rule for the step size) while keeping the same number of score evaluations as used with our decision rule. For our tuning, we keep using the same values: $\gamma=0.5$, $T=10$ and $t_0=1e-4$.

Tables~\ref{tab::compar_diff_sir} and \ref{tab::compar_diff_lotka} show that diffusion sampling indeed yields a greater empirical Wasserstein error than annealed Langevin sampling especially for $n>2$, as theoretically expected. They also show that we achieve a better Wasserstein quality with our tuned annealed Langevin sampling than with Geffner's empirical tuning for both compositional scores, especially for $n>2$.
 
\begin{table}[H]
  \caption{Mean final Wasserstein error when using \textit{diffusion} sampling, annealed Langevin sampling with Geffner's \textit{empirical} tuning or \textit{our} fine-tuning on the \textbf{SIR} task for different numbers of conditional observations $n$. Mean and std are computed over $5$ seeds. Bold figures indicate the lowest Wasserstein error across sampling methods.}
  \label{tab::compar_diff_sir}
  \centering

    \begin{tabular}{lccc}
\multicolumn{4}{c}{\textbf{Geffner's compositional score $s_t^G$}}\\
\toprule
$n$ & ours & empirical & diffusion \\
\midrule
$2$  & $1.6\mathrm{e}{-04} \pm 1.8\mathrm{e}{-04}$
     & $6.4\mathrm{e}{-02} \pm 6.4\mathrm{e}{-02}$
     & $\mathbf{8.0\mathrm{e}{-05} \pm 8.0\mathrm{e}{-05}}$ \\

$12$ & $5.3\mathrm{e}{-05} \pm 5.0\mathrm{e}{-05}$
     & $\mathbf{4.2\mathrm{e}{-05} \pm 4.2\mathrm{e}{-05}}$
     & $1.2\mathrm{e}{-02} \pm 1.2\mathrm{e}{-02}$ \\

$22$ & $\mathbf{5.0\mathrm{e}{-05} \pm 5.4\mathrm{e}{-05}}$
     & $1.1\mathrm{e}{-04} \pm 1.1\mathrm{e}{-04}$
     & $3.5\mathrm{e}{-01} \pm 3.5\mathrm{e}{-01}$ \\

$32$ & $\mathbf{4.9\mathrm{e}{-05} \pm 5.2\mathrm{e}{-05}}$
     & $3.7\mathrm{e}{-04} \pm 3.7\mathrm{e}{-04}$
     & $2.9\mathrm{e}{+00} \pm 2.9\mathrm{e}{+00}$ \\

$42$ & $\mathbf{5.1\mathrm{e}{-05} \pm 5.6\mathrm{e}{-05}}$
     & $1.0\mathrm{e}{-03} \pm 1.0\mathrm{e}{-03}$
     & $2.6\mathrm{e}{+01} \pm 2.6\mathrm{e}{+01}$ \\

$52$ & $\mathbf{4.8\mathrm{e}{-05} \pm 5.2\mathrm{e}{-05}}$
     & $2.1\mathrm{e}{-03} \pm 2.1\mathrm{e}{-03}$
     & $8.1\mathrm{e}{+01} \pm 8.1\mathrm{e}{+01}$ \\
\bottomrule
\end{tabular}

\vspace{0.5em}

\begin{tabular}{lccc}
\multicolumn{4}{c}{\textbf{Linhart's compositional score $s_t^L$}}\\
\toprule
$n$ & ours & empirical & diffusion \\
\midrule
$2$  & $\mathbf{5.5\mathrm{e}{-05} \pm 6.9\mathrm{e}{-05}}$
     & $3.5\mathrm{e}{-01} \pm 1.7\mathrm{e}{-01}$
     & $8.7\mathrm{e}{-05} \pm 9.0\mathrm{e}{-05}$ \\

$12$ & $\mathbf{3.0\mathrm{e}{-05} \pm 2.5\mathrm{e}{-05}}$
     & $3.9\mathrm{e}{-01} \pm 6.5\mathrm{e}{-02}$
     & $7.2\mathrm{e}{-04} \pm 1.0\mathrm{e}{-03}$ \\

$22$ & $\mathbf{3.2\mathrm{e}{-05} \pm 2.3\mathrm{e}{-05}}$
     & $3.8\mathrm{e}{-01} \pm 6.4\mathrm{e}{-02}$
     & $2.2\mathrm{e}{-03} \pm 1.5\mathrm{e}{-03}$ \\

$32$ & $\mathbf{3.1\mathrm{e}{-05} \pm 2.3\mathrm{e}{-05}}$
     & $3.8\mathrm{e}{-01} \pm 6.4\mathrm{e}{-02}$
     & $8.4\mathrm{e}{-03} \pm 1.0\mathrm{e}{-02}$ \\

$42$ & $\mathbf{4.1\mathrm{e}{-05} \pm 2.8\mathrm{e}{-05}}$
     & $3.8\mathrm{e}{-01} \pm 6.5\mathrm{e}{-02}$
     & $7.0\mathrm{e}{-02} \pm 1.6\mathrm{e}{-01}$ \\

$52$ & $\mathbf{5.2\mathrm{e}{-05} \pm 4.3\mathrm{e}{-05}}$
     & $3.7\mathrm{e}{-01} \pm 6.5\mathrm{e}{-02}$
     & $7.7\mathrm{e}{-01} \pm 1.8\mathrm{e}{+00}$ \\
\bottomrule
\end{tabular}
\end{table}

\begin{table}[H]
    \centering
    \caption{Mean final Wasserstein error when using \textit{diffusion} sampling, annealed Langevin sampling with Geffner's \textit{empirical} tuning or \textit{our} fine-tuning on \textbf{Lotka Volterra} task for different numbers of conditional observations $n$. Mean and std are computed over $5$ seeds. Bold figures indicate the lowest Wasserstein error across sampling methods.}
    \label{tab::compar_diff_lotka}

\begin{tabular}{lccc}
\multicolumn{4}{c}{\textbf{Geffner's compositional score $s_t^G$}}\\
\toprule
$n$ & ours & empirical & diffusion \\
\midrule
$2$  & $5.7\mathrm{e}{-03} \pm 8.4\mathrm{e}{-03}$
     & $\mathbf{5.2\mathrm{e}{-03} \pm 5.2\mathrm{e}{-03}}$
     & $7.2\mathrm{e}{-03} \pm 7.2\mathrm{e}{-03}$ \\

$12$ & $1.6\mathrm{e}{-03} \pm 1.6\mathrm{e}{-03}$
     & $1.4\mathrm{e}{-03} \pm 1.4\mathrm{e}{-03}$
     & $\mathbf{1.3\mathrm{e}{-03} \pm 1.3\mathrm{e}{-03}}$ \\

$22$ & $\mathbf{2.4\mathrm{e}{-03} \pm 3.9\mathrm{e}{-03}}$
     & $9.5\mathrm{e}{-03} \pm 9.5\mathrm{e}{-03}$
     & $1.3\mathrm{e}{+00} \pm 1.3\mathrm{e}{+00}$ \\

$32$ & $\mathbf{2.2\mathrm{e}{-03} \pm 3.7\mathrm{e}{-03}}$
     & $3.7\mathrm{e}{-02} \pm 3.7\mathrm{e}{-02}$
     & $4.7\mathrm{e}{+01} \pm 4.7\mathrm{e}{+01}$ \\

$42$ & $\mathbf{2.2\mathrm{e}{-03} \pm 3.7\mathrm{e}{-03}}$
     & $9.6\mathrm{e}{-02} \pm 9.6\mathrm{e}{-02}$
     & $1.6\mathrm{e}{+04} \pm 1.6\mathrm{e}{+04}$ \\

$52$ & $\mathbf{2.3\mathrm{e}{-03} \pm 4.1\mathrm{e}{-03}}$
     & $1.6\mathrm{e}{-01} \pm 1.6\mathrm{e}{-01}$
     & $1.7\mathrm{e}{+05} \pm 1.7\mathrm{e}{+05}$ \\
\bottomrule
\end{tabular}
\begin{tabular}{lccc}
\multicolumn{4}{c}{\textbf{Linhart's compositional score $s_t^L$}}\\
\toprule
$n$ & ours & empirical & diffusion \\
\midrule
$2$  & $5.2\mathrm{e}{-03} \pm 7.4\mathrm{e}{-03}$
     & $\mathbf{4.7\mathrm{e}{-03} \pm 6.1\mathrm{e}{-03}}$
     & $6.2\mathrm{e}{-03} \pm 5.3\mathrm{e}{-03}$ \\

$12$ & $\mathbf{1.7\mathrm{e}{-03} \pm 1.4\mathrm{e}{-03}}$
     & $8.5\mathrm{e}{-03} \pm 9.7\mathrm{e}{-03}$
     & $2.0\mathrm{e}{-03} \pm 1.2\mathrm{e}{-03}$ \\

$22$ & $\mathbf{2.1\mathrm{e}{-03} \pm 3.0\mathrm{e}{-03}}$
     & $2.6\mathrm{e}{-02} \pm 1.2\mathrm{e}{-02}$
     & $3.0\mathrm{e}{-03} \pm 1.5\mathrm{e}{-03}$ \\

$32$ & $\mathbf{2.2\mathrm{e}{-03} \pm 2.5\mathrm{e}{-03}}$
     & $2.5\mathrm{e}{-02} \pm 9.7\mathrm{e}{-03}$
     & $5.7\mathrm{e}{-03} \pm 3.7\mathrm{e}{-03}$ \\

$42$ & $\mathbf{2.8\mathrm{e}{-03} \pm 2.7\mathrm{e}{-03}}$
     & $2.5\mathrm{e}{-02} \pm 8.4\mathrm{e}{-03}$
     & $8.2\mathrm{e}{-03} \pm 5.7\mathrm{e}{-03}$ \\

$52$ & $\mathbf{4.0\mathrm{e}{-03} \pm 3.6\mathrm{e}{-03}}$
     & $2.7\mathrm{e}{-02} \pm 1.2\mathrm{e}{-02}$
     & $1.1\mathrm{e}{-02} \pm 7.8\mathrm{e}{-03}$ \\
\bottomrule
\end{tabular}
\end{table}
\newpage
\subsection{Discussion on the Gaussian approximations}
\label{app::discuss_gaussian_approx}
To get analytical expressions of bridging densities $\pi_t^G$ and $\pi_t^L$ and apply our decision rule in practice, we resort to a Gaussian approximation of individual posteriors $p(\theta\mid x_i)$ (see Section~\ref{sec::empirical_comparaison}). We do not make any additional assumptions here but use the ones of \texttt{GAUSS} algorithm from \citet{linhart2024diffusion}: we thus choose a \textit{moment matching} estimation of each individual posterior, where the first two moments are empirically computed.

As $\pi_t^G$ and $\pi_t^L$ coincide with the target multi-observation posterior $\pi_0$ at time $t=0$, we approximate $\pi_0$ by composing the $n$ Gauss-approximated individual posteriors as per $\pi_0(\theta\mid x_{1:n})\propto\lambda(\theta)^{1-n}\prod_{i=1}^n p(\theta\mid x_i)$. This approximation at time $t=0$ is used to get an analytical expression of $\pi_t^L$ for all $t>0$ and is easy to implement. However, beyond the pure Gaussian case, it does not ensure that the product of Gauss-approximated individual posteriors coincides with the asymptotic Gaussian distribution $\pi_0$ should converge to as $n$ grows (according to the Bernstein von Mises theorem). Table~\ref{tab::gauss_approx_error} shows that there is indeed a mismatch with our moment-matching estimation but the Wasserstein error between $\pi_0$ (considered as Gaussian with the first two moments computed on MCMC samples) and our approximation remains low across different tasks, especially for unimodal posteriors and even for small $n$.
\begin{table}[H]
\caption{Wasserstein error made when approximating the target multi-observation posterior by a product of Gauss-approximated individual posteriors. Mean is computed over $5$ seeds.}
  \label{tab::gauss_approx_error}
  \centering
\begin{tabular}{lllllll}
\toprule
 & $n=2$ & $n=12$ & $n=22$ & $n=32$ & $n=42$ & $n=52$ \\
\midrule
\textbf{GMM likelihood} & $2.3\mathrm{e}{-01}$ & $1.2\mathrm{e}{+00}$ & $3.1\mathrm{e}{-01}$ & $1.9\mathrm{e}{-01}$ & $2.1\mathrm{e}{-01}$ & $1.3\mathrm{e}{-01}$ \\
\textbf{Lotka Volterra} & $5.5\mathrm{e}{-04}$ & $2.7\mathrm{e}{-04}$ & $7.8\mathrm{e}{-05}$ & $1.0\mathrm{e}{-04}$ & $2.1\mathrm{e}{-04}$ & $1.6\mathrm{e}{-04}$ \\
\textbf{SIR} & $9.5\mathrm{e}{-06}$ & $3.9\mathrm{e}{-05}$ & $4.4\mathrm{e}{-05}$ & $6.9\mathrm{e}{-05}$ & $6.2\mathrm{e}{-05}$ & $7.8\mathrm{e}{-05}$ \\
\bottomrule
\end{tabular}
\end{table}
Using instead a \textit{Laplace/curvature matching} estimation of each individual posterior would be an interesting alternative as it would yield a multi-observation posterior closer to the one provided by the Bernstein von Mises theorem; this approximation is however harder to obtain as we need to estimate the curvature of each individual posterior.
\newpage
\subsection{Empirical estimation of the score error $\epsilon_\text{DSM}$}
\label{app::estimation_score_error}
The score error $\epsilon_\text{DSM}$ that appears in our decision rule denotes the compositional score matching error taken on average over the \textit{iterates of the annealed Langevin process}. In practice, this error is difficult to estimate and we resort to approximations: first, we learn an amortized individual score by minimizing the denoising score matching loss and thus get an estimate of the score error $\epsilon_\text{indiv}$ on average over \textit{samples from the diffused posteriors $p_t(\theta\mid x_i)$}; second, we upperbound the compositional score error $\epsilon_\text{DSM}$ as $O(n\epsilon_\text{indiv})$ according to Lemma~\ref{lemma::composit_error} for both algorithms, even if Linhart's compositional score error seems to stabilize as $n$ increases (see \citet{compositional_error}).

This way of learning the score does not guarantee that the learned (individual or compositional) score provides an accurate estimate along the whole annealed Langevin sampling trajectory, since the Langevin sampler may visit regions that differ from the diffusion path, especially in middle to low noise regimes: this is not necessarily the case in (Gaussian) high noise regimes, where both samplers should have close trajectories or visit similar regions. Importantly, we empirically observe that miss estimating the compositional error $\epsilon_\text{DSM}$ does not have a huge impact on the choice of hyper-parameters for early annealing levels, as it is compensated by log-concavity constants $m_t$. This is even true when $n$ grows as the log-concavity constant increases as $O(n)$ (see Appendix~\ref{proof::preservation_cst}). Tables~\ref{tab::evol_m_sir} and \ref{tab::evol_m_lotka} present the evolution of the log-concavity constants $m_t^G$ and $m_t^L$ for the tasks \textbf{SIR} and \textbf{Lotka Volterra}.
\begin{table}[H]
\caption{Evolution of $m_t^G$ and $m_t^L$ for \textbf{SIR}. Mean and std are computed over $5$ seeds.}
  \label{tab::evol_m_sir}
    \centering

\begin{tabular}{lcccc}
\toprule
$m_t^G$ & $t=0.001$ & $t=0.25$ & $t=0.5$ & $t=1$ \\
\midrule
$n=2$  & $1524.19 \pm 1078.99$ & $2.81 \pm 0.01$ & $1.15 \pm 0.00$ & $1.01 \pm 0.00$ \\
$n=22$ & $16838.38 \pm 12186.26$ & $12.51 \pm 0.13$ & $1.67 \pm 0.01$ & $1.03 \pm 0.00$ \\
$n=52$ & $39721.55 \pm 28684.69$ & $27.05 \pm 0.31$ & $2.45 \pm 0.01$ & $1.08 \pm 0.00$ \\
\bottomrule
\end{tabular}

\vspace{0.5em}

\begin{tabular}{lcccc}
\toprule
$m_t^L$ & $t=0.001$ & $t=0.25$ & $t=0.5$ & $t=1$ \\
\midrule
$n=2$  & $1514.96 \pm 1066.81$ & $2.33 \pm 0.00$ & $1.12 \pm 0.00$ & $1.01 \pm 0.00$ \\
$n=22$ & $14785.96 \pm 9764.15$ & $2.33 \pm 0.00$ & $1.12 \pm 0.00$ & $1.01 \pm 0.00$ \\
$n=52$ & $29945.66 \pm 17734.49$ & $2.33 \pm 0.00$ & $1.12 \pm 0.00$ & $1.01 \pm 0.00$ \\
\bottomrule
\end{tabular}
\end{table}
\begin{table}[H]
\caption{Evolution of $m_t^G$ and $m_t^L$ for \textbf{Lotka Volterra}. Mean and std are computed over $5$ seeds.}
  \label{tab::evol_m_lotka}
  \centering
\begin{tabular}{lcccc}
\toprule
$m_t^G$ & $t=0.001$ & $t=0.25$ & $t=0.5$ & $t=1$ \\
\midrule
$n=2$  & $41.37 \pm 6.11$ & $2.65 \pm 0.04$ & $1.14 \pm 0.00$ & $1.01 \pm 0.00$ \\
$n=22$ & $440.34 \pm 72.71$ & $11.71 \pm 0.44$ & $1.68 \pm 0.02$ & $1.04 \pm 0.00$ \\
$n=52$ & $1040.12 \pm 171.19$ & $25.28 \pm 0.99$ & $2.49 \pm 0.05$ & $1.08 \pm 0.00$ \\
\bottomrule
\end{tabular}

\vspace{0.5em}

\begin{tabular}{lcccc}
\toprule
$m_t^L$ & $t=0.001$ & $t=0.25$ & $t=0.5$ & $t=1$ \\
\midrule
$n=2$  & $41.38 \pm 6.09$ & $2.26 \pm 0.01$ & $1.12 \pm 0.00$ & $1.01 \pm 0.00$ \\
$n=22$ & $440.57 \pm 74.18$ & $2.32 \pm 0.00$ & $1.12 \pm 0.00$ & $1.01 \pm 0.00$ \\
$n=52$ & $1032.74 \pm 172.97$ & $2.33 \pm 0.00$ & $1.12 \pm 0.00$ & $1.01 \pm 0.00$ \\
\bottomrule
\end{tabular}
\end{table}
\newpage
 \subsection{Additional results on the choice of $\gamma$}
 \label{app::different_gamma}
 Theoretical results indicate that decreasing $\gamma$ yields a better Wasserstein quality but requires \textit{lower} steps sizes and \textit{more} steps per level.
 Indeed, for $\gamma\in[ \frac{\epsilon_\text{DSM}}{\omega m_t},+\infty[$ (which is the case in the experiments we run), $h_t$ is a second order polynomial with unique root $\frac{\epsilon_\text{DSM}}{\omega m_t}$ and positive second order derivative, according to our decision rule in Section~\ref{sec::fine_tune}: it is thus an increasing function of $\gamma$. This means that $\gamma\leq \tilde \gamma$ implies $h_t\leq \tilde h_t$.
 Let $f(\gamma)=\frac{(1-\omega)\gamma}{\gamma+\mathcal{W}_2(\pi_t,\pi_{t+1})}$: it is an increasing function of $\gamma>0$ that takes values in $[0,1]$ since
 \begin{align*}
 f'(\gamma)=\frac{(1-\omega)\mathcal{W}_2(\pi_t,\pi_{t+1})}{(\gamma+\mathcal{W}_2(\pi_t,\pi_{t+1}))^2}\geq 0.
 \end{align*} As we choose $k_t=\log\left( f(\gamma)\right)\frac{1}{\log(1-m_{t}h_{t})}$ according to our decision rule,
 \begin{align*}
 \gamma\leq \tilde \gamma &\Rightarrow \log f(\gamma)\leq \log f(\tilde\gamma)\leq 0 \quad \text{since $f$ and $\log$ are increasing}\\
 &\Rightarrow\log\left( f(\gamma)\right)\frac{1}{\log(1-m_{t}h_{t})}\geq \log\left( f(\tilde\gamma)\right)\frac{1}{\log(1-m_{t}h_{t})}\\
 &\Rightarrow k_t\geq \tilde k_t.
 \end{align*} We use the same $\gamma$ for all the experiments for consistency and notice that $\gamma=0.5$ yields good empirical Wasserstein quality for multimodal tasks (e.g. \textbf{GMM likelihood}) and a reasonable computational cost with respect to lower or higher $\gamma$. Concerning other tasks, a slightly greater $\gamma$ would also be a proper choice. Indeed, Tables~\ref{tab::choice_gamma_gmm} and \ref{tab:choice_gamma_lotka} seem to indicate that the final Wasserstein quality in these tasks decreases slower than $O(\gamma)$ for small $n$: this means that a greater $\gamma$ yields a lower total computational cost without degrading the final accuracy too rapidly.
\begin{table}[H]
\centering
\caption{Results for \textbf{GMM likelihood}. The first table presents the final Wasserstein accuracy;  the second one shows the corresponding total computational cost. Mean and std are computed over $5$ seeds.}
\label{tab::choice_gamma_gmm}

\begin{minipage}{0.98\linewidth}
\centering

\textbf{Wasserstein accuracy}

\vspace{0.3em}

\begin{tabular}{lcc@{\qquad}cc}
\toprule
& \multicolumn{2}{c}{\textbf{Geffner's compositional score $s_t^G$}}
& \multicolumn{2}{c}{\textbf{Linhart's compositional score $s_t^L$}} \\
\cmidrule(lr){2-3} \cmidrule(lr){4-5}
$n$ & $\gamma=0.5$ & $\gamma=1.5$ & $\gamma=0.5$ & $\gamma=1.5$ \\
\midrule
$2$  & $0.28 \pm 0.17$ & $0.30 \pm 0.20$ & $0.38 \pm 0.17$ & $0.47 \pm 0.20$ \\
$12$ & $0.06 \pm 0.02$ & $0.07 \pm 0.02$ & $0.07 \pm 0.03$ & $0.11 \pm 0.09$ \\
$22$ & $0.03 \pm 0.00$ & $0.04 \pm 0.01$ & $0.03 \pm 0.01$ & $0.07 \pm 0.05$ \\
$32$ & $0.02 \pm 0.00$ & $0.04 \pm 0.01$ & $0.02 \pm 0.00$ & $0.07 \pm 0.05$ \\
$42$ & $0.02 \pm 0.00$ & $0.06 \pm 0.06$ & $0.02 \pm 0.00$ & $0.10 \pm 0.10$ \\
$52$ & $0.01 \pm 0.00$ & $0.07 \pm 0.07$ & $0.02 \pm 0.00$ & $0.11 \pm 0.11$ \\
\bottomrule
\end{tabular}

\vspace{0.8em}

\textbf{Computational cost}

\vspace{0.3em}

\begin{tabular}{lcc@{\qquad}cc}
\toprule
& \multicolumn{2}{c}{\textbf{Geffner's compositional score $s_t^G$}}
& \multicolumn{2}{c}{\textbf{Linhart's compositional score $s_t^L$}} \\
\cmidrule(lr){2-3} \cmidrule(lr){4-5}
$n$ & $\gamma=0.5$ & $\gamma=1.5$ & $\gamma=0.5$ & $\gamma=1.5$ \\
\midrule
$2$  & $9548 \pm 7379$ & $906 \pm 710$ & $8152 \pm 5787$ & $766 \pm 544$ \\
$12$ & $9014 \pm 11693$ & $781 \pm 1006$ & $3858 \pm 2643$ & $356 \pm 238$ \\
$22$ & $5292 \pm 3801$ & $452 \pm 325$ & $2701 \pm 562$ & $249 \pm 48$ \\
$32$ & $4711 \pm 2638$ & $400 \pm 224$ & $2510 \pm 371$ & $232 \pm 30$ \\
$42$ & $4438 \pm 2498$ & $375 \pm 213$ & $2432 \pm 328$ & $226 \pm 26$ \\
$52$ & $4606 \pm 3059$ & $389 \pm 261$ & $2351 \pm 242$ & $219 \pm 19$ \\
\bottomrule
\end{tabular}

\end{minipage}
\end{table}

\begin{table}[H]
\centering
\caption{Results for \textbf{Lotka Volterra}. The first table presents the final Wasserstein accuracy;  the second one shows the corresponding total computational cost. Mean and std are computed over $5$ seeds.}
\label{tab:choice_gamma_lotka}

\begin{minipage}{0.98\linewidth}
\centering

\textbf{Wasserstein accuracy}

\vspace{0.3em}

\begin{tabular}{lcc@{\qquad}cc}
\toprule
& \multicolumn{2}{c}{\textbf{Geffner's compositional score $s_t^G$}}
& \multicolumn{2}{c}{\textbf{Linhart's compositional score $s_t^L$}} \\
\cmidrule(lr){2-3} \cmidrule(lr){4-5}
$n$ & $\gamma=0.5$ & $\gamma=1.5$ & $\gamma=0.5$ & $\gamma=1.5$ \\
\midrule
$2$  & $5.7\mathrm{e}{-03}\pm 8.4\mathrm{e}{-03}$ & $5.7\mathrm{e}{-03}\pm 8.0\mathrm{e}{-03}$ & $5.2\mathrm{e}{-03}\pm 7.4\mathrm{e}{-03}$ & $5.1\mathrm{e}{-03}\pm 7.0\mathrm{e}{-03}$ \\
$12$ & $1.6\mathrm{e}{-03} \pm 1.6\mathrm{e}{-03}$
     & $1.7\mathrm{e}{-03} \pm 1.7\mathrm{e}{-03}$
     & $1.7\mathrm{e}{-03} \pm 1.4\mathrm{e}{-03}$
     & $2.2\mathrm{e}{-03} \pm 2.2\mathrm{e}{-03}$ \\
$22$ & $2.4\mathrm{e}{-03} \pm 3.9\mathrm{e}{-03}$
     & $2.5\mathrm{e}{-03} \pm 4.0\mathrm{e}{-03}$
     & $2.1\mathrm{e}{-03} \pm 3.0\mathrm{e}{-03}$
     & $3.7\mathrm{e}{-03} \pm 3.3\mathrm{e}{-03}$ \\
$32$ & $2.2\mathrm{e}{-03} \pm 3.7\mathrm{e}{-03}$
     & $2.4\mathrm{e}{-03} \pm 3.7\mathrm{e}{-03}$
     & $2.2\mathrm{e}{-03} \pm 2.5\mathrm{e}{-03}$
     & $5.8\mathrm{e}{-03} \pm 3.0\mathrm{e}{-03}$ \\
$42$ & $2.2\mathrm{e}{-03} \pm 3.7\mathrm{e}{-03}$
     & $2.3\mathrm{e}{-03} \pm 3.8\mathrm{e}{-03}$
     & $2.8\mathrm{e}{-03} \pm 2.7\mathrm{e}{-03}$
     & $1.2\mathrm{e}{-02} \pm 5.6\mathrm{e}{-03}$ \\
$52$ & $2.3\mathrm{e}{-03} \pm 4.1\mathrm{e}{-03}$
     & $2.4\mathrm{e}{-03} \pm 4.3\mathrm{e}{-03}$
     & $4.0\mathrm{e}{-03} \pm 3.6\mathrm{e}{-03}$
     & $1.8\mathrm{e}{-02} \pm 1.2\mathrm{e}{-02}$ \\
\bottomrule
\end{tabular}

\vspace{0.8em}

\textbf{Computational cost}

\vspace{0.3em}

\begin{tabular}{lcc@{\qquad}cc}
\toprule
& \multicolumn{2}{c}{\textbf{Geffner's compositional score $s_t^G$}}
& \multicolumn{2}{c}{\textbf{Linhart's compositional score $s_t^L$}} \\
\cmidrule(lr){2-3} \cmidrule(lr){4-5}
$n$ & $\gamma=0.5$ & $\gamma=1.5$ & $\gamma=0.5$ & $\gamma=1.5$ \\
\midrule
$2$  & $31048 \pm 18609$ & $14717 \pm 8856$
     & $16926 \pm 9567$ & $5218 \pm 2955$ \\

$12$ & $5805 \pm 3771$ & $2828 \pm 1826$
     & $411 \pm 85$ & $122 \pm 26$ \\

$22$ & $3249 \pm 2124$ & $1587 \pm 1031$
     & $299 \pm 16$ & $94 \pm 3$ \\

$32$ & $2268 \pm 1406$ & $1106 \pm 683$
     & $283 \pm 8$ & $91 \pm 2$ \\

$42$ & $1758 \pm 1089$ & $856 \pm 529$
     & $278 \pm 7$ & $88 \pm 2$ \\

$52$ & $1442 \pm 855$ & $702 \pm 415$
     & $276 \pm 7$ & $87 \pm 1$ \\
\bottomrule
\end{tabular}

\end{minipage}
\end{table}
\newpage
\subsection{Additional results on the choice of $T$}
\label{app::different_T}
According to theoretical results, the choice of $T$ should not affect the final Wasserstein accuracy: indeed, we choose hyperparameters at each annealing level so that the intermediate Wasserstein quality is kept controlled by $\gamma$ (see Section~\ref{sec::fine_tune}). However, as $T$ increases, the Wasserstein distance between consecutive bridging densities, $\mathcal{W}_2(\pi_{t_p},\pi_{t_{p+1}})$, which enters the selection of $k_t$, decreases: in the \textit{pure Gaussian} setting, we observe that numerical instabilities may arise when $T$ becomes sufficiently large (e.g. $T=50$ or $100$), although this is not systematic. The experiments presented in the paper always use $T=10$ as originally proposed in \citet{song_annealed_langevin} for image generation. We test different values of $T$ on non Gaussian tasks but this does not provide any significant empirical Wasserstein improvement (especially for Linhart's score). Table~\ref{tab::choice_T_sir} presents the results for the \textbf{SIR} task: on this task, numerical instabilities may appear for larger $T$ than in the pure Gaussian case (e.g. $T\geq1000$).
\begin{table}[H]
\caption{Wasserstein error obtained for $T=10,20,50$ on the \textbf{SIR} task for different values of $n$. Mean and std are computed over $5$ seeds.}
\centering
\label{tab::choice_T_sir}
\begin{tabular}{lccc}
\multicolumn{4}{c}{\textbf{Geffner's compositional score $s_t^G$}}\\
\toprule
$n$ & $T=10$ & $T=20$ & $T=50$ \\
\midrule
$2$  & $1.6\mathrm{e}{-04} \pm 1.8\mathrm{e}{-04}$
     & $6.9\mathrm{e}{-05} \pm 7.6\mathrm{e}{-05}$
     & $4.6\mathrm{e}{-05} \pm 4.5\mathrm{e}{-05}$ \\

$12$ & $5.3\mathrm{e}{-05} \pm 5.0\mathrm{e}{-05}$
     & $2.3\mathrm{e}{-05} \pm 1.8\mathrm{e}{-05}$
     & $1.8\mathrm{e}{-05} \pm 1.8\mathrm{e}{-05}$ \\

$22$ & $5.0\mathrm{e}{-05} \pm 5.4\mathrm{e}{-05}$
     & $1.7\mathrm{e}{-05} \pm 1.2\mathrm{e}{-05}$
     & $1.2\mathrm{e}{-05} \pm 9.1\mathrm{e}{-06}$ \\

$32$ & $4.9\mathrm{e}{-05} \pm 5.2\mathrm{e}{-05}$
     & $1.5\mathrm{e}{-05} \pm 1.1\mathrm{e}{-05}$
     & $1.2\mathrm{e}{-05} \pm 9.7\mathrm{e}{-06}$ \\

$42$ & $5.1\mathrm{e}{-05} \pm 5.6\mathrm{e}{-05}$
     & $1.5\mathrm{e}{-05} \pm 1.2\mathrm{e}{-05}$
     & $1.3\mathrm{e}{-05} \pm 1.1\mathrm{e}{-05}$ \\

$52$ & $4.8\mathrm{e}{-05} \pm 5.2\mathrm{e}{-05}$
     & $1.3\mathrm{e}{-05} \pm 1.1\mathrm{e}{-05}$
     & $1.0\mathrm{e}{-05} \pm 7.0\mathrm{e}{-06}$ \\
\bottomrule
\end{tabular}

\vspace{0.5em}

\begin{tabular}{lccc}
\multicolumn{4}{c}{\textbf{Linhart's compositional score $s_t^L$}}\\
\toprule
$n$ & $T=10$ & $T=20$ & $T=50$ \\
\midrule
$2$  & $5.5\mathrm{e}{-05} \pm 6.9\mathrm{e}{-05}$
     & $5.4\mathrm{e}{-05} \pm 6.5\mathrm{e}{-05}$
     & $4.8\mathrm{e}{-05} \pm 4.8\mathrm{e}{-05}$ \\

$12$ & $3.0\mathrm{e}{-05} \pm 2.5\mathrm{e}{-05}$
     & $2.9\mathrm{e}{-05} \pm 2.5\mathrm{e}{-05}$
     & $2.8\mathrm{e}{-05} \pm 2.8\mathrm{e}{-05}$ \\

$22$ & $3.2\mathrm{e}{-05} \pm 2.3\mathrm{e}{-05}$
     & $2.8\mathrm{e}{-05} \pm 2.6\mathrm{e}{-05}$
     & $2.8\mathrm{e}{-05} \pm 3.1\mathrm{e}{-05}$ \\

$32$ & $3.1\mathrm{e}{-05} \pm 2.3\mathrm{e}{-05}$
     & $2.9\mathrm{e}{-05} \pm 2.4\mathrm{e}{-05}$
     & $2.7\mathrm{e}{-05} \pm 2.8\mathrm{e}{-05}$ \\

$42$ & $4.1\mathrm{e}{-05} \pm 2.8\mathrm{e}{-05}$
     & $3.3\mathrm{e}{-05} \pm 2.3\mathrm{e}{-05}$
     & $2.8\mathrm{e}{-05} \pm 2.8\mathrm{e}{-05}$ \\

$52$ & $5.2\mathrm{e}{-05} \pm 4.3\mathrm{e}{-05}$
     & $3.5\mathrm{e}{-05} \pm 2.5\mathrm{e}{-05}$
     & $2.5\mathrm{e}{-05} \pm 2.6\mathrm{e}{-05}$ \\
\bottomrule
\end{tabular}
\end{table}

\end{document}